\documentclass[11pt]{article}
\usepackage{fullpage}
\usepackage{amsmath,amsthm,amsfonts,amssymb,dsfont,bm}
\usepackage{enumerate,color,xcolor}
\usepackage[colorlinks,linkcolor=blue,citecolor=orange]{hyperref}

\usepackage{url}
\usepackage{graphicx}
\usepackage{array,longtable,booktabs}
\usepackage{subfigure}
\usepackage{placeins}

\usepackage[nameinlink,capitalise,noabbrev]{cleveref}

\numberwithin{equation}{section}
\theoremstyle{plain}

\newtheorem{theorem}{Theorem}
\newtheorem{corollary}[theorem]{Corollary}

\newtheorem{lemma}[theorem]{Lemma}
\newtheorem{proposition}[theorem]{Proposition}
\newtheorem{assumption}[theorem]{Assumption}

\newcommand{\ep}{\varepsilon}

\title{Penalized Nonreversible Langevin for Constrained Sampling}
\author{Pervez Ali\thanks{Department of Mathematics, Florida State University, Tallahassee, Florida, United States of America, \texttt{pa22g@fsu.edu}}
\and Weihao Dong\thanks{Hong Kong University of Science and Technology (Guangzhou), Guangzhou, Guangdong, People's Republic of China, \texttt{wdong201@connect.hkust-gz.edu.cn}} \and Xiaoyu Wang\thanks{FinTech Thrust, Hong Kong University of Science and Technology (Guangzhou), Guangzhou, Guangdong, People's Republic of China, \texttt{xiaoyuwang@hkust-gz.edu.cn}} }
\date{\today}

\begin{document}

\maketitle

\begin{abstract}
We propose penalized nonreversible Langevin algorithms for sampling from $\pi(x)\propto e^{-f(x)}\mathbf 1_{\mathcal C}(x)$, where $\mathcal C\subset\mathbb R^d$ is a compact convex set. The algorithms combine a squared distance penalty with constant or compatible state dependent skew symmetric perturbations that preserve the penalized Gibbs distribution. For smooth, possibly nonconvex $f$, we derive nonasymptotic total variation bounds for the full gradient algorithm under a log Sobolev inequality. When unbiased stochastic gradients are available, we establish $2$-Wasserstein bounds under global contraction and Lipschitz conditions on the full drift in an adapted quadratic metric. For a fixed penalty parameter, the error relative to the penalized Gibbs distribution decays exponentially to an $\mathcal{O}(\sqrt{\eta})$ neighborhood, where $\eta$ is the stepsize. We also bound the discrepancy between the penalized Gibbs distribution and the constrained target. In a two dimensional quadratic model, we establish nonreversible acceleration by tuning the skew perturbation to the curvature imbalance induced by penalization. With the target accuracy and smaller curvature fixed and initial Wasserstein distances uniformly bounded, tuning the skew perturbation improves the sufficient Euler iteration bound from linear to logarithmic in the curvature ratio. Numerical experiments evaluate the algorithms on constrained Bayesian regression, classification, neural networks, and truncated sampling, and examine the acceleration mechanism in a stochastic quadratic model.
\end{abstract}

\section{Introduction}

We consider the problem of sampling from a probability distribution on a constrained domain $\mathcal C\subset\mathbb R^d$ with density
\begin{equation}
\pi(x)\propto e^{-f(x)}\mathbf 1_{\mathcal C}(x).
\label{eq:intro:constrained-target}
\end{equation}
Such distributions arise in Bayesian inference, inverse problems, and statistical learning when prior information or feasibility requirements restrict the parameter space \cite{gelman1995bayesian,stuart2010inverse,andrieu2003introduction}. On $\mathbb R^d$, Langevin diffusions and their discretizations provide a standard method for sampling smooth target distributions \cite{roberts1996exponential}. When $f$ is a finite sum, stochastic gradient Langevin dynamics replaces the full gradient by a minibatch estimate and makes Langevin sampling applicable to large data sets \cite{welling2011bayesian,teh2016consistency,dalalyan2019user,raginsky2017non,xu2018global}. Related Langevin methods include regime switching algorithms with nonasymptotic Wasserstein guarantees~\cite{wang2025regime}, and Hessian-free high-resolution dynamics for non-log-concave targets~\cite{wang2026hessianfree}, with further contraction and discretization bounds in~\cite{lv2026improved}. The constraint in~\eqref{eq:intro:constrained-target}, however, prevents a direct application of unconstrained Langevin algorithms.
Several approaches have been developed for constrained sampling. Projected Langevin algorithms return every proposal to the constraint set \cite{bubeck2018projected}, and projected stochastic gradient variants permit nonconvex objectives and noisy gradients \cite{Lamperski2021,zheng2022constrained}. Proximal algorithms replace the indicator of the constraint set by a tractable proximal regularization \cite{brosse2017proximal,SR2020}, while mirror Langevin algorithms encode the constraint geometry through a mirror map \cite{hsieh2018mirrored,Chewi2020,Zhang2020,TaoMirror2021,Ahn2021}. Reflected Langevin dynamics instead enforces the constraint through a boundary local time. A nonreversible extension based on skew reflection and skew projection was recently studied in~\cite{DFTWZ2025}. These methods treat the boundary directly, but their implementations require a projection, a proximal map, a mirror map, or a discretization of the reflection mechanism.

Penalization offers a different route. The squared distance penalty replaces the hard constraint by the smooth potential $V_\delta=f+\operatorname{dist}(\cdot,\mathcal C)^2/\delta$ on $\mathbb R^d$. The corresponding Gibbs law $\pi_\delta$ approaches $\pi$ as $\delta\rightarrow0$, and the resulting Langevin iterates evolve on $\mathbb R^d$, with the projection entering through the penalty gradient $\nabla S(x)=2(x-\mathcal P_{\mathcal C}x)$. Gurbuzbalaban, Hu, and Zhu~\cite{gurbuzbalaban2024penalized} developed this approach for overdamped and underdamped Langevin algorithms and proved quantitative penalty approximation and iteration bounds. The penalty also creates a stiffness tradeoff. Decreasing $\delta$ improves the approximation of $\pi$ but increases the smoothness constant of $V_\delta$, so the discretization step must be reduced. Moreover, the penalized chain evolves on all of $\mathbb R^d$. Compactness of $\mathcal C$ therefore does not itself control its moments, and the analysis must obtain these bounds from the penalized potential.

Nonreversible Langevin dynamics provide a second mechanism for improving a sampler. A skew symmetric perturbation of the drift can preserve the Gibbs law while breaking detailed balance. For Gaussian diffusions, suitable skew matrices improve the spectral convergence rate \cite{hwang1993accelerating,HHS05,lelievre2013optimal,WHC2014}. More generally, nonreversible perturbations have been studied through convergence to equilibrium, asymptotic variance, large deviations, discretization, and stochastic gradient methods~\cite{reybellet2015irreversible,duncan2016variance,ma2015complete,hu2020non}. For stochastic gradient discretizations, Ni et al.~\cite{ni2026variance} establish a small stepsize central limit theorem and give conditions under which an antisymmetric perturbation reduces the leading fluctuation variance of empirical averages. Related recent developments include a large deviations approach to accelerating constrained sampling and generalized EXTRA stochastic gradient Langevin dynamics~\cite{wang2026accelerating,gurbuzbalaban2024generalized}. These results do not by themselves provide a constrained sampler based on penalization. In particular, an improvement of a continuous time rate need not improve the iteration complexity, because the same skew perturbation also changes the discretization constants and the admissible stepsize.

In this paper, we combine squared distance penalization with the drift
$$
-(I+J(x))\nabla V_\delta(x).
$$
The matrix $J(x)$ may be constant or state dependent. For bounded Lipschitz $J$, skew symmetry and the compatibility condition $\nabla\cdot(J\nabla V_\delta)=0$ ensure that the perturbation preserves $\pi_\delta$. We analyze both full and stochastic gradient discretizations. The proof keeps three errors separate: convergence of the algorithm to $\pi_\delta$, approximation of $\pi$ by $\pi_\delta$, and stochastic gradient noise. This separation makes the dependence on $\delta$, the stepsize, and the matrix $J$ explicit.

Our contributions are as follows.

\begin{itemize}
\item We establish invariance of $\pi_\delta$ for bounded state dependent skew matrices that are Lipschitz and satisfy the compatibility condition. For the full gradient algorithm, an entropy calculation gives nonasymptotic total variation bounds under a log Sobolev inequality. Combining this estimate with the target approximation results of~\cite{gurbuzbalaban2024penalized} yields a guarantee for the constrained law $\pi$.

\item For stochastic gradients, under the global drift contraction condition, we prove Wasserstein bounds in the strongly convex case and in the dissipative nonconvex case. The estimates use uniform moments of the iterates and of $\pi_\delta$ to control the Euler remainder and the stochastic gradient error. They cover constant and state dependent matrices, with the additional fourth moment terms displayed in the state dependent case.

\item We formulate contraction in an adapted $P_J$ metric. At a critical point where $\nabla V_\delta$ is differentiable with positive definite derivative, the matrix $P_J$ is constructed from the linearized drift $(I+J(x_*))\nabla^2V_\delta(x_*)$. A strict improvement of its spectral rate over the reversible curvature yields a faster local contraction. To obtain a global Wasserstein bound, the same metric is required to control $B_J(x)-B_J(y)$ for every $x,y\in\mathbb R^d$.

\item For a quadratic Hessian block in the small $\delta$ regime, we make the acceleration comparison explicit. Every fixed $a\neq0$ raises both the slow spectral rate and the leading mixing term at a common admissible stepsize by the factor $1+a^2$. For fixed target accuracy and polynomially controlled initial distances, the tuned choice $a_\delta=\Theta(\Lambda_\delta^{1/2})$ gives a sufficient bound logarithmic in $\Lambda_\delta/\lambda$. Pairing Hessian eigenvectors extends the construction to high dimensional block matrices. The same real two dimensional representation applies to skew circulant matrices when their Fourier planes are invariant under the Hessian.

\end{itemize}

The rest of the paper is organized as follows. Section~\ref{sec:full:grad} introduces the full gradient penalized nonreversible Langevin algorithm and proves its total variation guarantees. Section~\ref{sec:stochastic:grad} studies stochastic gradient algorithms, constructs the $P_J$ metric, and gives the strongly convex and nonconvex Wasserstein analyses. The appendices contain the supporting proofs and constant calculations.

We consider the penalized target distribution:
\begin{equation}
\label{eq:penalized:target}
\pi_{\delta} \propto \exp\left(-f(x) - \frac{1}{\delta}S(x)\right), \qquad x \in \mathbb{R}^d,
\end{equation}
for sufficiently small $\delta>0$. The penalty $S$ satisfies the following assumptions.
\begin{assumption}
\label{assumption:S}
Assume that $S(x) = 0$ for any $x \in \mathcal{C}$ and $S(x) > 0$ for any $x \not\in \mathcal{C}$.
\end{assumption}

\begin{assumption}
\label{assumption:C}
Assume that $\mathcal{C}$ is a convex body, i.e., $\mathcal{C}$ is a compact convex set, contains an open ball centered at 0 with radius $r>0$, and is contained in a Euclidean ball centered at 0 with radius $R>0$.
\end{assumption}
Define the distance to the constraint set by
\begin{equation}
\delta_{\mathcal{C}}(x) = \mathrm{distance}(x, \mathcal{C}) := \min_{c \in \mathcal{C}}\|x - c\|, \qquad\text{for $x \in \mathbb{R}^d$.}
\end{equation}
We use the following projection identity for the squared distance penalty.
\begin{lemma}
\label{lemma:squared:distance:penalty}
\cite[Lemma 2.6]{gurbuzbalaban2024penalized}. If $\mathcal{C}$ is convex, then $S(x)=\delta_{\mathcal{C}}(x)^2$ is convex and continuously differentiable on $\mathbb{R}^d$, and its gradient is Lipschitz with constant $\ell=4$. Moreover, $\nabla S(x)=2\left(x-\mathcal{P}_{\mathcal{C}}(x)\right)$, where ${\mathcal{P}}_{\mathcal{C}}(x):=\arg\min_{c\in\mathcal{C}}\|x-c\|$ is the Euclidean projection onto $\mathcal C$.
\end{lemma}
The corresponding target approximation estimate is the following result.
\begin{theorem}
\cite[Theorem 2.7]{gurbuzbalaban2024penalized}.
\label{thm:dist:perturb}
Let $S(x)=\delta_{\mathcal C}(x)^2$. Suppose Assumptions \ref{assumption:S} and \ref{assumption:C} hold. Moreover, we assume that $f$ is continuous and $e^{-f}$ is integrable over $\mathcal{C}$ and that there exist $\hat{\alpha}>0$ and $\hat{x} \in \mathbb{R}^d$ such that $\int_{\mathbb{R}^d} e^{\hat{\alpha}\|x-\hat{x}\|^2} e^{-\frac{S(x)}{\delta}-f(x)} d x<\infty$. Then, as $\delta \rightarrow 0$,
\begin{equation}
\mathcal{W}_2\left(\pi_\delta, \pi\right) \leq \mathcal{O}\left(\delta^{1 / 8}(\log (1 / \delta))^{1 / 8}\right) .
\end{equation}
\end{theorem}
We use $S(x)=\operatorname{dist}(x,\mathcal C)^2$ throughout, except where another penalty is explicitly specified. We first study the deterministic gradient penalized Langevin algorithm when $f$ is nonconvex.
\begin{assumption}
\label{assumption:smooth}
Assume that $f$ has an $L$ Lipschitz gradient, that is, $\|\nabla f(x)-\nabla f(y)\|\leq L\|x-y\|$ for all $x,y\in\mathbb{R}^d$.
\end{assumption}
For $0<\delta<(L+1/2)^{-1}$, Lemmas~\ref{lemma:dissipative} and~\ref{lemma:uniform:bound} give $m_\delta>0$ and
$$
V_\delta(x)\geq\frac{m_\delta}{4}\|x\|^2-c_\delta,\qquad \int_{\mathbb R^d}e^{m_\delta\|x\|^2/8-V_\delta(x)}dx\leq e^{c_\delta}\left(\frac{8\pi}{m_\delta}\right)^{d/2}<\infty.
$$
Thus $Z_\delta<\infty$ and the exponential integrability hypothesis of Theorem~\ref{thm:dist:perturb} holds. We propose the penalized nonreversible Langevin SDE:
\begin{equation}
\label{eq:sde}
dX(t) = -\left(I+J(X(t))\right)\left(\nabla f(X(t)) + \frac{1}{\delta}\nabla S(X(t))\right)dt + \sqrt{2}dW_t,
\end{equation}
where $J(x)^{\top}=-J(x)$. The constant skew case is recovered by taking $J(x)\equiv J$. Compactness of $\mathcal C$ does not directly control the penalized process, which evolves on $\mathbb R^d$ without projection. Instead, for sufficiently small $\delta$, the penalized potential $V_\delta=f+S/\delta$ is dissipative without imposing a separate dissipativity assumption on $f$. This implication is proved in Lemmas~\ref{lemma:dissipative:S} and~\ref{lemma:dissipative}, with the constants used below, and corresponds to \cite[Lemmas~C.1 and C.2]{gurbuzbalaban2024penalized}. The nonreversible analysis requires additional conditions beyond this penalty argument. In the Wasserstein results, Assumption~\ref{assumption:PJ:full:drift} requires the full drift $B_J=(I+J)\nabla V_\delta$ to satisfy pairwise contraction and Lipschitz bounds in a fixed metric for every $x,y\in\mathbb R^d$. This condition is stronger than dissipativity of $V_\delta$ and does not follow from smoothness of $f$ and compactness of $\mathcal C$. Moreover, when $J$ is state dependent, controlling the term $(J(x)-J(y))\nabla V_\delta(y)$ in the local discretization error requires fourth moment bounds for the penalized target and the Euler iterates. Lemma~\ref{lemma:uniform:bound} establishes these bounds. Neither this fourth moment analysis nor the global pairwise condition on the full drift is required in the corresponding reversible analysis of~\cite{gurbuzbalaban2024penalized}. We first establish invariance of the penalized target.
\begin{lemma}
\label{lemma:pi}
Let $V_\delta=f+S/\delta\in C^1(\mathbb R^d)$ have globally Lipschitz gradient and satisfy $Z_\delta:=\int_{\mathbb R^d}e^{-V_\delta(x)}dx<\infty$. Suppose that $J:\mathbb R^d\to\mathbb R^{d\times d}$ is bounded and Lipschitz in the operator norm, $J(x)^\top=-J(x)$ for every $x$, and
\begin{equation}
\label{eq:J:compatibility}
\nabla\cdot\left(J\nabla V_\delta\right)=0\quad \text{in distributions on }\mathbb R^d.
\end{equation}
Then~\eqref{eq:sde} has a unique solution and admits $\pi_\delta(dx)=Z_\delta^{-1}e^{-V_\delta(x)}dx$ as an invariant distribution.
\end{lemma}

\begin{proof}
See Appendix~\ref{proof:lemma:pi}.
\end{proof}

\section{Unadjusted Penalized Nonreversible Langevin Dynamics}
\label{sec:full:grad}
We first consider the penalized nonreversible Langevin algorithm for \emph{nonconvex} $f$:
\begin{equation}
\label{eq:alg}
x_{k+1} = x_k - \eta (I + J(x_k))\left(\nabla f\left(x_k\right)+\frac{1}{\delta} \nabla S\left(x_k\right)\right)+\sqrt{2 \eta} \xi_{k+1},
\end{equation}
where $\xi_k$ are i.i.d. $\mathcal{N}\left(0, I\right)$ Gaussian noises in $\mathbb{R}^d$. The iterates are not projected onto $\mathcal C$, and the required moment bounds come from dissipativity of the penalized potential. Theorem~\ref{thm:nu:pi:delta} adapts the entropy argument of \cite[Theorem~1]{ma2019sampling} to the nonreversible drift, Proposition~\ref{main:full:grad} then combines it with the target level consequences of \cite[Proposition~2.11]{gurbuzbalaban2024penalized}.

Fix $\ep>0$. Assume that $V_\delta=f+S/\delta$ has an $L_\delta$ Lipschitz gradient and is $(m_\delta,b_\delta)$ dissipative as in Lemma~\ref{lemma:dissipative}. Let $\pi_\delta$ be defined by~\eqref{eq:penalized:target}, and assume that it satisfies the log Sobolev inequality with constant $\rho_*>0$ in the form of
\begin{equation}
\mathrm{KL}(q\Vert\pi_\delta) \leq\frac1{2\rho_*} \int_{\mathbb R^d}\left\|\nabla\log\frac{q(x)}{p_\delta^*(x)}\right\|^2q(x)\,dx
\label{eq:LSI:convention}
\end{equation}
for every smooth probability density $q$ for which the right hand side is finite, where $p_\delta^*$ is the density of $\pi_\delta$. Let $X_0\sim\nu_0$, where $\nu_0$ has density $p_0$. Assume that $p_0\in C^\infty(\mathbb R^d)$ and has finite relative Fisher information:
\begin{equation}
\label{eq:init:fisher}
\int_{\mathbb R^d} \left\|\nabla\log\frac{p_0(x)}{p_\delta^*(x)}\right\|^2p_0(x)\,dx <\infty.
\end{equation}
Write $F(q):=\mathrm{KL}(q\Vert\pi_\delta)$ for a density $q$. Then~\eqref{eq:LSI:convention} gives
\begin{equation}
\label{eq:init:entropy}
F_0:=F(p_0)=\mathrm{KL}(p_0\Vert\pi_\delta)<\infty.
\end{equation}
Assume in addition that
\begin{equation}
\label{eq:init:potential:moment}
\mathbb E_{\nu_0}U_\delta(X_0)<\infty, \qquad U_\delta:=1+c_\delta+V_\delta,
\end{equation}
where $c_\delta$ is defined in~\eqref{eq:moment:quad:constants}. The following theorem gives the entropy bound.
\begin{theorem}
\label{thm:nu:pi:delta}
Under the preceding conditions \eqref{eq:init:fisher}, \eqref{eq:init:entropy}, and \eqref{eq:init:potential:moment}, suppose $J(x)^\top=-J(x)$ and $\|J(x)\|_{\mathrm{op}}\leq M_J$. If $J$ is state dependent, assume that it is $L_J$ Lipschitz in the operator norm, satisfies~\eqref{eq:J:compatibility}, and that $\nu_0$ has finite fourth moment. Let $\nu_K$ denote the law of the iterate after $K$ steps in~\eqref{eq:alg}.
\begin{itemize}
\item if $J(x)\equiv J$ is constant, assume $0<\eta\leq \eta_{\mathrm{const}}(\ep)$, where $\eta_{\mathrm{const}}(\ep)$ is defined in \eqref{eq:stepsize:const:explicit}. \item if $J$ is state dependent, assume $0<\eta\leq \eta_{\mathrm{sd}}(\ep)$, where $\eta_{\mathrm{sd}}(\ep)$ is defined in~\eqref{eq:stepsize:sd:explicit}.
\end{itemize}
If $2F_0\leq\ep^2$, let $K\in\mathbb N_0$ be arbitrary. If $2F_0>\ep^2$, choose $K\in\mathbb N_0$ such that
\begin{equation}
\label{eq:K:explicit:theorem}
K\geq \frac{1}{\rho_*\eta}\log\frac{2F_0}{\ep^2}.
\end{equation}
Then
\begin{equation}
\mathrm{TV}(\nu_K,\pi_{\delta})\leq \sqrt{2}\ep.
\end{equation}
\end{theorem}
\begin{proof}
See Appendix~\ref{proof:thm:nu:pi:delta}.
\end{proof}

For the nontrivial case $2F_0>\ep^2$, choosing the largest admissible stepsize in the applicable alternative of Theorem~\ref{thm:nu:pi:delta} and taking the smallest admissible integer $K$ gives
\begin{equation}
K=\mathcal O\left(\frac{1}{\rho_*\eta} \left(1+\log\frac{2F_0}{\ep^2}\right)\right).
\end{equation}
Under Assumptions~\ref{assumption:S},~\ref{assumption:C}, and \ref{assumption:smooth}, assume the constraint used in \cite[Proposition~2.11]{gurbuzbalaban2024penalized}:
\begin{equation}
\mathcal C=\{x:h(x)\leq0\}, \qquad h(x)=\max_{1\leq i\leq m}h_i(x),
\label{eq:ghz:constraint:representation}
\end{equation}
where $m\geq1$, each $h_i:\mathbb R^d\to\mathbb R$ is finite and convex, and $h$ is either strongly convex or merely convex. In the merely convex case, we additionally assume that the chosen constraint representation is strictly feasible at the origin:
\begin{equation}
h(0)<0.
\label{eq:ghz:strict:feasibility}
\end{equation}
This condition concerns the chosen defining function rather than only the set $\mathcal C$, it rules out degenerate representations for which the regularized set below may have zero volume. We consider
\begin{equation}
\label{eq:C:alpha}
\mathcal C^\alpha:=\left\{x:h(x)+\frac\alpha2\|x\|^2\leq0\right\}, \qquad S^\alpha:=\mathrm{dist}(\cdot,\mathcal C^\alpha)^2,
\end{equation}
where $\alpha=0$ if $h$ is strongly convex and $\alpha=\ep^2$ otherwise. The following proposition specializes Theorem~\ref{thm:nu:pi:delta} to this penalty.
\begin{proposition}
\label{main:full:grad}
For each fixed dimension $d$ and fixed problem data, take sufficiently small $\ep>0$, let $\delta=\ep^4$ and let $V_\delta^\alpha=f+S^\alpha/\delta$ and $\pi_\delta^\alpha\propto e^{-V_\delta^\alpha}$. Under the setting of Theorem~\ref{thm:nu:pi:delta}~\footnote{Here we adapt its notation $(V_\delta,\pi_\delta,U_\delta)$ by using $(V_\delta^\alpha,\pi_\delta^\alpha,U_\delta^\alpha)$.}, let $F_0=\mathrm{KL}(\nu_0\Vert\pi_\delta^\alpha) < \infty$. Lemma~\ref{lemma:ghz:transfer} gives
\begin{equation}
\label{eq:external:penalty:TV}
\mathrm{TV}(\pi_\delta^\alpha,\pi)=\widetilde{\mathcal O}(\ep), \qquad \rho_*^{-1}=\mathcal O(1).
\end{equation}
The basic $\delta$ orders used to reduce the stepsize restrictions in Theorem~\ref{thm:nu:pi:delta} are given in Lemma~\ref{lemma:complexity:penalty:orders}. Let $X_0\sim\nu_0$ satisfy the initialization assumptions of Theorem~\ref{thm:nu:pi:delta} with $U_\delta$ replaced by $U_\delta^\alpha$. \paragraph{Constant $J$.} Suppose first that $J$ is constant. Take $\eta_{\mathrm{const}}(\ep)$ defined in \eqref{eq:stepsize:const:explicit}. Then
\begin{equation}
\eta_{\mathrm{const}}^{-1} =\mathcal O\left((1+M_J)^2 \max\left\{\frac d{\ep^{10}}, \frac{\sqrt{\mathbb E_{\nu_0}U_\delta^\alpha(X_0)+\ep^{-4}+d}} {\ep^7}\right\}\right).
\label{eq:complexity:const:stepsize}
\end{equation}
\paragraph{State dependent $J$.} Suppose next that $J$ is state dependent and that $\nu_0$ has finite fourth moment. Take $\eta_{\mathrm{sd}}(\ep)$ defined in~\eqref{eq:stepsize:sd:explicit}. Then
\begin{equation}
\eta_{\mathrm{sd}}^{-1} =\mathcal O\left( \left((1+M_J)^2\vee L_J^2\right) \frac d{\ep^6}\left( \left(\mathbb E_{\nu_0}U_\delta^\alpha(X_0)^2\right)^{1/2} +\ep^{-4}+d\right)\right).
\label{eq:complexity:sd:general:initialization}
\end{equation}

For either stepsize $\eta_{\rm const}^{-1}$ or $\eta_{\rm sd}^{-1}$ satisfy \eqref{eq:complexity:const:stepsize} or \eqref{eq:complexity:sd:general:initialization}, respectively, take $$
K=\lceil(\rho_*\eta)^{-1}\log(2F_0/\ep^2)\rceil,
$$
when $2F_0>\ep^2$, and $K=0$ otherwise. Under either setting,
\begin{equation}
\mathrm{TV}(\nu_K,\pi) \leq\sqrt2\ep+\mathrm{TV}(\pi_\delta^\alpha,\pi) =\widetilde{\mathcal O}(\ep).
\end{equation}
\end{proposition}

\begin{proof}
See Appendix~\ref{proof:main:full:grad}.
\end{proof}

\section{Stochastic Penalized Nonreversible Langevin Dynamics}
\label{sec:stochastic:grad}
We next allow stochastic gradient estimates in the penalized nonreversible Langevin update. Write
\begin{equation}
V_\delta(x):=f(x)+\frac1\delta S(x), \qquad B_J(x):=(I+J(x))\nabla V_\delta(x).
\label{eq:full:drift:definition}
\end{equation}
For a fixed symmetric positive definite matrix $P_J\succ0$, define
$$
\|z\|_{P_J}^2:=z^\top P_Jz, \qquad \langle u,v\rangle_{P_J}:=u^\top P_Jv, \qquad \mathcal W_{2,P_J}^2(\mu,\nu) :=\inf_{\gamma\in\Gamma(\mu,\nu)}\int\|x-y\|_{P_J}^2\,\gamma(dx,dy).
$$
For a random vector $Z$ and $p\geq1$, define
\begin{equation}
\|Z\|_{L^p}:=\left(\mathbb E\|Z\|^p\right)^{1/p}, \qquad \|Z\|_{L^p(P_J)} :=\left(\mathbb E\|Z\|_{P_J}^p\right)^{1/p},
\label{eq:Lp:PJ:norms}
\end{equation}
whenever the corresponding moment is finite. We use the condition number
\begin{equation}
\kappa(P_J):=\frac{\lambda_{\max}(P_J)}{\lambda_{\min}(P_J)}.
\label{eq:PJ:condition:number}
\end{equation}
The penalized nonreversible stochastic gradient Langevin dynamics (PNSGLD) update is
\begin{equation}
\label{eq:pnsgld}
x_{k+1}=x_k-\eta\left(I+J(x_k)\right)\left(\widetilde\nabla f\left(x_k\right)+\frac{1}{\delta} \nabla S\left(x_k\right)\right)+\sqrt{2 \eta} \xi_{k+1},
\end{equation}
where $\xi_k \sim_{i.i.d.}\mathcal{N}\left(0, I\right)$ Gaussian noises in $\mathbb{R}^d$ and we assume that we have access to noise estimates $\widetilde\nabla f\left(x_k\right)$ of the actual gradients satisfying the following assumption.
\begin{assumption}
\label{assumption:sg}
At iteration $k$, let $\widetilde\nabla f(x_k,w_k)$ be a random estimate of $\nabla f(x_k)$ with random input $w_k$. The seeds $(w_k)_{k\geq0}$ are mutually independent, and the seed family, the initial state $x_0$, and the Gaussian noise sequence $(\xi_k)_{k\geq1}$ are independent. It satisfies
$$
\mathbb{E}\left[\widetilde\nabla f\left(x_k, w_k\right)-\nabla f\left(x_k\right) \Big\vert x_k\right]=0
$$
and
\begin{equation}
\mathbb{E}\left[\left\|\widetilde\nabla f\left(x_k, w_k\right)-\nabla f\left(x_k\right)\right\|^2 \bigg| x_k\right] \leq 2 \sigma^2\left(L^2\left\|x_k\right\|^2+\|\nabla f(0)\|^2\right) .
\end{equation}
\end{assumption}
For the finite sum model $f(x)=n^{-1}\sum_{i=1}^n f_i(x)$, a standard minibatch estimator is $\nabla\tilde f(x)=b^{-1}\sum_{j\in\Omega}\nabla f_j(x)$, where $1\leq b\leq n$ and $\Omega$ is sampled uniformly without replacement from the subsets of $\{1,\ldots,n\}$. For $g_i=\nabla f_i(x)$, $\bar g=n^{-1}\sum_i g_i$, and $n>1$, uniform sampling without replacement gives
$$
\mathbb E\left[\nabla\tilde f(x)\Big\vert x\right]=\bar g,\qquad \mathbb E\left[\left\|\nabla\tilde f(x)-\bar g\right\|^2 \Big\vert x\right]=\frac{n-b}{b(n-1)}\frac1n\sum_{i=1}^n\|g_i-\bar g\|^2.
$$
For $n=b=1$, the variance is zero. The convergence results require this variance to satisfy Assumption~\ref{assumption:sg}.

In the following, we first construct the adapted $P_J$ metric from the linearized nonreversible drift. We then use this metric in the strongly convex and nonconvex convergence analysis. Explicit acceleration examples are given after the convergence results.

\subsection{Construction of the \texorpdfstring{$P_J$}{P-J} Metric}
Let $x_*$ be a critical point of $V_\delta$. Assume that $J$ is Lipschitz in a neighborhood of $x_*$ and that $\nabla V_\delta$ is differentiable at $x_*$ with positive definite derivative, so that
\begin{equation}
\nabla V_\delta(x_*)=0, \qquad H_*:=\nabla^2V_\delta(x_*)\succ0,
\end{equation}
and define
\begin{equation}
A_J:=(I+J(x_*))H_*, \qquad \beta_J:=\min_{\lambda\in\sigma(A_J)}\operatorname{Re}\lambda.
\label{eq:beta:J}
\end{equation}
As $h\to0$, the differentiability of $\nabla V_\delta$ and the local Lipschitz continuity of $J$ give $\nabla V_\delta(x_*+h)=H_*h+o(\|h\|)$ and $J(x_*+h)=J(x_*)+\mathcal O(\|h\|)$, then we can get $B_J(x_*+h)-B_J(x_*)=(I+J(x_*))H_*h+o(\|h\|)$ under the setting in~\eqref{eq:full:drift:definition}. Thus $DB_J(x_*)=A_J$, and the local linearized dynamics is
\begin{equation}
\dot z_t=-A_Jz_t.
\label{eq:PJ:linearized:dynamics}
\end{equation}
The next proposition gives its contraction rate in an adapted metric.

\begin{proposition}[Local contraction in an adapted metric]
\label{prop:PJ:metric:construction}
Under the preceding setup, the following statements hold.
\begin{enumerate}[(i)]
\item The spectral rate satisfies
\begin{equation}
\beta_J\geq\lambda_{\min}(H_*)>0.
\label{eq:PJ:spectral:lower:bound}
\end{equation}

\item For every $\alpha_J\in(0,\beta_J)$, there exists a matrix $P_J\succ0$, depending on $\alpha_J$, such that
\begin{equation}
\label{eq:PJ:linear:contraction}
A_J^\top P_J+P_JA_J\succeq 2\alpha_JP_J.
\end{equation}
Consequently, every solution of~\eqref{eq:PJ:linearized:dynamics} satisfies
\begin{equation}
\|z_t\|_{P_J}\leq e^{-\alpha_Jt}\|z_0\|_{P_J}, \qquad t\geq0.
\end{equation}

\item In the reversible case $J(x_*)=0$, one may choose $P_0:=H_*$, and
\begin{equation}
A_0^\top P_0+P_0A_0\succeq 2\lambda_{\min}(H_*)P_0.
\end{equation}
Therefore, if $\beta_J>\lambda_{\min}(H_*)$, then choosing
\begin{equation}
\lambda_{\min}(H_*)<\alpha_J<\beta_J,
\end{equation}
gives a strictly faster local contraction rate in the adapted $P_J$ metric.
\end{enumerate}
\end{proposition}

\begin{proof}
See Appendix~\ref{proof:prop:PJ:metric:construction}.
\end{proof}
The condition $J(x_*)\neq0$ does not guarantee the strict inequality $\beta_J>\lambda_{\min}(H_*)$. For example, if $H_*=\lambda I$ with $\lambda>0$, the eigenvalues of a real skew symmetric matrix $J(x_*)$ are purely imaginary, hence $\mathrm{Re}\lambda_i\left((I+J(x_*))H_*\right)=\lambda$ and $\beta_J=\lambda=\lambda_{\min}(H_*)$, even when $J(x_*)\neq0$. Thus strict acceleration is a spectral property of the pair $(H_*,J(x_*))$ and depends on how the skew term couples the slow Hessian directions with directions of different curvature. Proposition~\ref{prop:PJ:metric:construction} therefore provides a candidate metric $P_J$ and its local linearized contraction rate. To obtain convergence of the nonlinear diffusion, the same metric must also control the full drift $B_J$ defined in~\eqref{eq:full:drift:definition} for every pair $x,y$. The following condition is used in both the strongly convex and nonconvex analyses.

\begin{assumption}
\label{assumption:PJ:full:drift}
For $B_J$ defined in~\eqref{eq:full:drift:definition}, there exist a fixed matrix $P_J\succ0$ and constants $\alpha_{P,J},L_{P,J}>0$ such that, for all $x,y\in\mathbb R^d$,
\begin{equation}
\label{eq:strong:full:drift}
\langle x-y,B_J(x)-B_J(y)\rangle_{P_J} \geq\alpha_{P,J}\|x-y\|_{P_J}^2, \quad \|B_J(x)-B_J(y)\|_{P_J} \leq L_{P,J}\|x-y\|_{P_J}.
\end{equation}
\end{assumption}
The first inequality in~\eqref{eq:strong:full:drift} is the global contraction condition for the full nonreversible drift. For constant $J$, smoothness and norm equivalence verify the Lipschitz condition as follows:
\begin{align}
B_J(x)-B_J(y) &=(I+J)\bigl(\nabla V_\delta(x)-\nabla V_\delta(y)\bigr),
\nonumber\\
\|B_J(x)-B_J(y)\|_{P_J} &\leq\sqrt{\kappa(P_J)}\|I+J\|_{\mathrm{op}}L_\delta \|x-y\|_{P_J}
\nonumber\\
&\leq\sqrt{\kappa(P_J)} (1+\|J\|_{\mathrm{op}})L_\delta\|x-y\|_{P_J} =L_{P,J}\|x-y\|_{P_J},
\label{eq:PJ:lip}
\end{align}
where $\kappa(P_J)$ is defined in \eqref{eq:PJ:condition:number} and $L_{P,J}:=\sqrt{\kappa(P_J)} (1+\|J\|_{\mathrm{op}})L_\delta$ where $V_{\delta}$ has an $L_\delta$ Lipschitz gradient. Thus the second inequality in~\eqref{eq:strong:full:drift} follows for constant $J$.
\begin{lemma}
\label{lemma:H}
For $B_J$ defined in~\eqref{eq:full:drift:definition}, let $H_\eta=I-\eta B_J$ and suppose Assumption~\ref{assumption:PJ:full:drift} holds. If $0<\eta\leq\alpha_{P,J}/(2L_{P,J}^2) \wedge 2 / (3\alpha_{P,J})$, then
\begin{equation}
\label{eq:strong:onestep}
\|H_\eta(x)-H_\eta(y)\|_{P_J} \leq\sqrt{1-\frac32\alpha_{P,J}\eta} \|x-y\|_{P_J}.
\end{equation}
\end{lemma}

\begin{proof}
See Appendix~\ref{proof:lemma:H}.
\end{proof}

For $J=0$, strong convexity and smoothness of $V_\delta$ verify Assumption~\ref{assumption:PJ:full:drift} with $P_J=I$, $\alpha_{P,J}=\mu$, and $L_{P,J}=L_\delta$.  For constant $J$, the local remainder estimate below requires a finite second moment of $\nabla V_\delta$ under $\pi_\delta$. For state dependent $J$, the additional term $(J(x)-J(y))\nabla V_\delta(y)$ is controlled using fourth moments. Lemma~\ref{lemma:uniform:bound} supplies these moments from strong convexity in the next subsection and from Assumption~\ref{assumption:non-convex:dissipativity} in the nonconvex case.

\begin{lemma}
\label{lemma:1:step:error}
Assume $Z_\delta:=\int_{\mathbb R^d}e^{-V_\delta(x)}\,dx\in(0,\infty)$, and let $\pi_\delta(dx)=Z_\delta^{-1}e^{-V_\delta(x)}dx$. Let $(Y_t)_{t\geq0}$ solve
$$
dY_t=-(I+J(Y_t))\nabla V_\delta(Y_t)dt+\sqrt2\,dW_t, \qquad Y_0\sim\pi_\delta.
$$
Assume $V_\delta$ has an $L_\delta$ Lipschitz gradient, $J(x)^\top=-J(x)$, and $\|J(x)\|_{\mathrm{op}}\leq M_J$. If $J$ is state dependent, also assume that it is $L_J$ Lipschitz in the operator norm and $\nabla\cdot(J\nabla V_\delta)=0$ in the sense of Lemma~\ref{lemma:pi}. For $0<\eta\leq1$, define
$$
\mathcal R_0:=\int_0^\eta \left((I+J(Y_t))\nabla V_\delta(Y_t)-(I+J(Y_0))\nabla V_\delta(Y_0)\right)dt.
$$
Using the random vector norms defined in \eqref{eq:Lp:PJ:norms}, the remainder satisfies the following bounds.

\paragraph{Constant $J$.} Define
\begin{equation}
G_2:=(\mathbb E\|\nabla V_\delta(Y_0)\|^2)^{1/2}<\infty, \qquad C_{\mathcal R}^{\rm const} :=(1+M_J)L_\delta\left[\frac12(1+M_J)G_2+\frac23\sqrt{2d}\right].
\label{eq:strong:local:error:constant}
\end{equation}
Then $ \|\mathcal R_0\|_{L^2}\leq C_{\mathcal R}^{\rm const}\,\eta^{3/2}. $ \paragraph{State dependent $J$.} Define
\begin{equation}
\begin{aligned}
G_4&:=(\mathbb E\|\nabla V_\delta(Y_0)\|^4)^{1/4}<\infty, \quad C_{\mathcal R}^{\rm sd}:=C_{\mathcal R}^{\rm const}+L_JG_4\left[\frac12(1+M_J)G_4+\frac{2\sqrt2}{3}[d(d+2)]^{1/4}\right].
\end{aligned}
\label{eq:strong:local:error:sd:constant}
\end{equation}
Then $ \|\mathcal R_0\|_{L^2}\leq C_{\mathcal R}^{\rm sd}\,\eta^{3/2}. $
\end{lemma}

\begin{proof}
See Appendix~\ref{proof:lemma:1:step:error}.
\end{proof}
In the applications below, we denote
\begin{equation}
C_{\mathcal R}:=
\begin{cases}
C_{\mathcal R}^{\rm const},&J\text{ is constant},\\
C_{\mathcal R}^{\rm sd},&J\text{ is state dependent}.
\end{cases}
\label{eq:local:error:choice}
\end{equation}
When Lemma~\ref{lemma:1:step:error} is applied with $f=V_\delta$ and $\pi=\pi_\delta$, the gradient moment bounds in \eqref{eq:moment:target:gradient:bounds} give
\begin{equation}
G_2^2\leq2L_\delta^2M_{2,\delta}+2g_\delta^2, \qquad G_4^4\leq8L_\delta^4M_{4,\delta}+8g_\delta^4.
\label{eq:local:error:target:moments}
\end{equation}

\subsection{Strongly Convex Case}
We first state the convexity and smoothness assumption.
\begin{assumption}
\label{assumption:mu:L}
The function $f$ is strongly convex with parameter $\mu$ and has an $L$ Lipschitz gradient.
\end{assumption}
By Lemma~\ref{lemma:squared:distance:penalty}, $V_\delta$ is strongly convex with parameter $\mu$ and has an $L_\delta$ Lipschitz gradient. Hence
\begin{align}
\langle x,\nabla V_\delta(x)\rangle &=\langle x,\nabla V_\delta(x)-\nabla V_\delta(0)\rangle +\langle x,\nabla V_\delta(0)\rangle \geq\frac{\mu}{2}\|x\|^2 -\frac{\|\nabla V_\delta(0)\|^2}{2\mu}.
\label{eq:strong:convexity:implies:dissipativity}
\end{align}
In particular, applying Lemma~\ref{lemma:uniform:bound} with $m_\delta=\mu/2$ and $b_\delta=\|\nabla V_\delta(0)\|^2/(2\mu)$ verifies $M_{2,\delta},M_{4,\delta}<\infty$ in~\eqref{eq:local:error:target:moments}. Lemma C.5 of~\cite{gurbuzbalaban2024penalized} verifies the exponential integrability hypothesis of Theorem~\ref{thm:dist:perturb} with $\hat\alpha=\mu/4$ and $\hat x$ the minimizer of $f$, so its penalty approximation bound applies. The unique minimizer $x_*$ of $V_\delta$ satisfies the bound in Lemma~\ref{lemma:min}. For $J=0$, Proposition 2.21 of~\cite{gurbuzbalaban2024penalized} gives the reversible constrained sampling complexity. The estimates below retain the penalty layer and treat the full nonreversible drift in the $P_J$ metric.

\begin{lemma}
\label{lemma:L:2}
Suppose Assumptions~\ref{assumption:C},~\ref{assumption:mu:L}, and~\ref{assumption:sg} hold, $J(x)^\top=-J(x)$, and $\|J(x)\|_{\mathrm{op}}\leq M_J$. Let $x_*=\arg\min V_\delta$ and $c=\max\left\{0,\frac{\|\nabla f(0)\|}{\mu R}-1\right\}$, as in Lemma~\ref{lemma:min}. Let $x_0\sim\nu_0$ and assume the finite initial potential gap moment
\begin{equation}
\label{eq:strong:initial:potential}
\mathbb E_{\nu_0} \left[V_\delta(x_0)-V_\delta(x_*)\right]<\infty.
\end{equation}
If
\begin{equation}
\label{eq:strong:moment:stepsize}
\eta \leq \frac{1}{L_\delta\left(1+M_J^2\right)\left(1 + \frac{4L_\delta^2\sigma^2}{\mu^2}\right)} \wedge \frac{1}{2\mu},
\end{equation}
then
\begin{equation}
\sup_k\mathbb{E}\left\|x_k\right\|^2 \leq \frac{4}{\mu}\mathbb E_{\nu_0} \left[V_\delta\left(x_0\right)-V_\delta\left(x_*\right)\right]+ \frac{4(\eta C_{\mathrm{sg},\delta} + L_\delta d)}{\mu^2} + 2(1+c)^2R^2 < \infty,
\end{equation}
with
\begin{equation}
C_{\mathrm{sg},\delta} := 2L_\delta L^2\left(1+M_J^2\right) \sigma^2\left\|x_*\right\|^2+ L_\delta\left(1+M_J^2\right) \sigma^2\|\nabla f(0)\|^2.
\label{eq:strong:sg:constant}
\end{equation}
\end{lemma}

\begin{proof}
See Appendix~\ref{proof:lemma:L:2}.
\end{proof}

The stochastic gradient noise is therefore bounded by
\begin{equation}
\mathbb{E}\left\|\widetilde\nabla f\left(x_k\right)-\nabla f\left(x_k\right)\right\|^2 \leq 2 \sigma^2\left(L^2 \mathbb{E}\left\|x_k\right\|^2+\|\nabla f(0)\|^2\right) \leq \sigma_V^2 d,
\end{equation}
where
\begin{equation}
\label{eq:sigma:V}
\sigma_V^2 := \frac{2\sigma^2}{d} \Bigg[ L^2\left( \frac{4}{\mu}\mathbb E_{\nu_0} \left[V_\delta\left(x_0\right)-V_\delta\left(x_*\right)\right] +\frac{4(\eta C_{\mathrm{sg},\delta} + L_\delta d)}{\mu^2} +2(1+c)^2R^2 \right) +\|\nabla f(0)\|^2 \Bigg].
\end{equation}
Applying Lemma~\ref{lemma:1:step:error} with $f=V_\delta$ and $\pi=\pi_\delta$ controls the local remainder in the synchronous coupling used below.


\begin{theorem}
\label{thm:wasserstein:mu}
Let $x_0\sim\nu_0$ and assume the finite initial potential gap moment~\eqref{eq:strong:initial:potential}. Suppose Assumptions~\ref{assumption:C},~\ref{assumption:mu:L},~\ref{assumption:sg}, and \ref{assumption:PJ:full:drift} hold, $J(x)^\top=-J(x)$, and $\|J(x)\|_{\mathrm{op}}\leq M_J$. If $J$ is state dependent, also assume that it is $L_J$ Lipschitz in the operator norm and satisfies~\eqref{eq:J:compatibility}. Let
\begin{equation}
\sigma_{P,J}:=\sqrt{\lambda_{\max}(P_J)}(1+M_J)\sigma_V,
\label{eq:strong:coupling:sigma}
\end{equation}
where $\sigma_V$ is defined in~\eqref{eq:sigma:V}. If $ 0<\eta\leq1\wedge\frac{\alpha_{P,J}}{2L_{P,J}^2}, $ and~\eqref{eq:strong:moment:stepsize} holds, then for every $K\in\mathbb N_0$,
\begin{align}
\label{eq:strong:W2:constrained}
\mathcal W_2(\nu_K,\pi) &\leq \sqrt{\frac{\lambda_{\max}(P_J)}{\lambda_{\min}(P_J)}} e^{-\alpha_{P,J}K\eta/2}\mathcal W_2(\nu_0,\pi_\delta)
\nonumber\\
&\qquad+\frac{\sqrt\eta}{\sqrt{\lambda_{\min}(P_J)}} \left(\frac{\sqrt{2\lambda_{\max}(P_J)}C_{\mathcal R}}{\alpha_{P,J}} +\frac{\sigma_{P,J}\sqrt d}{\sqrt{\alpha_{P,J}}}\right) +\mathcal W_2(\pi_\delta,\pi).
\end{align}
\end{theorem}

\begin{proof}
See Appendix~\ref{proof:thm:wasserstein:mu}.
\end{proof}

Moreover, if $P_J$ is compatible with the Hessian in the sense that $\mu I\preceq P_J\preceq L_\delta I$, then $\sqrt{\frac{\lambda_{\max}(P_J)}{\lambda_{\min}(P_J)}} \leq\sqrt{\frac{L_\delta}{\mu}}$. Equation~\eqref{eq:strong:coupling:sigma} then gives
\begin{align}
\frac{\sigma_{P,J}}{\sqrt{\lambda_{\min}(P_J)}} &= \sqrt{\frac{\lambda_{\max}(P_J)}{\lambda_{\min}(P_J)}} (1+M_J)\sigma_V\leq\sqrt{\frac{L_\delta}{\mu}}(1+M_J)\sigma_V.
\label{eq:strong:W2:metric:constants}
\end{align}
Combining \eqref{eq:strong:W2:metric:constants} with \eqref{eq:strong:W2:constrained} proves
\begin{align}
\mathcal W_2(\nu_K,\pi) &\leq\sqrt{\frac{L_\delta}{\mu}} \left[e^{-\alpha_{P,J}K\eta/2}\mathcal W_2(\nu_0,\pi_\delta) +\sqrt\eta\left(\frac{\sqrt2C_{\mathcal R}}{\alpha_{P,J}} +\frac{(1+M_J)\sigma_V\sqrt d}{\sqrt{\alpha_{P,J}}}\right)\right] +\mathcal W_2(\pi_\delta,\pi).
\end{align}

\subsection{Nonconvex Case}

In this section, we consider~\eqref{eq:pnsgld} with a nonconvex potential, the reversible penalized stochastic gradient result under smoothness and dissipativity is given in \cite[Proposition~2.23]{gurbuzbalaban2024penalized}. Here the additional $P_J$ condition isolates contraction of the full nonreversible drift.
\begin{assumption}
\label{assumption:non-convex:dissipativity}
There exist $m_\delta>0$ and $0\leq b_\delta<\infty$ such that
\begin{equation}
\left\langle x,\nabla f(x)+\frac1\delta\nabla S(x)\right\rangle \geq m_\delta\|x\|^2-b_\delta, \qquad x\in\mathbb R^d.
\end{equation}
\end{assumption}
Thus $V_\delta=f+S/\delta$ is $(m_\delta,b_\delta)$ dissipative. For the squared distance penalty under Assumptions~\ref{assumption:C} and~\ref{assumption:smooth}, Lemma~\ref{lemma:dissipative} verifies this condition whenever $0<\delta<1/(L+1/2)$. Assumption~\ref{assumption:PJ:full:drift} continues to supply the global contraction and Lipschitz bounds for the full drift. Assume
\begin{equation}
J(x)^\top=-J(x), \qquad \|J(x)\|_{\mathrm{op}}\leq M_J, \qquad x\in\mathbb R^d.
\label{eq:non-convex:J:bounds}
\end{equation}
For constant $J$, $L_J=0$ and skew symmetry gives $\nabla\cdot(J\nabla V_\delta)=\mathrm{Tr}(J\nabla^2V_\delta)=0$ in the distributional sense, so only the $G_2$ bound is required. If $J$ is state dependent, also assume that it is $L_J$ Lipschitz in the operator norm and satisfies~\eqref{eq:J:compatibility}. Assumption \ref{assumption:non-convex:dissipativity} and Lemma~\ref{lemma:uniform:bound} give
$$
M_{2,\delta}=\frac{d+b_\delta}{m_\delta}<\infty, \qquad M_{4,\delta}=\frac{(d+b_\delta)(d+2+b_\delta)}{m_\delta^2}<\infty.
$$
Substitution into~\eqref{eq:local:error:target:moments} verifies the required $G_2$ and $G_4$, so Lemma~\ref{lemma:1:step:error} therefore applies with $f=V_\delta$ and $\pi=\pi_\delta$.

\begin{theorem}
\label{thm:PJ:wasserstein:convergence}
Let $B_J$ be defined by~\eqref{eq:full:drift:definition}, and let $\nu_K$ be the law of the PNSGLD iterate after $K$ steps in~\eqref{eq:pnsgld}. Suppose Assumptions~\ref{assumption:C},~\ref{assumption:smooth},~\ref{assumption:sg}, \ref{assumption:PJ:full:drift}, and \ref{assumption:non-convex:dissipativity} hold, and assume $\mathbb E_{\nu_0}V_\delta(x_0)<\infty$. Suppose in addition that $J$ satisfies \eqref{eq:non-convex:J:bounds}. If $J$ is state dependent, assume that it is $L_J$ Lipschitz in the operator norm and satisfies \eqref{eq:J:compatibility}. If
\begin{equation}
0<\eta\leq 1\wedge\frac{\alpha_{P,J}}{2L_{P,J}^2} \wedge\frac1{L_\delta(1+M_J^2)} \wedge \frac{q_{U,\delta}c_{U,1}} {4L_\delta(1+M_J^2)\sigma^2L^2},
\label{eq:PJ:convergence:stepsize}
\end{equation}
where the last upper bound is interpreted as $+\infty$ if $\sigma L=0$, then Lemma~\ref{lemma:PJ:sg:second:moment} defines $C_x^{\mathrm{sg}}$ and proves~\eqref{eq:PJ:sg:second:moment}. For every $K\in\mathbb N_0$,
\begin{align}
\label{eq:PJ:W2:convergence}
\mathcal W_2(\nu_K,\pi_\delta) &\leq \sqrt{\frac{\lambda_{\max}(P_J)}{\lambda_{\min}(P_J)}} e^{-\alpha_{P,J}K\eta/2} \mathcal W_2(\nu_0,\pi_\delta) + \frac{\sqrt{2\lambda_{\max}(P_J)}C_{\mathcal R}} {\alpha_{P,J}\sqrt{\lambda_{\min}(P_J)}}\sqrt\eta
\nonumber\\
&\qquad +\frac{\sqrt\eta}{\sqrt{\lambda_{\min}(P_J)}} \left( \frac{2\lambda_{\max}(P_J)(1+M_J^2)\sigma^2} {\alpha_{P,J}} \left( L^2C_x^{\mathrm{sg}} +\|\nabla f(0)\|^2 \right) \right)^{1/2}.
\end{align}
The constants $\alpha_{P,J}$ and $L_{P,J}$ are defined in \eqref{eq:strong:full:drift}, $M_J$ in \eqref{eq:non-convex:J:bounds}, $L_\delta$ in \eqref{eq:diss:const}, $c_{U,1}$ and $q_{U,\delta}$ in \eqref{eq:cU1:explicit} and~\eqref{eq:moment:euler:qr:definition}, and $C_{\mathcal R}$ in~\eqref{eq:local:error:choice}. The variance parameter $\sigma$ is the one in Assumption~\ref{assumption:sg}.
\end{theorem}

\begin{proof}
See Appendix~\ref{proof:thm:PJ:wasserstein:convergence}.
\end{proof}

In the strongly convex case, Lemma~\ref{lemma:L:2} yields the uniform second moment used to define $\sigma_V$ in~\eqref{eq:sigma:V}. In the nonconvex case, Lemma~\ref{lemma:PJ:sg:second:moment} instead gives $\sup_k\mathbb E\|x_k\|^2\leq C_x^{\mathrm{sg}}$, which enters the stochastic gradient term in~\eqref{eq:PJ:W2:convergence}. Theorem~\ref{thm:wasserstein:mu} includes the penalty approximation error and bounds the distance to $\pi$, whereas Theorem~\ref{thm:PJ:wasserstein:convergence} bounds the distance to $\pi_\delta$.

\begin{corollary}
\label{cor:PJ:complexity}
Under the assumptions of Theorem~\ref{thm:PJ:wasserstein:convergence}, let $\kappa(P_J)$ be defined by \eqref{eq:PJ:condition:number}. Assume $\mathcal W_2(\nu_0,\pi_\delta)<\infty$. To make $\mathcal W_2(\nu_K,\pi_\delta)\leq \varepsilon$, choose $\eta$ so that~\eqref{eq:PJ:convergence:stepsize} holds and
\begin{equation}
\eta \leq \frac{\lambda_{\min}(P_J)\varepsilon^2} {16\lambda_{\max}(P_J)}\left( \dfrac{C_{\mathcal R}^2}{\alpha_{P,J}^2} +\dfrac{(1+M_J^2)\sigma^2}{\alpha_{P,J}} \left( L^2C_x^{\mathrm{sg}} +\|\nabla f(0)\|^2 \right) \right)^{-1}.
\label{eq:PJ:complexity:accuracy:stepsize}
\end{equation}
If
\begin{equation}
\label{eq:PJ:complexity:initial:large}
2\sqrt{\kappa(P_J)}\mathcal W_2(\nu_0,\pi_\delta)>\varepsilon,
\end{equation}
choose
\begin{equation}
K \geq \frac{2}{\alpha_{P,J}\eta} \log\left( \frac{2\sqrt{\kappa(P_J)}\mathcal W_2(\nu_0,\pi_\delta)}{\varepsilon} \right).
\label{eq:PJ:complexity:K:lower}
\end{equation}
If the reverse inequality in~\eqref{eq:PJ:complexity:initial:large} holds, take $K=0$. Taking $\eta$ to be the minimum of the upper bounds in~\eqref{eq:PJ:convergence:stepsize} and~\eqref{eq:PJ:complexity:accuracy:stepsize}, and $K$ to be the smallest integer allowed above, we obtain
\begin{align}
K &=\widetilde{\mathcal O}\Bigg( \frac1{\alpha_{P,J}}\vee \frac{L_{P,J}^2}{\alpha_{P,J}^2} \vee \frac{L_\delta(1+M_J^2)}{\alpha_{P,J}} \vee \frac{L_\delta(1+M_J^2)\sigma^2L^2} {\alpha_{P,J}q_{U,\delta}c_{U,1}}
\nonumber\\
&\qquad\qquad \vee \frac{\lambda_{\max}(P_J)} {\alpha_{P,J}\lambda_{\min}(P_J)\varepsilon^2} \left[ \frac{C_{\mathcal R}^2}{\alpha_{P,J}^2} +\frac{(1+M_J^2)\sigma^2}{\alpha_{P,J}} \left( L^2C_x^{\mathrm{sg}} +\|\nabla f(0)\|^2 \right) \right] \Bigg).
\label{eq:PJ:complexity:K:order}
\end{align}
If, in addition, $\mathcal W_2(\pi_\delta,\pi)\leq\varepsilon$, then $\mathcal W_2(\nu_K,\pi)\leq2\varepsilon$.
\end{corollary}

\begin{proof}
See Appendix~\ref{proof:cor:PJ:complexity}.
\end{proof}

The $\widetilde{\mathcal O}(\varepsilon^{-2})$ dependence in Corollary~\ref{cor:PJ:complexity} does not establish nonreversible acceleration. For fixed $\delta$ and constants independent of $\varepsilon$, the $\varepsilon$ dependent entry in \eqref{eq:PJ:complexity:K:order} is
\begin{equation}
\widetilde{\mathcal O}\left( \frac{\kappa(P_J)}{\varepsilon^2} \left[ \frac{C_{\mathcal R}^2}{\alpha_{P,J}^3} +\frac{(1+M_J^2)\sigma^2}{\alpha_{P,J}^2} \left( L^2C_x^{\mathrm{sg}}+\|\nabla f(0)\|^2 \right) \right] \right).
\label{eq:PJ:complexity:J:dependence}
\end{equation}
Thus a nonzero $J$ improves the bound for $J\equiv0$ only when the increase in $\alpha_{P,J}$ compensates for the changes in $\kappa(P_J)$, $M_J$, $L_J$ through $C_{\mathcal R}$, and $C_x^{\mathrm{sg}}$. Proposition \ref{prop:PJ:metric:construction} identifies the possible acceleration of this gain at the linearized level. If $\beta_J>\lambda_{\min}(H_*)$, it constructs $P_J$ with a local contraction rate
\begin{equation}
\lambda_{\min}(H_*)<\alpha_J<\beta_J.
\end{equation}
Corollary~\ref{cor:PJ:complexity}, however, depends on the global constant $\alpha_{P,J}$ in Assumption~\ref{assumption:PJ:full:drift}. Hence the local spectral gain yields an iteration complexity improvement only if the same metric $P_J$ verifies that assumption for every $x,y\in\mathbb R^d$ and the complete coefficient in \eqref{eq:PJ:complexity:J:dependence} is smaller than its value for $J\equiv0$. In the following section, we study the choice of the skewed $J$ in the small penalty regime, i.e. $\delta \rightarrow 0$, moreover, we show the acceleration by breaking reversibility.

\subsection{Acceleration by Tuning the Skewed Coefficient}
To isolate the effect of the skew matrix, consider a two dimensional Hessian with eigenvalues $\lambda$ and $\Lambda_\delta$. Here $\lambda>0$ is independent of $\delta$, whereas $\Lambda_\delta>\lambda$ is the eigenvalue enlarged by the penalty term $S/\delta$ and satisfies $\Lambda_\delta\to\infty$ as $\delta\to0$. Define
\begin{equation}
H_\delta:=
\begin{pmatrix}\lambda&0\\0&\Lambda_\delta\end{pmatrix},
\qquad J_a:=\begin{pmatrix}0&a\\-a&0\end{pmatrix}.
\label{eq:PJ:penalty:block:definition}
\end{equation}
The next proposition compares the spectral rates obtained from a fixed skew coefficient $a$ and from a coefficient that varies with $\delta$.
\begin{proposition}
\label{prop:PJ:penalty:hessian:block}
Let $\lambda$, $\Lambda_\delta$, $H_\delta$, and $J_a$ be defined in \eqref{eq:PJ:penalty:block:definition}. For each fixed $a\in\mathbb R$ and all sufficiently small $\delta$, the matrix $(I+J_a)H_\delta$ has two real positive eigenvalues, denoted by
\begin{equation}
\mu_{\delta,-}(a)\leq\mu_{\delta,+}(a).
\end{equation}
As $\delta\to0$, equivalently $\Lambda_\delta\to\infty$,
\begin{equation}
\mu_{\delta,-}(a)\to(1+a^2)\lambda, \qquad \frac{\mu_{\delta,+}(a)}{\Lambda_\delta}\to1.
\label{eq:PJ:penalty:fixed:J:rates}
\end{equation}
For the choice depending on $\delta$
\begin{equation}
a_\delta:=\frac{\Lambda_\delta-\lambda} {2\sqrt{\lambda\Lambda_\delta}},
\label{eq:PJ:penalty:tuned:J}
\end{equation}
the two eigenvalues of $(I+J_{a_\delta})H_\delta$ coincide, and
\begin{equation}
\mu_{\delta,-}(a_\delta) =\mu_{\delta,+}(a_\delta) =\frac{\lambda+\Lambda_\delta}{2}, \qquad \frac{a_\delta}{\Lambda_\delta^{1/2}} \to\frac1{2\sqrt\lambda}.
\label{eq:PJ:penalty:tuned:rate}
\end{equation}
\end{proposition}

\begin{proof}
See Appendix~\ref{proof:prop:PJ:penalty:hessian:block}.
\end{proof}

For the linearized equation $\dot z=-(I+J_a)H_\delta z$, define its spectral decay rate by
\begin{equation}
r_\delta(a):= \min_{\mu\in\sigma((I+J_a)H_\delta)}\operatorname{Re}\mu.
\end{equation}
A larger value of $r_\delta(a)$ means faster asymptotic decay. Proposition~\ref{prop:PJ:penalty:hessian:block} identifies the slowest rate in the reversible, fixed skew, and tuned skew cases, in particular,
\begin{align}
\text{reversible case}:\,\,r_\delta(0) &=\mu_{\delta,-}(0)=\lambda,
\nonumber\\
\text{each fixed }a:\,\,r_\delta(a) &=\mu_{\delta,-}(a)\rightarrow(1+a^2)\lambda,
\nonumber\\
\text{tuned case}:\,\,r_\delta(a_\delta) &=\mu_{\delta,-}(a_\delta) =\frac{\lambda+\Lambda_\delta}{2}\rightarrow\infty.
\end{align}
Thus, relative to the reversible rate, the acceleration factors are
\begin{equation}
\frac{r_\delta(a)}{r_\delta(0)}\rightarrow1+a^2>1 \quad\text{for each fixed }a\neq0, \qquad \frac{r_\delta(a_\delta)}{r_\delta(0)} =\frac{\lambda+\Lambda_\delta}{2\lambda}\rightarrow\infty.
\end{equation}
The ratio of the larger and smaller eigenvalues measures the spectral stiffness, where a smaller ratio indicates less stiffness. As $\delta\to0$, equivalently $\Lambda_\delta\to\infty$,
\begin{equation}
\frac{\mu_{\delta,+}(a)}{\mu_{\delta,-}(a)} \sim\frac{\Lambda_\delta}{(1+a^2)\lambda} \to\infty \quad\text{for each fixed }a, \qquad \frac{\mu_{\delta,+}(a_\delta)} {\mu_{\delta,-}(a_\delta)}=1.
\end{equation}
Hence a fixed nonzero $J_a$ raises the slowest rate from $\lambda$ to the limiting value $(1+a^2)\lambda$, but it does not remove the diverging stiffness. In contrast, the tuned choice $J_{a_\delta}$ equalizes the two eigenvalues and raises the slowest rate to $(\lambda+\Lambda_\delta)/2$. Thus, at the linearized level, $J_{a_\delta}$ simultaneously yields an unbounded acceleration factor relative to the reversible dynamics and completely removes the penalty induced spectral stiffness, whereas any fixed $J_a$ provides only a constant factor acceleration and leaves the stiffness diverging as $\delta\to0$.

We next compute the complexity in the same two dimensional quadratic form. At a critical point $x_\delta^*$ satisfying $\nabla V_\delta(x_\delta^*)=0$, suppose the model is globally quadratic, namely, for every $z\in\mathbb R^2$,
\begin{equation}
V_\delta(x_\delta^*+z)-V_\delta(x_\delta^*) =\frac12z^\top H_\delta z, \qquad \nabla V_\delta(x_\delta^*+z)=H_\delta z, \qquad z\in\mathbb R^2.
\end{equation}
Substituting this expression into the exact gradient update~\eqref{eq:alg} with $Z_k:=x_k-x_\delta^*$ gives
\begin{equation}
Z_{k+1} =\left[I-\eta(I+J_a)H_\delta\right]Z_k +\sqrt{2\eta}\xi_{k+1}.
\end{equation}
This yields the following proposition.
\begin{proposition}[Invariant bias of the Euler discretization]
\label{prop:PJ:penalty:Euler:invariant:bias}
Let $a\neq0$ be fixed independently of $\delta$, and let $a_\delta$ be defined in~\eqref{eq:PJ:penalty:tuned:J}. Denote the invariant laws of the fixed, reversible, and tuned chains by $\gamma_\delta^a$, $\gamma_\delta^0$, and $\gamma_\delta^{a_\delta}$, respectively. For all sufficiently small $\delta$, assume
\begin{equation}
0<\eta<\frac2{\lambda+\Lambda_\delta}.
\label{eq:PJ:penalty:Euler:bias:stepsize}
\end{equation}
Under this condition, all three chains have unique Gaussian invariant laws with mean $x_\delta^*$ and covariances $\Sigma_{\delta,a}$, $\Sigma_{\delta,0}$, and $\Sigma_{\delta,a_\delta}$, respectively. In the inequality below, $\gamma_\delta$ and $\Sigma_\delta$ denote any one of these invariant laws and its corresponding covariance. Then, in all three cases,
\begin{equation}
\mathcal W_2^2(\gamma_\delta,\pi_\delta) \leq\operatorname{Tr}(\Sigma_\delta-H_\delta^{-1}).
\end{equation}
\begin{enumerate}[(a).]
\item For the fixed nonreversible chain,
\begin{equation}
\mathcal W_2(\gamma_\delta^a,\pi_\delta) \leq\left(8(2+a^2)^2 \frac{\eta(\lambda+\Lambda_\delta)}\lambda\right)^{1/2}.
\end{equation}
\item For the reversible chain,
\begin{equation}
\mathcal W_2(\gamma_\delta^0,\pi_\delta) \leq\left(\frac{\eta(\lambda+\Lambda_\delta)}\lambda\right)^{1/2}.
\end{equation}
\item For the tuned nonreversible chain,
\begin{equation}
\mathcal W_2(\gamma_\delta^{a_\delta},\pi_\delta) \leq\left(54\frac{\eta(\lambda+\Lambda_\delta)}\lambda\right)^{1/2}.
\end{equation}
\end{enumerate}
\end{proposition}

\begin{proof}
See Appendix~\ref{proof:prop:PJ:penalty:Euler:invariant:bias}.
\end{proof}

\begin{proposition}
\label{prop:PJ:penalty:Euler:acceleration}
Fix $a\neq0$ independently of $\delta$, let $a_\delta$ be defined in~\eqref{eq:PJ:penalty:tuned:J}, and fix $\varepsilon>0$. Under the globally quadratic setting of Proposition~\ref{prop:PJ:penalty:Euler:invariant:bias}, suppose the initial laws have finite second moments. For the fixed, tuned, and reversible chains, respectively, use the stepsizes
\begin{equation}
\begin{aligned}
\eta_{\mathrm{fixed}} &:=\frac1{\lambda+\Lambda_\delta} \min\left\{1,
\frac{\lambda\varepsilon^2}{32(2+a^2)^2}\right\},\\
\eta_{\mathrm{tuned}} &:=\frac1{\lambda+\Lambda_\delta}
\min\left\{1,\frac{\lambda\varepsilon^2}{216}\right\},\\
\eta_{\mathrm{rev}} &:=\frac1{\lambda+\Lambda_\delta} \min\left\{1,\frac{\lambda\varepsilon^2}{4}\right\}.
\end{aligned}
\label{eq:PJ:penalty:Euler:accuracy:stepsizes}
\end{equation}
All three stepsizes satisfy \eqref{eq:PJ:penalty:Euler:bias:stepsize}. For all sufficiently small $\delta$, the following bounds hold.
\begin{enumerate}[(a).]
\item For the fixed nonreversible chain ($J=J_a$),
\begin{equation}
\mathcal W_2(\nu_K^a,\pi_\delta)\leq\varepsilon \quad\text{if}\quad K\geq \frac{2}{\eta_{\mathrm{fixed}}(1+a^2)\lambda} \log\left[4(2+a^2)\left( 1\vee\frac{\mathcal W_2(\nu_0^a,\gamma_\delta^a)}{\varepsilon} \right)\right].
\label{eq:PJ:penalty:Euler:fixed:complexity}
\end{equation}
\item For the tuned nonreversible chain ($J=J_{a_{\delta}}$),
\begin{equation}
\begin{aligned}
\mathcal W_2(\nu_K^{a_\delta},\pi_\delta)&\leq\varepsilon\quad\text{if}\\
K&\geq\frac{4}{\eta_{\mathrm{tuned}}(\lambda+\Lambda_\delta)}
\Bigg[\frac12\log\frac{\Lambda_\delta}{\lambda}+\log\left(2\left(1+\frac8e\right) \left(1\vee\frac{\mathcal W_2(\nu_0^{a_\delta},\gamma_\delta^{a_\delta})}{\varepsilon}\right)\right)\Bigg].
\end{aligned}
\label{eq:PJ:penalty:Euler:tuned:complexity}
\end{equation}
\item For the reversible chain ($J=J_0=0$),
\begin{equation}
\mathcal W_2(\nu_K^0,\pi_\delta)\leq\varepsilon \quad\text{if}\quad K\geq \frac{1}{\eta_{\mathrm{rev}}\lambda} \log\left[2\left( 1\vee\frac{\mathcal W_2(\nu_0^0,\gamma_\delta^0)}{\varepsilon} \right)\right].
\label{eq:PJ:penalty:Euler:reversible:complexity}
\end{equation}
\end{enumerate}
\end{proposition}

\begin{proof}
See Appendix~\ref{proof:prop:PJ:penalty:Euler:acceleration}.
\end{proof}

For fixed $\lambda$, $a$, and $\varepsilon$, and initial distances growing at most polynomially in $\Lambda_\delta/\lambda$, the three bounds in Proposition~\ref{prop:PJ:penalty:Euler:acceleration} reduce, up to their displayed logarithmic factors, to
$$
K^{\mathrm{rev}} =\widetilde{\mathcal O}\left(\frac{\Lambda_\delta}{\lambda}\right), \qquad K^{\mathrm{fixed}} =\widetilde{\mathcal O}\left( \frac{\Lambda_\delta}{(1+a^2)\lambda}\right), \qquad K^{\mathrm{tuned}} =\widetilde{\mathcal O}\left( \log\frac{\Lambda_\delta}{\lambda}\right).
$$
For varying accuracy, the respective additional factors in the reversible, fixed, and tuned bounds are
$$
\min\{1,\lambda\varepsilon^2/4\}^{-1},\qquad \min\{1,\lambda\varepsilon^2/[32(2+a^2)^2]\}^{-1},\qquad \min\{1,\lambda\varepsilon^2/216\}^{-1}.
$$
As $\delta\to0$ with $a$ fixed, the factor $1+a^2$ describes the asymptotic improvement in the slow spectral rate and in the leading mixing term when the fixed and reversible chains use a common admissible stepsize. Because the accuracy controlled stepsizes $\eta_{\mathrm{fixed}}$ and $\eta_{\mathrm{rev}}$ differ, the displayed complete iteration bounds do not by themselves imply that the fixed chain always requires fewer iterations. Discretization bias, nonnormality, and the conditioning of the fixed eigenvector matrix may offset the spectral gain.

For a globally quadratic model in dimension $d$, choose orthogonal coordinates $Q$ that diagonalize the Hessian and pair its eigenvectors so that
$$
Q^\top\nabla^2V_\delta Q=\operatorname{diag}(\lambda_{\delta,1},\ldots,\lambda_{\delta,d}),\qquad
Q^\top JQ=\operatorname{diag}(J_{a_1},\ldots,J_{a_r},0_{d-2r}).
$$
The Euler transition matrix is then $M=\operatorname{diag}(M_1,\ldots,M_m)$, with each unpaired coordinate a scalar block. Choose one common stepsize that makes every block stable and satisfies the preceding two dimensional hypotheses on each paired block. Write $\Gamma=\bigotimes_i\gamma_i$ for the Euler invariant law and $\pi_\delta=\bigotimes_i\pi_{\delta,i}$ for the quadratic target. For any initial law $\nu_0$ with a finite second moment, synchronous coupling and the product identity for squared Wasserstein distance give
$$
\mathcal W_2(\nu_K,\pi_\delta)\leq\max_i\|M_i^K\|_{\mathrm{op}}\,\mathcal W_2(\nu_0,\Gamma)+\left(\sum_i\mathcal W_2^2(\gamma_i,\pi_{\delta,i})\right)^{1/2}.
$$
Thus, the preceding two dimensional estimates extend to higher dimensions through the largest block contraction bound and the sum of squared invariant biases. The initial blocks may be correlated, and unpaired coordinates remain reversible. For a skew circulant matrix, this extension requires its Fourier planes to be invariant under the quadratic Hessian.


\section{Numerical Experiments}
\label{sec:numerical}
We present numerical experiments to illustrate the convergence behavior
and practical benefits of penalized nonreversible Langevin sampling.
We first consider constrained Bayesian linear and logistic regression,
examining the effect of the skew symmetric perturbation and its matrix
structure. We then investigate sampling from a truncated Laplace
distribution as the penalty parameter decreases, followed by a nonconvex
Bayesian neural network example. Finally, we study the penalty dependent
choice in Proposition~\ref{prop:PJ:penalty:hessian:block} and compare the
iterations required to achieve a common sampling accuracy.

We compare PNSGLD in~\eqref{eq:pnsgld} with PSGLD, obtained by setting
$J=0$. We write $K$ for the number of iterations, $N$ for the number of
particles or chains per repetition, and $b$ for the minibatch size.
The constrained target is denoted by $\pi_{\mathcal C}=\pi$.
We use the squared distance penalty except in the final quadratic
example. Stochastic gradients are conditionally unbiased, and penalty
gradients are evaluated exactly. In the first four experiments, the
methods share initial states, stochastic gradient draws, and Gaussian
innovations within each repetition.

For logistic regression and Laplace sampling, the constant skew symmetric   matrices are constructed from
\begin{align}
& (\bar J_{\rm block})_{2r-1,2r}=1, \qquad
(\bar J_{\rm block})_{2r,2r-1}=-1,
\nonumber\\
& (\bar J_{\rm tri})_{j,j+1}=1,\qquad\qquad
(\bar J_{\rm tri})_{j+1,j}=-1,
\nonumber\\
& (\bar J_{\rm cir})_{j,j+1\ ({\rm mod}\ d)}=-1,\quad
(\bar J_{\rm cir})_{j,j-1\ ({\rm mod}\ d)}=1,
\label{eq:numerical:constant:J:templates}
\end{align}
by setting $J=\gamma\bar J/\|\bar J\|_{\rm op}$. The circulant matrix in the neural network experiment uses the opposite sign convention. The state dependent matrices have divergence free columns and act on triples of coordinates, with zero action on any remaining coordinates. Together with skew symmetry, this divergence free condition gives the compatibility condition~\eqref{eq:J:compatibility}. Lemma~\ref{lemma:pi} then identifies the same invariant distribution $\pi_\delta$ for the corresponding continuous diffusions. The discrete algorithms can still have different errors at a fixed stepsize.

For logistic regression and Laplace sampling, we choose the stepsize using a bound on the expected potential. Let $L_\delta$ be an upper bound on the Lipschitz constant of $\nabla V_\delta$, assume $\|J(x)\|_{\rm op}\leq M_J$, and let $v_g$ bound the conditional mean squared error of the stochastic gradient. Smoothness and skew symmetry give
\begin{align}
\mathbb E[V_\delta(x_{k+1})\mid x_k]
&\leq V_\delta(x_k)
-\eta\left(1-\frac{L_\delta\eta(1+M_J^2)}2\right)
\|\nabla V_\delta(x_k)\|^2+\frac{L_\delta\eta^2(1+M_J^2)}2v_g+d L_\delta\eta.
\label{eq:numerical:energy}
\end{align}
The projection identity in Lemma~\ref{lemma:squared:distance:penalty} and the nonexpansiveness of $I-\mathcal P_{\mathcal C}$ give the sharper bound $\|\nabla S(x)-\nabla S(y)\|\leq2\|x-y\|$. We therefore take $L_\delta=L_f+2/\delta$ and impose $L_\delta\eta(1+M_J^2)\leq1$. Both penalized potentials are strongly convex for fixed $\delta$, so this choice gives a uniform bound on $\mathbb E V_\delta(x_k)$. It controls expected potential growth but does not verify the global contraction condition in Assumption~\ref{assumption:PJ:full:drift}.

Reported uncertainties are rounded to the same decimal place as their corresponding estimates; all comparisons use unrounded values.


\subsection{Constrained Bayesian Linear Regression}
\label{subsec:numerical:linear}
We first consider constrained Bayesian linear regression to illustrate
the effect of a constant nonreversible perturbation on the approach to
the target distribution. This example provides a simple setting for
examining the role of the skew symmetric matrix discussed in
Proposition~\ref{prop:PJ:metric:construction}.

We generate $n=1000$ observations according to
$a_i\sim\mathcal N(0,\Sigma_a)$ and
$y_i=a_i^\top x^\star+\varepsilon_i$, where
\[
\Sigma_a=\begin{pmatrix}1&-0.5\\-0.5&1\end{pmatrix},
\qquad x^\star=(1,1)^\top,
\qquad \varepsilon_i\sim\mathcal N(0,4).
\]
Using the Gaussian prior $\mathcal N(0,10I_2)$ restricted to
$\mathcal C=\{x:\|x\|_1\leq2.2\}$, the constrained posterior has
potential
\begin{equation}
f(x)=\frac18\sum_{i=1}^n(y_i-a_i^\top x)^2
+\frac1{20}\|x\|^2.
\label{eq:linear:hard:posterior}
\end{equation}
We use $J=\left(\begin{smallmatrix}0&0.70\\-0.70&0\end{smallmatrix}\right)$,
selected from regional Hessian bounds of the penalized potential.
The penalty parameter and stepsize are $\delta=0.005$ and
$\eta=3.30\times10^{-4}$. Both methods use $N=1000$ particles,
$K=1000$ iterations, and minibatch size $b=50$, with initial states
drawn from the prior conditioned on $\mathcal C$.
Each particle samples its likelihood minibatch from a balanced
collection formed from 100 data permutations.

We measure convergence using the sliced $W_2$ distance to reference
samples from $\pi_\delta$ and $\pi_{\mathcal C}$, with 1000 reference
points and 64 fixed projection directions~\cite{bonneel2015sliced}.
Figure~\ref{fig:lasso-convergence} reports averages over 30 paired
repetitions. The distances in its two panels are divided by the
reference scales $s_\delta=0.1045$ and $s_{\mathcal C}=0.1017$,
respectively.

\begin{figure}[!htbp]
\centering
\refstepcounter{figure}\label{fig:lasso-convergence}
\includegraphics[width=0.98\textwidth]{\detokenize{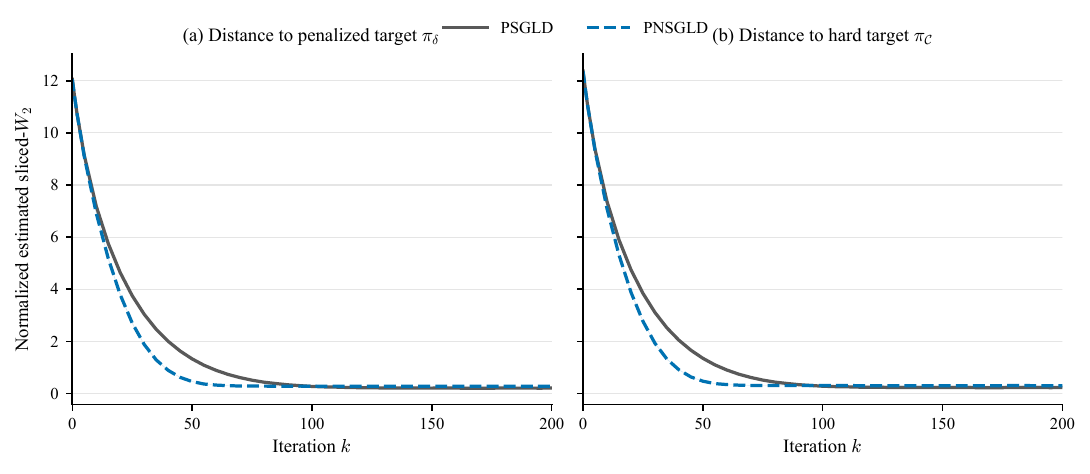}}\\[-0.2em]
{\small \textbf{Figure~\thefigure.} Normalized sliced $W_2$ distances
to $\pi_\delta$ in (a) and $\pi_{\mathcal C}$ in (b).
Curves are averaged over 30 paired repetitions, with pointwise
95\% confidence intervals. The normalization uses $s_\delta$ and
$s_{\mathcal C}$, respectively.}
\end{figure}

Figure~\ref{fig:lasso-convergence} shows a substantially faster initial
decrease in the sampling error for PNSGLD, both relative to the
penalized distribution and relative to the constrained target.
At iteration 50, the unnormalized sliced $W_2$ distance to
$\pi_{\mathcal C}$ is $0.0474$ for PNSGLD and $0.1367$ for PSGLD,
a reduction of approximately $65\%$.
At the common fixed stepsize, PSGLD has the smaller terminal distance
at iteration 1000, namely $0.0224$ compared with $0.0300$ for PNSGLD.
The pronounced improvement during the initial iterations illustrates
how a suitable constant skew perturbation can accelerate the approach
to a constrained posterior.


\subsection{Constrained Bayesian Logistic Regression}
\label{subsec:numerical:logistic}
We next examine how the matrix structure affects predictive performance
in constrained Bayesian logistic regression. We consider the Pima
Indians Diabetes~\cite{smith1988adap} and Titanic~\cite{seabornTitanic}
data sets, with posterior potential
\begin{equation}
f(x)=\sum_{i\in\mathcal D_{\rm tr}}
\left\{\log(1+e^{a_i^\top x})-y_i a_i^\top x\right\}
+\tfrac12\|x\|^2.
\label{eq:logistic:posterior:shared}
\end{equation}
The Gaussian prior makes $f$ strongly convex. Both models have $d=9$;
the Pima model includes an intercept, whereas the Titanic model uses
nine encoded features without an additional intercept. We use 20
stratified training–validation–test splits with proportions
$64/16/20\%$, fitting the preprocessing on the training data.
For Pima, the constraint is $\|x\|_2\leq\sqrt2$.
For Titanic, we consider both this ball and the sublevel constraint
$\sum_j(x_j^2+0.18^2)^{1.2}\leq4$.

All methods use $N=100$ particles, $K=400$ iterations, and minibatch
size $b=30$, with uniform initialization in the unit ball.
The factor $n_{\rm tr}/b$ is applied to the likelihood gradient.
In addition to the constant matrices
in~\eqref{eq:numerical:constant:J:templates}, we use
\begin{equation}
[J(x)v]_{(r)}=\frac{\gamma}{\sqrt3}
\tanh(x_{(r)}/R_s)\times v_{(r)},\qquad R_s=2,
\label{eq:numerical:state:J:template}
\end{equation}
where $(r)$ denotes a consecutive triple of coordinates and $\tanh$
is applied componentwise. This bounded, Lipschitz field satisfies the
compatibility condition in Lemma~\ref{lemma:pi}.

The penalty parameters and skew coefficients are selected in preliminary
validation runs. We use the common stepsize
\[
\eta=\frac{0.9}{(L_f+2/\delta)(1+M_J^2)},
\]
where $M_J$ bounds the operator norms of all perturbations in the
corresponding setting. The curvature bounds are $L_f=300$ for Pima
and $L_f=320$ for Titanic. The penalty parameters are
$8.5\times10^{-4}$ for Pima, $6.0\times10^{-4}$ for Titanic with the
ball constraint, and $2.5\times10^{-4}$ for Titanic with the sublevel
constraint. Table~\ref{tab:logistic-summary} reports the skew coefficients.

We report classification accuracy and predictive negative log likelihood
(NLL). Accuracy is averaged over particles at each recorded iteration,
and terminal test NLL is computed from the mean predictive probabilities
across particles. Early and terminal scores are averaged over
$k=5,10,\ldots,100$ and $k=320,325,\ldots,400$, respectively.
Early validation NLL uses the average of the individual particle scores.
Table~\ref{tab:logistic-summary} reports paired comparisons over the
20 data splits. Early accuracy gains are measured in percentage points
relative to PSGLD, with unadjusted paired 95\% confidence intervals;
Holm's procedure is applied to the twelve early accuracy
tests~\cite{holm1979simple}.
\begin{figure}[!htbp]
\centering
\begin{minipage}{\textwidth}
\centering
\refstepcounter{figure}\label{fig:logistic-accuracy}
\begin{minipage}{0.46\textwidth}
\centering
\includegraphics[width=\textwidth]{\detokenize{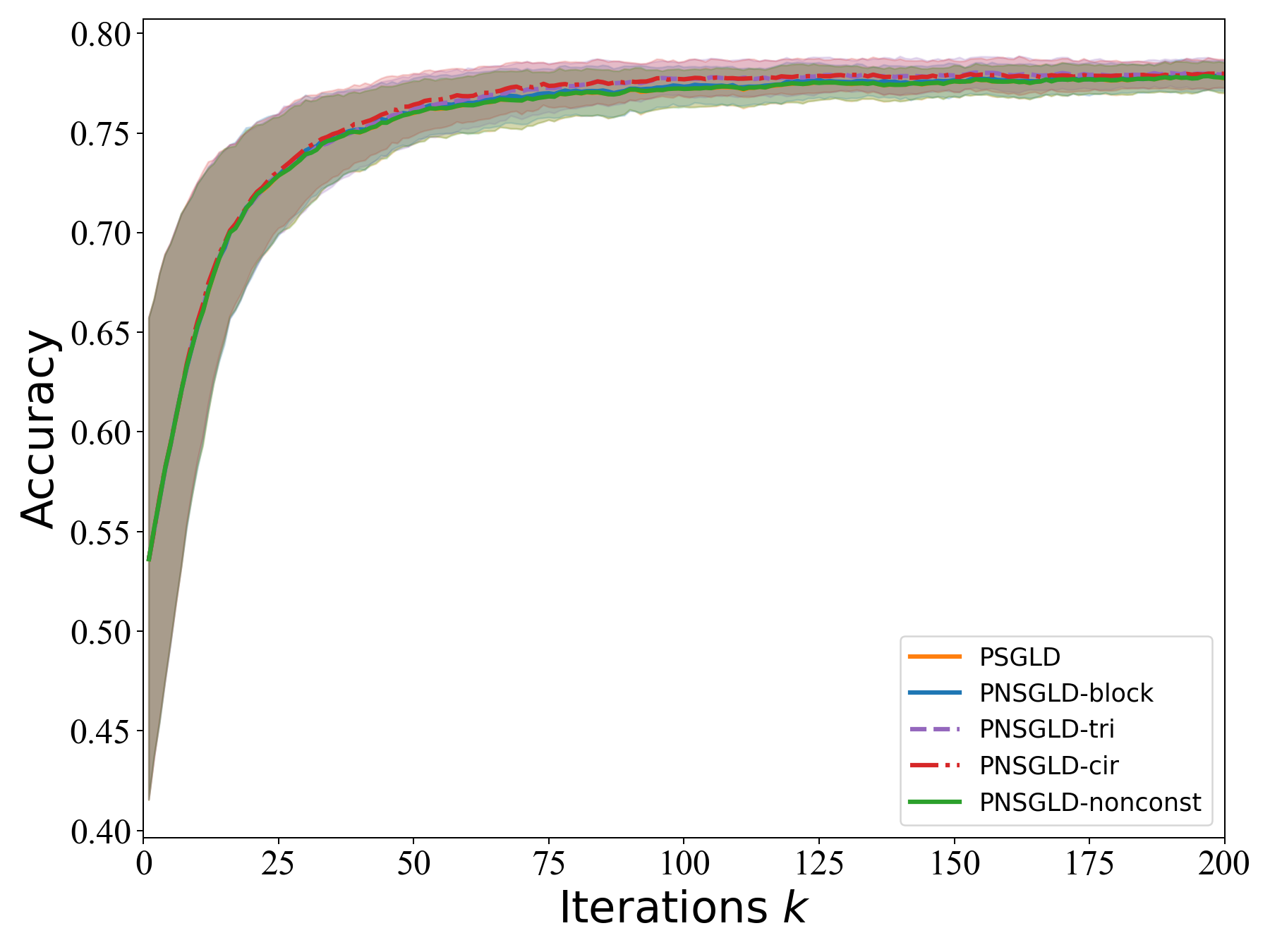}}\\[-0.2em]
{\small (a) Pima training accuracy}
\end{minipage}\hfill
\begin{minipage}{0.46\textwidth}
\centering
\includegraphics[width=\textwidth]{\detokenize{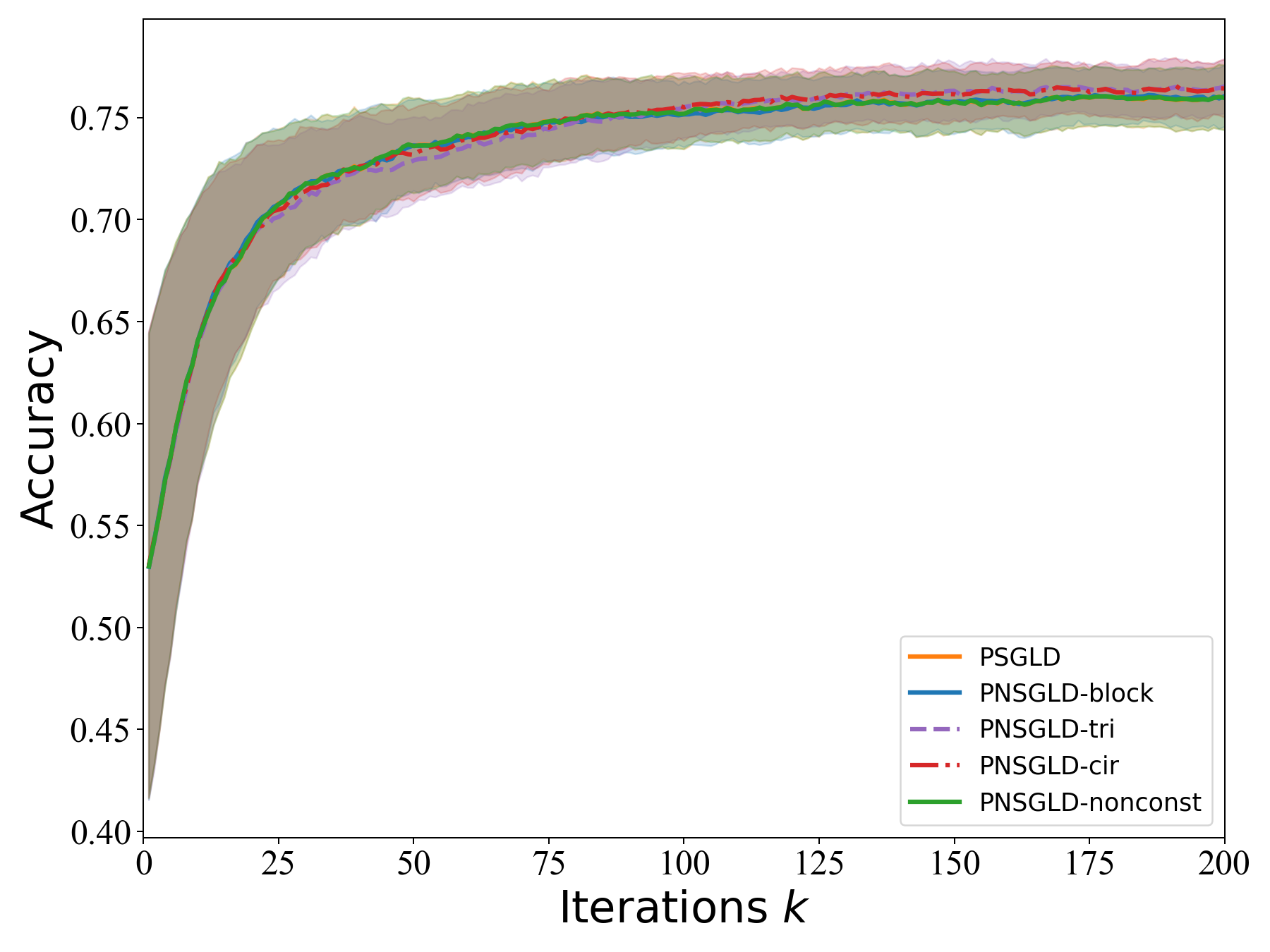}}\\[-0.2em]
{\small (b) Pima test accuracy}
\end{minipage}\\[0.25em]
\begin{minipage}{0.46\textwidth}
\centering
\includegraphics[width=\textwidth]{\detokenize{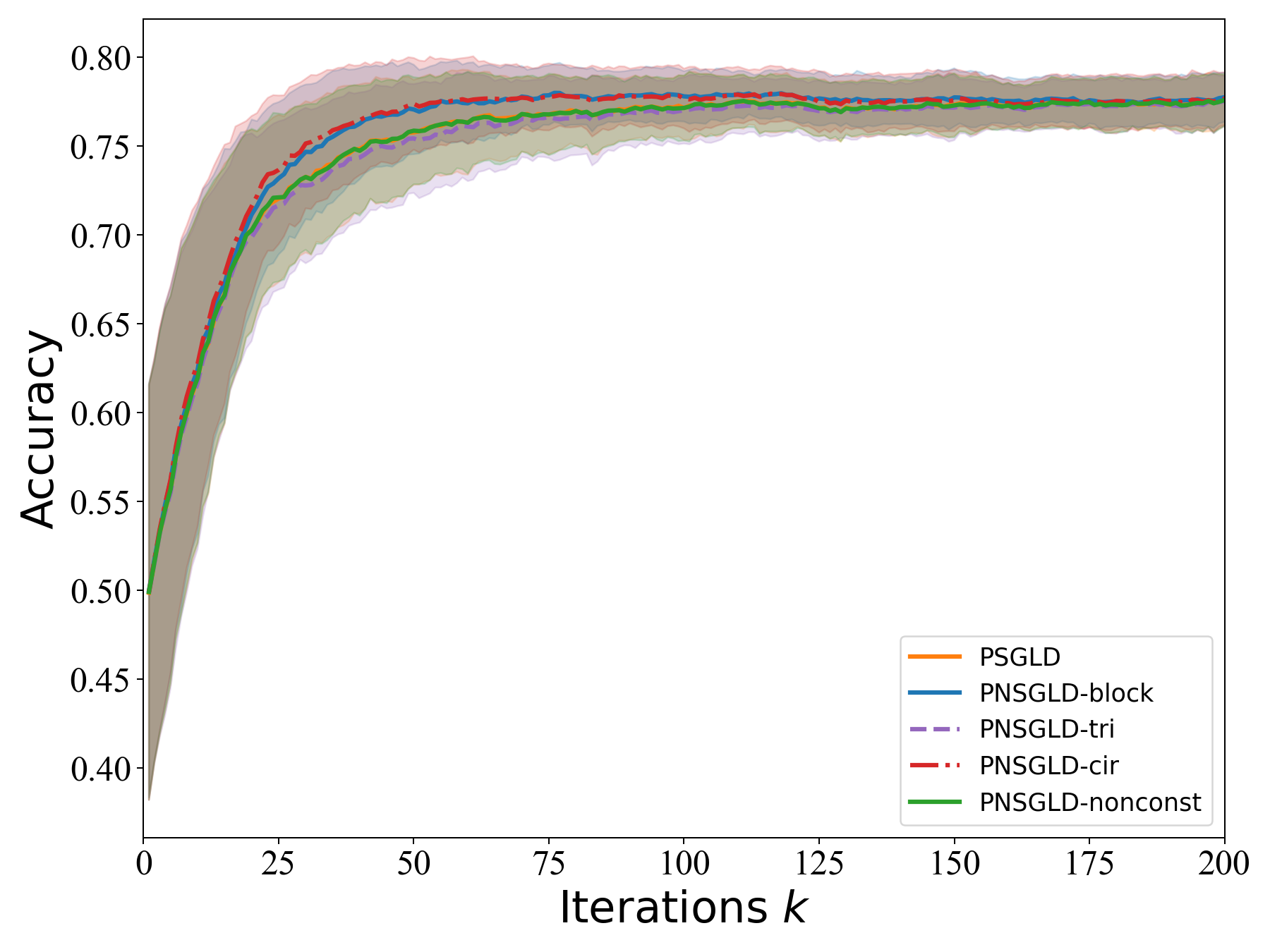}}\\[-0.2em]
{\small (c) Titanic, ball constraint}
\end{minipage}\hfill
\begin{minipage}{0.46\textwidth}
\centering
\includegraphics[width=\textwidth]{\detokenize{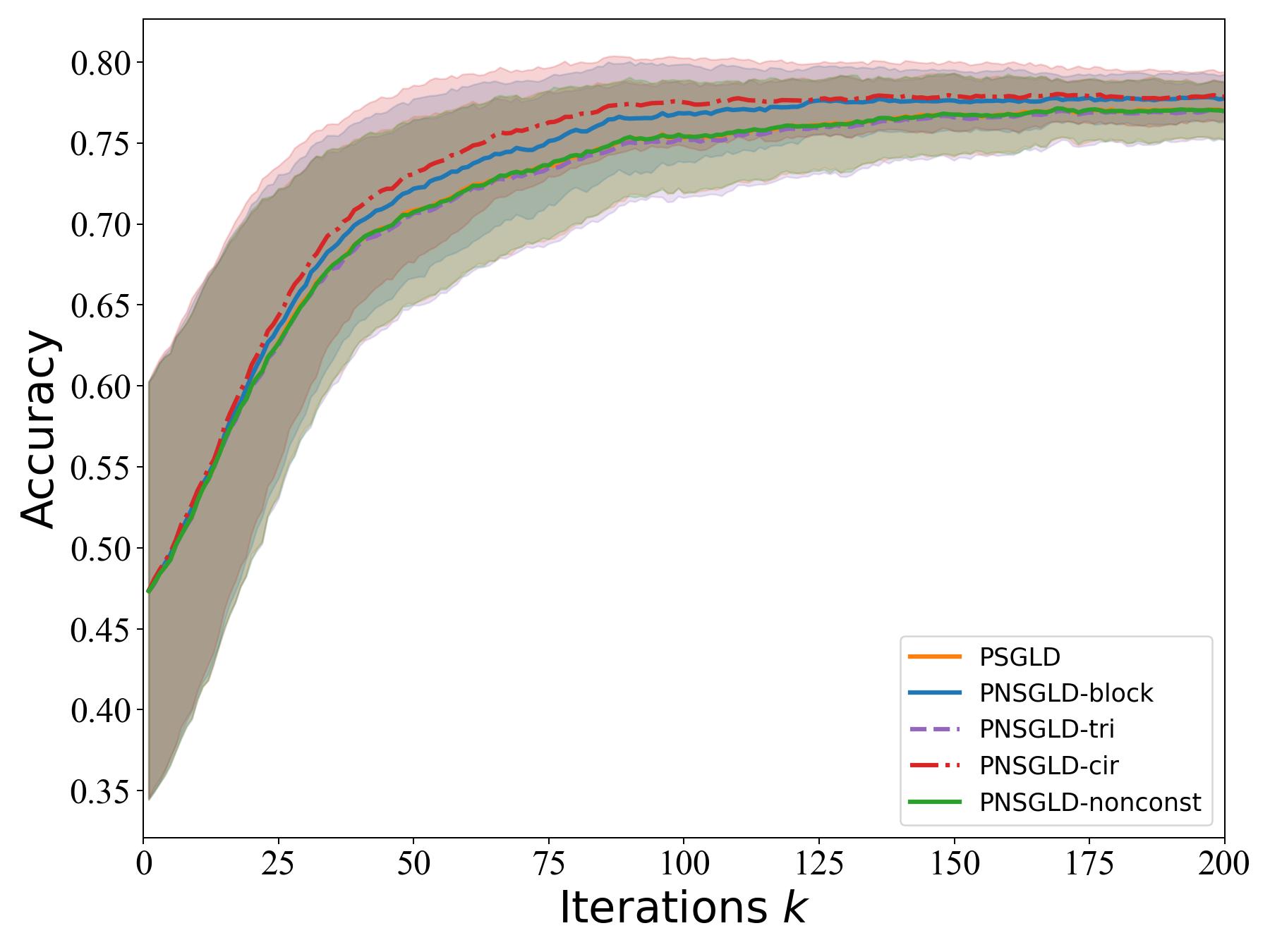}}\\[-0.2em]
{\small (d) Titanic, sublevel constraint}
\end{minipage}\\[0.35em]
{\small \textbf{Figure~\thefigure.} Accuracy trajectories for the fixed
data split. Curves show mean accuracy across
particles, and shading indicates one standard deviation across
particles. Table~\ref{tab:logistic-summary} summarizes all 20 splits.}
\end{minipage}
\end{figure}

\begin{table}[!htbp]
\centering
\caption{Predictive performance over 20 data splits. Early test accuracy
gains relative to PSGLD are in percentage points, with unadjusted paired
95\% confidence-interval half-widths. Terminal accuracy is in percent;
terminal NLL uses particle-averaged predictive probabilities.}
\label{tab:logistic-summary}
\resizebox{\textwidth}{!}{%
\begin{tabular}{llcccc}
\toprule
Setting & Method & $\gamma$ & Early test gain & Terminal accuracy & Terminal NLL \\
\midrule
Pima ball & PSGLD & $0$ & Reference & $76.27$ & $0.4744$ \\
 & Block & $0.175$ & $0.039\pm0.091$ & $76.28$ & $0.4745$ \\
 & Tridiagonal & $0.88$ & $0.175\pm0.268$ & $76.29$ & $0.4746$ \\
 & Circulant & $-0.78$ & $0.411\pm0.274$ & $76.28$ & $0.4746$ \\
 & State dependent & $0.125$ & $0.000\pm0.004$ & $76.27$ & $0.4744$ \\
\midrule
Titanic ball & PSGLD & $0$ & Reference & $78.46$ & $0.4712$ \\
 & Block & $-0.325$ & $0.372\pm0.170$ & $78.47$ & $0.4713$ \\
 & Tridiagonal & $0.1625$ & $-0.025\pm0.065$ & $78.46$ & $0.4712$ \\
 & Circulant & $0.3875$ & $0.379\pm0.186$ & $78.45$ & $0.4714$ \\
 & State dependent & $0.65$ & $0.008\pm0.012$ & $78.46$ & $0.4713$ \\
\midrule
Titanic sublevel & PSGLD & $0$ & Reference & $78.40$ & $0.4690$ \\
 & Block & $-0.325$ & $0.320\pm0.185$ & $78.57$ & $0.4676$ \\
 & Tridiagonal & $0.0625$ & $-0.009\pm0.030$ & $78.41$ & $0.4690$ \\
 & Circulant & $0.5$ & $0.549\pm0.267$ & $78.48$ & $0.4696$ \\
 & State dependent & $0.575$ & $-0.004\pm0.011$ & $78.41$ & $0.4690$ \\
\bottomrule
\end{tabular}}
\end{table}

The block and circulant perturbations give the clearest improvements
in early predictive accuracy. Five of the twelve comparisons remain
significant after Holm adjustment at level $0.05$: the circulant method
on Pima, and the block and circulant methods in both Titanic settings.
Their mean gains range from approximately $0.32$ to $0.55$ percentage
points. The final accuracies remain close across methods, while the
block perturbation also improves both terminal accuracy and NLL under
the Titanic sublevel constraint
(Table~\ref{tab:logistic-summary}).

Halving the common stepsize and doubling the number of iterations
preserves four of the five significant early gains.
These results show that suitable matrix structures improve predictive
performance during the iterations at the common penalty parameters
and stepsizes considered here. We next examine the role of the penalty
parameter directly, using distributional distances to a fixed
constrained target.

\subsection{Penalized Truncated Laplace Sampling}
\label{subsec:numerical:laplace}
We now investigate sampling from a fixed constrained distribution as
the penalty parameter decreases. The approximation result in
Theorem~\ref{thm:dist:perturb} motivates the use of small $\delta$.
To examine the resulting sampling problem, we consider the smoothed
Laplace potential
\[
f(x)=\sum_{j=1}^3\sqrt{x_j^2+\epsilon^2},\qquad
\epsilon=0.05,\qquad
\mathcal C=\{x:\|x\|_2\leq2\}.
\]
At each iteration, each particle samples two coordinates without
replacement and multiplies their gradient components by $3/2$.
We use the constant matrices
in~\eqref{eq:numerical:constant:J:templates}, scaled to have operator
norm $6$, and the state dependent field
\[
J(x)v=\frac{6}{\sqrt{1+\|x\|^2}}\,x\times v.
\]
The latter satisfies $J(x)\nabla S(x)=0$, so it changes the drift from
$f$ while leaving the penalty drift unchanged.

For $\delta\in\{0.05,0.02,0.01,0.005\}$, we use the common stepsize
$\eta_\delta=[37(20+2/\delta)]^{-1}$ and run each method for
$K=16000$ iterations with $N=600$ particles.
Results are averaged over 100 paired repetitions. Within each repetition,
all particles start at one point drawn uniformly from $\mathcal C$.
Thus, the iteration count is fixed, while stronger penalization requires
a smaller stepsize.

Independent reference samples of size $10^4$ are generated by rejection
sampling from $\pi_{\mathcal C}$ and each $\pi_\delta$.
At each recorded iteration, the empirical $W_2$ distance is computed by
optimal matching between 300 particle locations and 300 reference points.
We also compare the reference distributions using 30 comparisons of
subsamples of size 600, including a comparison between two independent
reference samples from $\pi_{\mathcal C}$.

\begin{figure}[!htbp]
\centering
\begin{minipage}{\textwidth}
\centering
\refstepcounter{figure}\label{fig:laplace-w2}
\begin{minipage}{0.46\textwidth}
\centering
\includegraphics[width=\textwidth]{\detokenize{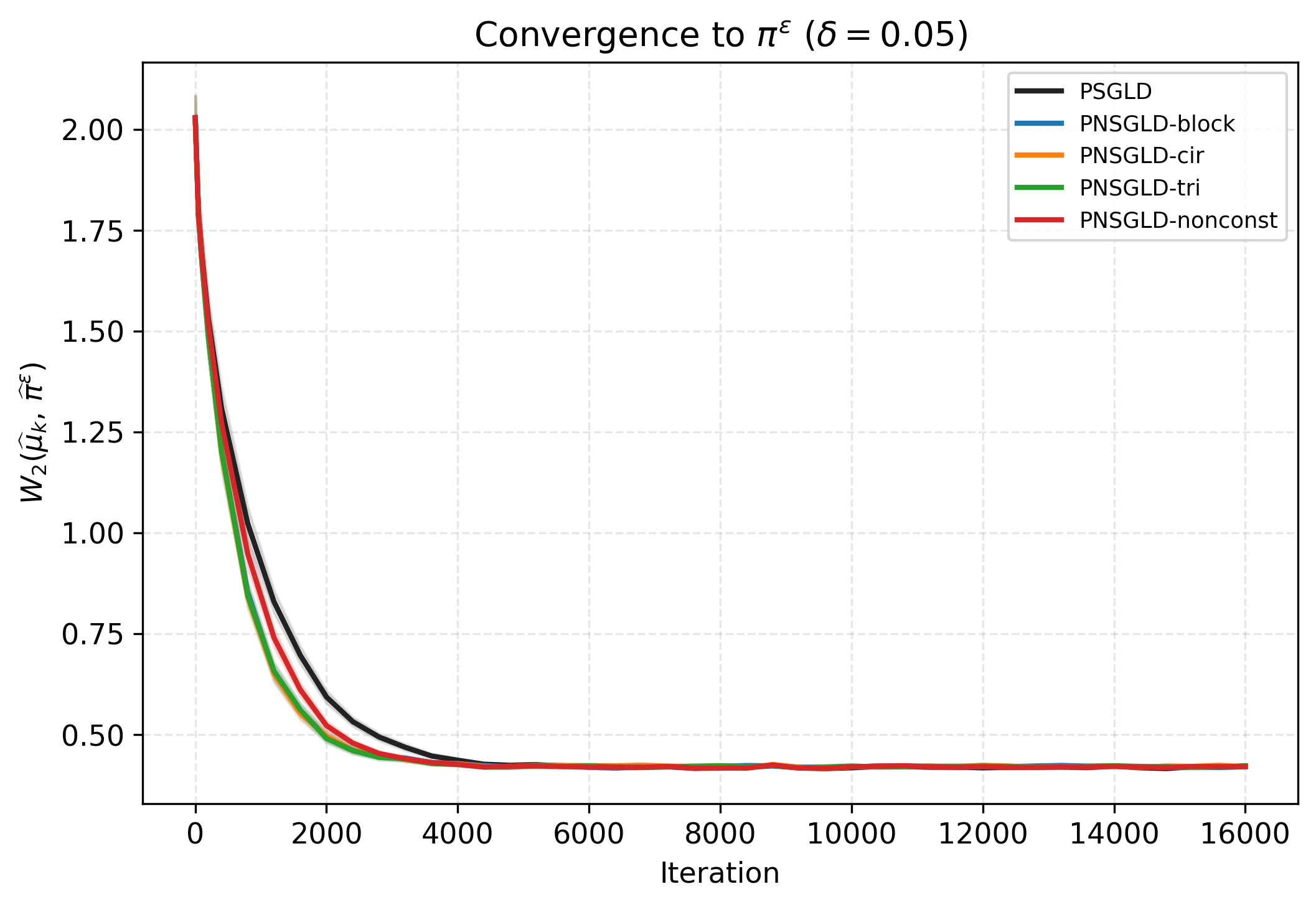}}\\[-0.2em]
{\small (a) $\delta=0.05$}
\end{minipage}\hfill
\begin{minipage}{0.46\textwidth}
\centering
\includegraphics[width=\textwidth]{\detokenize{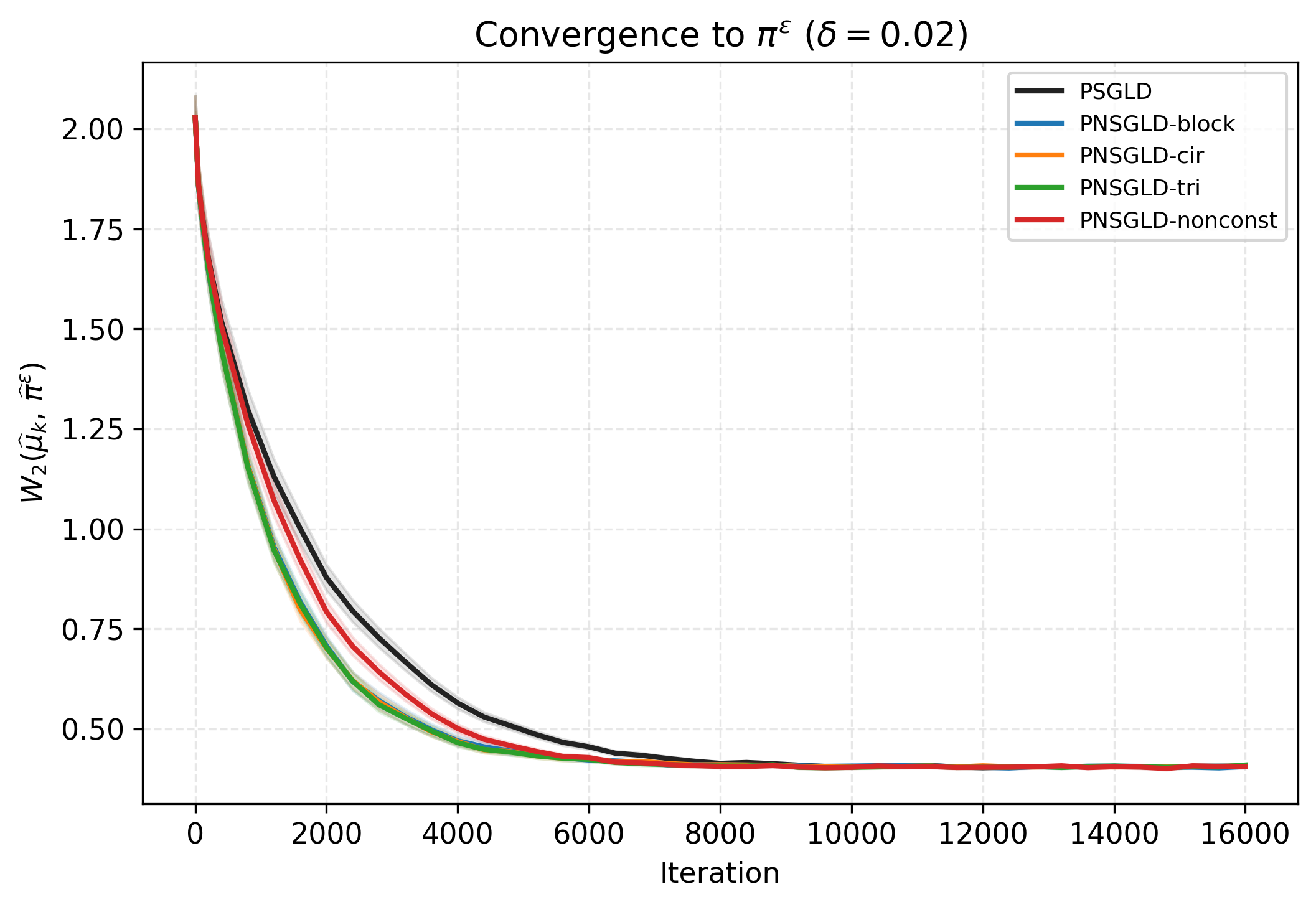}}\\[-0.2em]
{\small (b) $\delta=0.02$}
\end{minipage}\\[0.25em]
\begin{minipage}{0.46\textwidth}
\centering
\includegraphics[width=\textwidth]{\detokenize{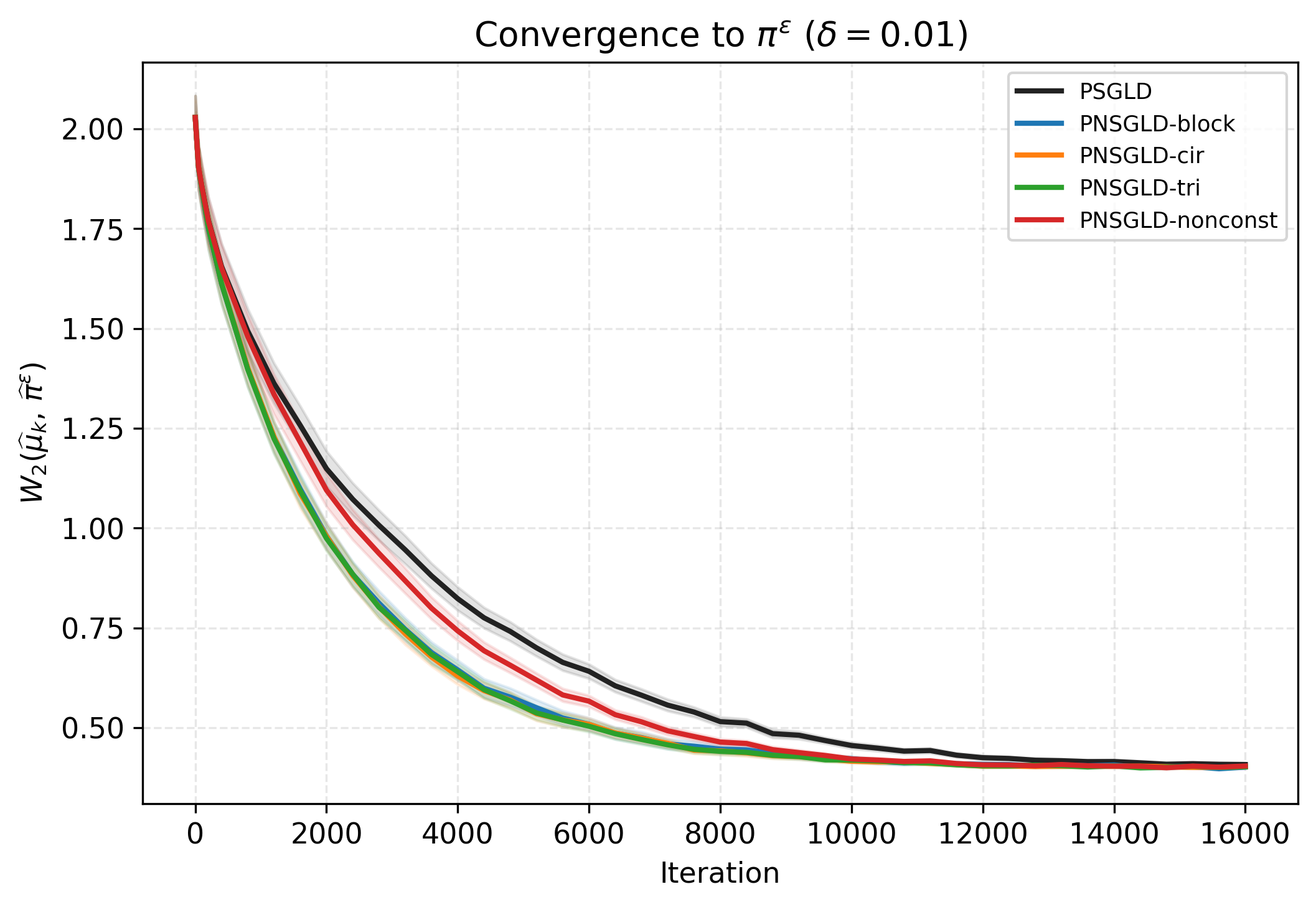}}\\[-0.2em]
{\small (c) $\delta=0.01$}
\end{minipage}\hfill
\begin{minipage}{0.46\textwidth}
\centering
\includegraphics[width=\textwidth]{\detokenize{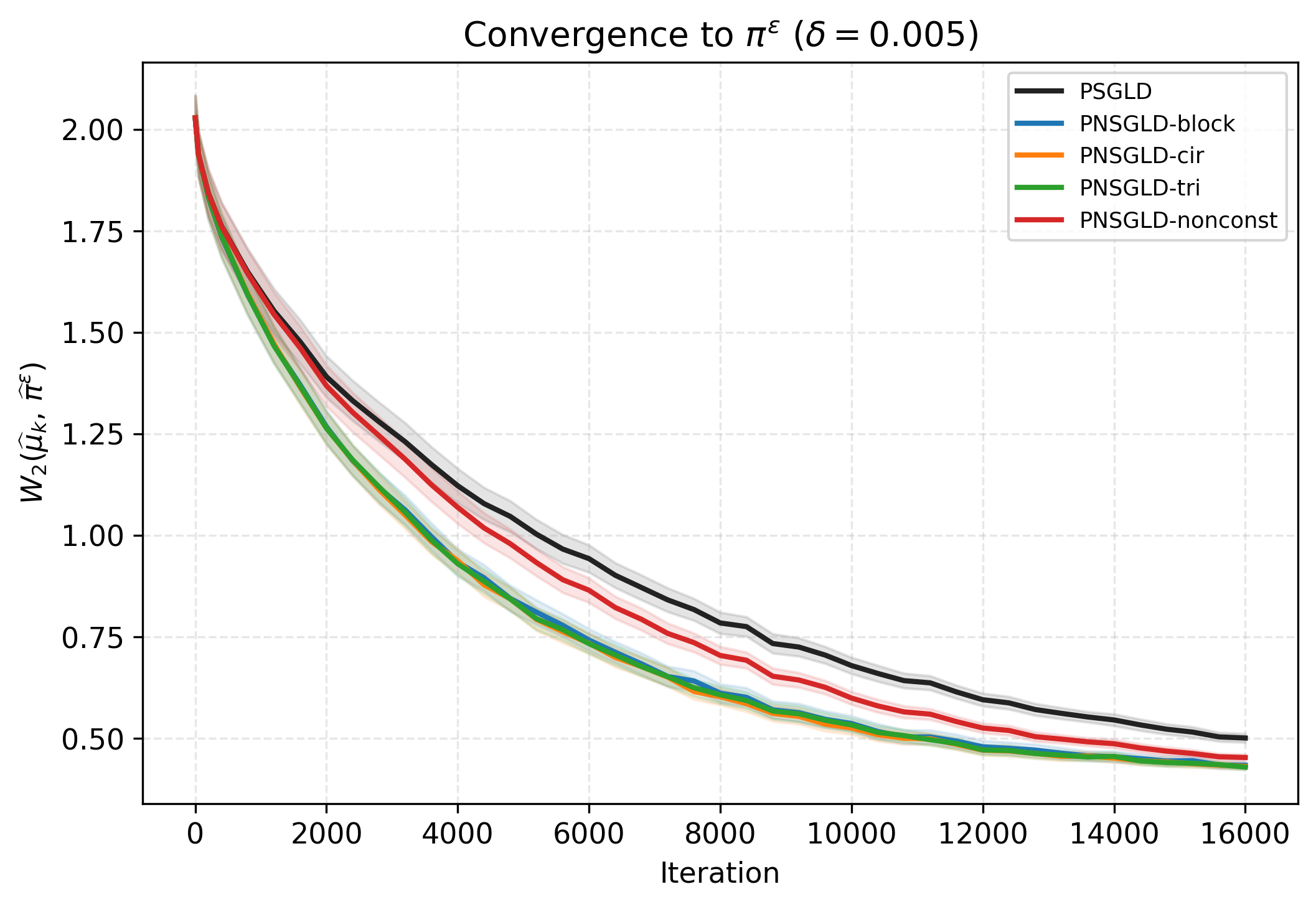}}\\[-0.2em]
{\small (d) $\delta=0.005$}
\end{minipage}\\[0.35em]
{\small \textbf{Figure~\thefigure.} Empirical $W_2$ distance to the
constrained reference at four penalty values, using 300 points per
sample. Curves are averages over 100 paired repetitions, with pointwise
approximate 95\% confidence intervals. All methods use the same
stepsize within each panel.}
\end{minipage}
\end{figure}

\begin{figure}[!htbp]
\centering
\refstepcounter{figure}\label{fig:laplace-density}
\begin{minipage}{0.30\textwidth}
\centering
\includegraphics[width=\textwidth]{\detokenize{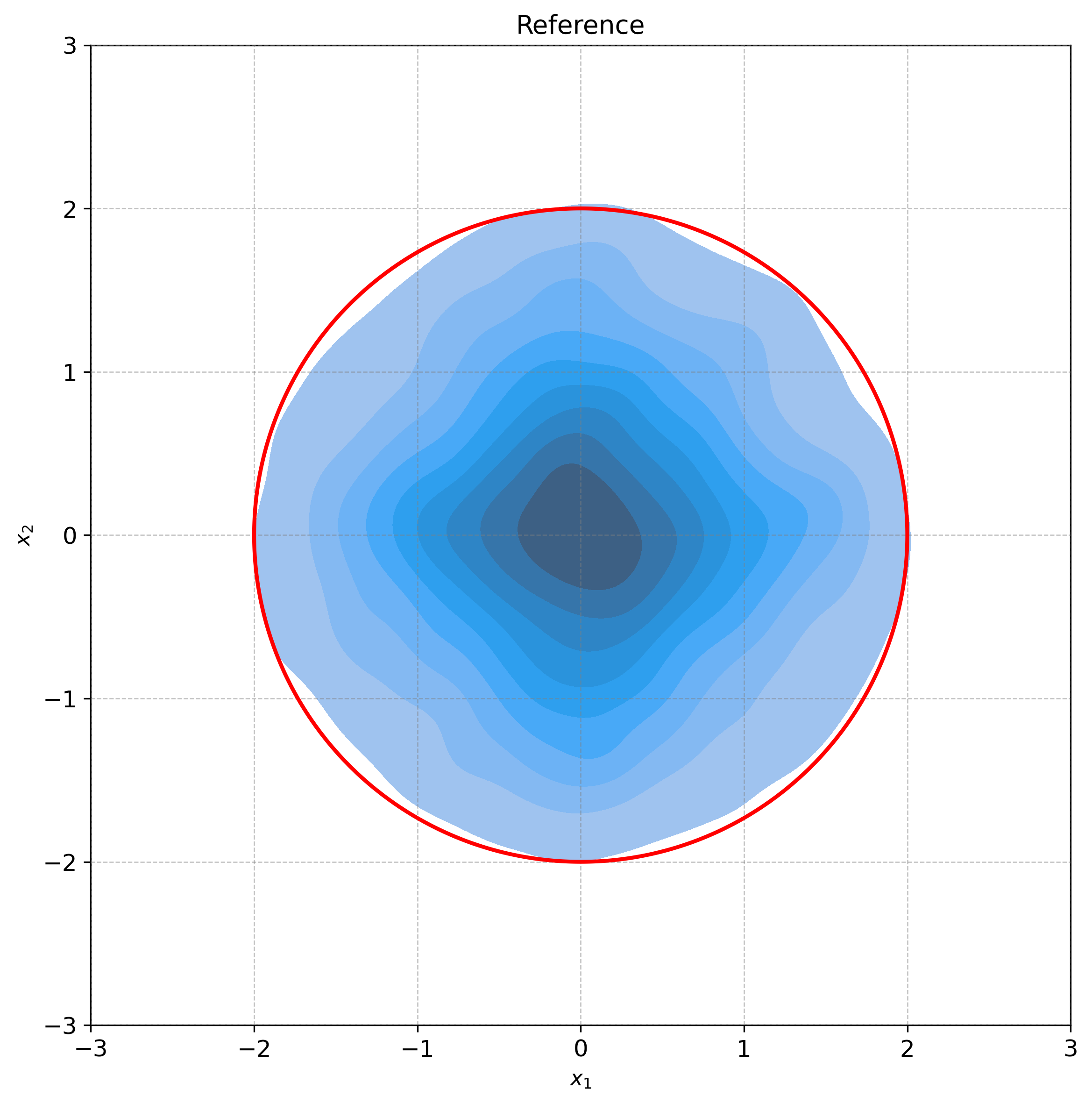}}\\[-0.2em]
{\small (a) Constrained reference}
\end{minipage}\hfill
\begin{minipage}{0.30\textwidth}
\centering
\includegraphics[width=\textwidth]{\detokenize{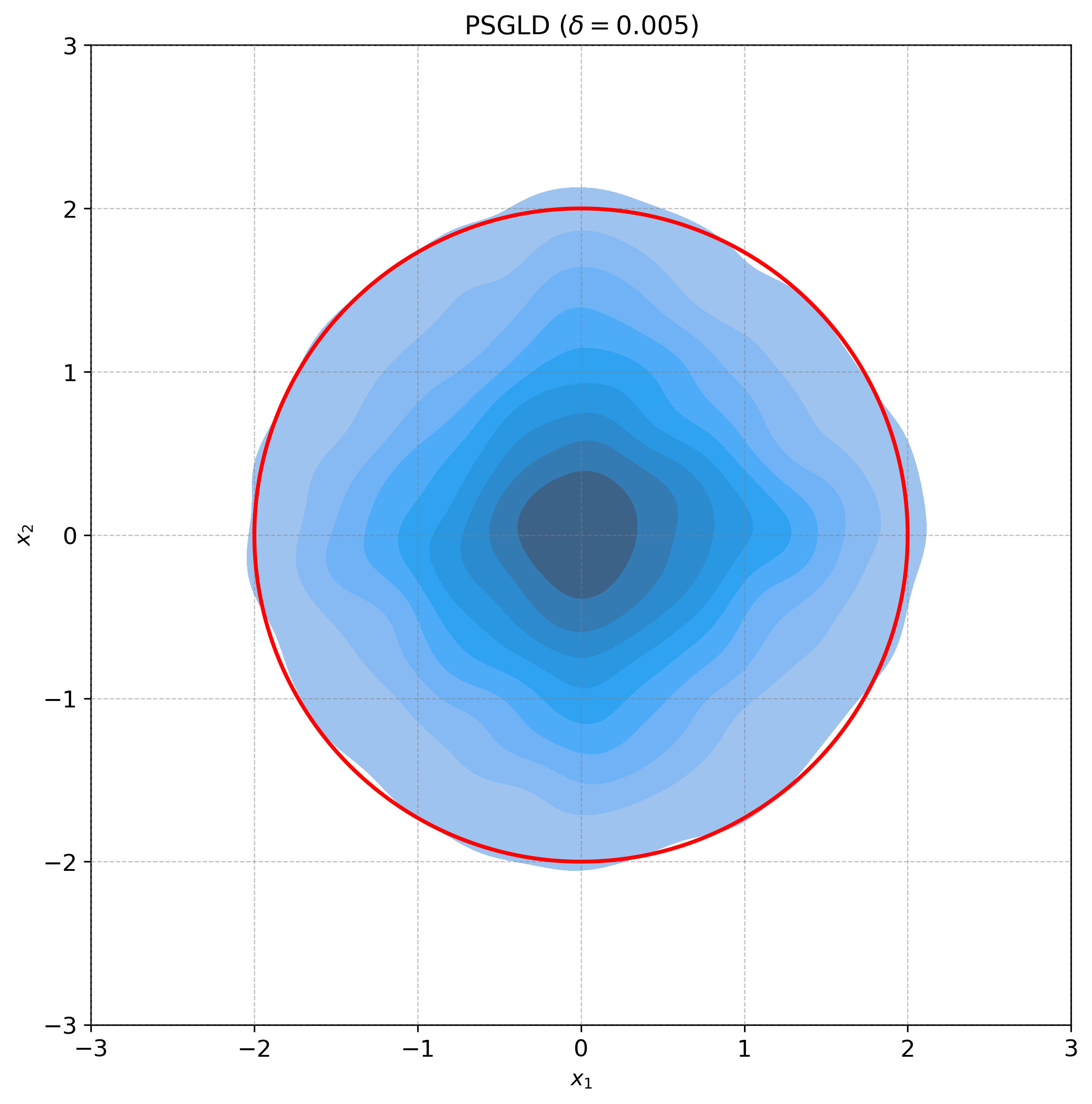}}\\[-0.2em]
{\small (b) PSGLD}
\end{minipage}\hfill
\begin{minipage}{0.30\textwidth}
\centering
\includegraphics[width=\textwidth]{\detokenize{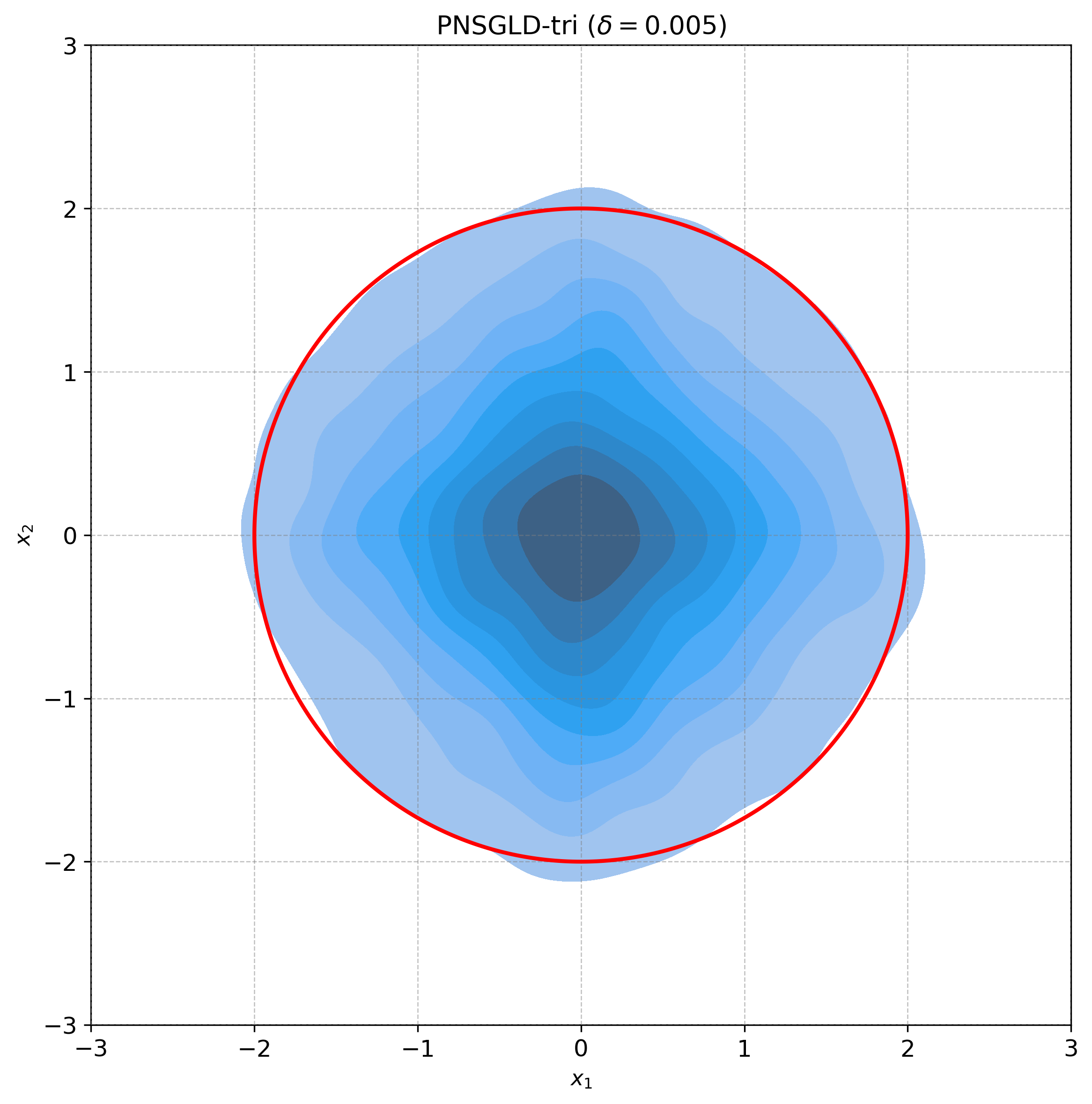}}\\[-0.2em]
{\small (c) PNSGLD (tridiagonal)}
\end{minipage}\\[0.35em]
{\small \textbf{Figure~\thefigure.} Two dimensional marginal density
estimates for the constrained reference and the terminal samples at
$\delta=0.005$. Panel (a) uses $10^4$ reference points; panels (b) and
(c) use $10^4$ terminal particles subsampled from the pooled runs.
The circle indicates the projected constraint boundary.}
\end{figure}

\begin{figure}[!htbp]
\centering
\refstepcounter{figure}\label{fig:laplace-penalty}
\includegraphics[width=0.62\textwidth]{\detokenize{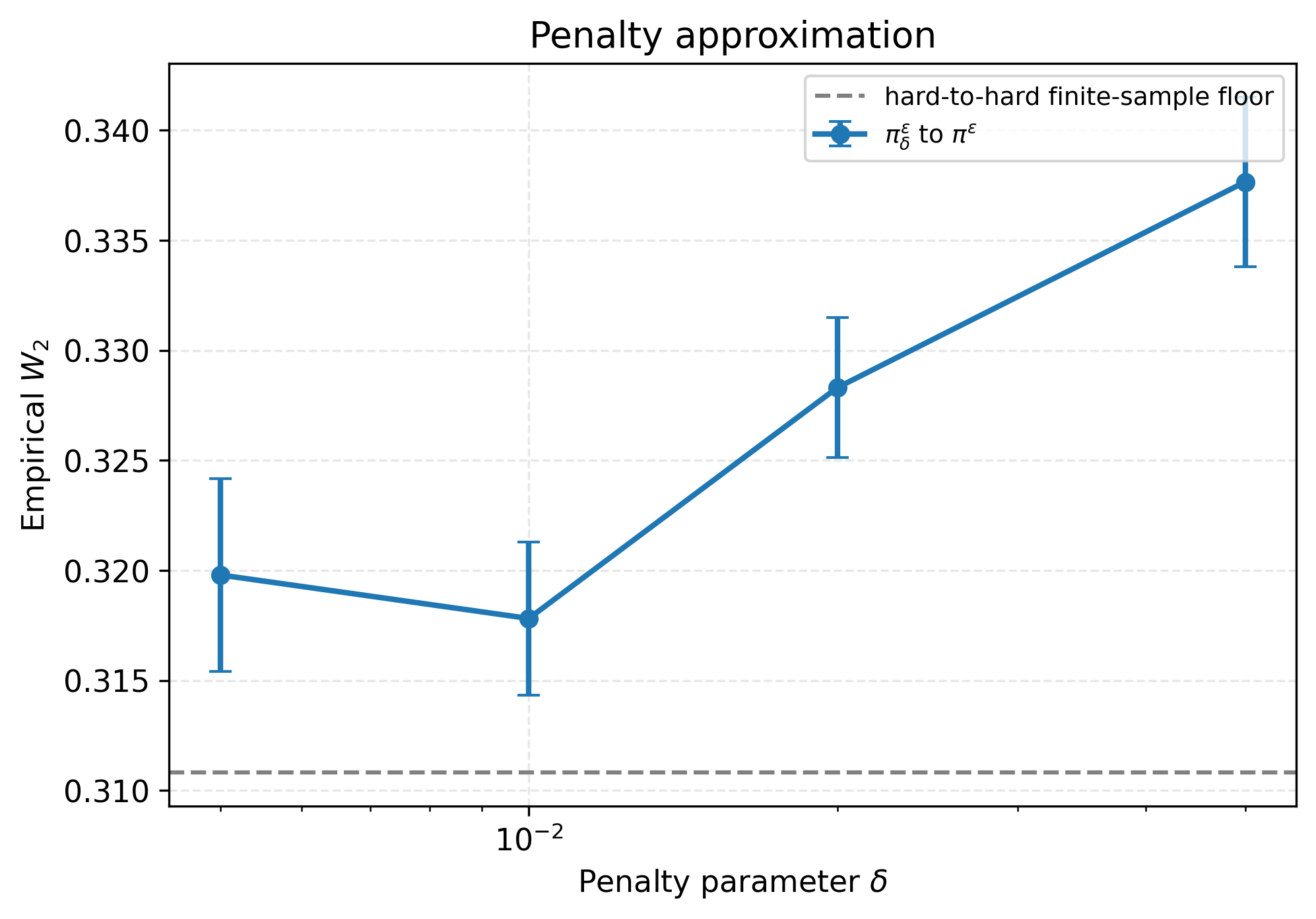}}\\[-0.1em]
{\small \textbf{Figure~\thefigure.} Empirical $W_2$ distances between
reference samples from $\pi_\delta$ and $\pi_{\mathcal C}$, using
600 points per sample. Each point averages 30 comparisons, with
approximate 95\% confidence intervals. The dashed line gives the
comparison between two independent constrained reference samples.}
\end{figure}

\begin{table}[!htbp]
\centering
\caption{Empirical $W_2$ distance to the constrained reference after
$K=16000$ iterations, using 300 points per sample. Entries are means
with approximate 95\% confidence-interval half-widths over 100
repetitions. All methods use the same stepsize within each row.}
\label{tab:laplace-final}
\resizebox{\textwidth}{!}{%
\begin{tabular}{cccccc}
\toprule
$\delta$ & PSGLD & Block & Tridiagonal & Circulant & State dependent \\
\midrule
$0.05$ & $0.4229\pm0.0041$ & $0.4211\pm0.0036$ & $0.4227\pm0.0042$ & $0.4211\pm0.0039$ & $0.4212\pm0.0041$ \\
$0.02$ & $0.4075\pm0.0042$ & $0.4063\pm0.0037$ & $0.4099\pm0.0038$ & $0.4078\pm0.0039$ & $0.4066\pm0.0040$ \\
$0.01$ & $0.4081\pm0.0047$ & $0.4018\pm0.0043$ & $0.4023\pm0.0038$ & $0.4024\pm0.0037$ & $0.4043\pm0.0044$ \\
$0.005$ & $0.5015\pm0.0109$ & $0.4344\pm0.0092$ & $0.4298\pm0.0073$ & $0.4316\pm0.0076$ & $0.4535\pm0.0081$ \\
\bottomrule
\end{tabular}}
\end{table}

Figure~\ref{fig:laplace-w2} shows that nonreversible perturbations
accelerate the decrease in the empirical sampling error.
The improvement is particularly pronounced at $\delta=0.005$:
tridiagonal PNSGLD reduces the terminal mean distance from $0.5015$
to $0.4298$, an improvement of approximately $14.3\%$ under the same
iteration budget. The block and circulant matrices give similar
improvements, and the state dependent method reduces the distance
to $0.4535$ (Table~\ref{tab:laplace-final}).
A comparable reduction is obtained relative to the common penalized
target $\pi_\delta$.

For the larger penalty values, the methods reach similar terminal
distances, while the nonreversible methods have smaller average
distances over the iterations at every tested penalty value.
Figure~\ref{fig:laplace-density} complements these comparisons by
showing the two-dimensional marginals of the constrained reference
and the sampled distributions.

The reference comparison in Figure~\ref{fig:laplace-penalty} also
illustrates the approximation of the constrained target by penalization.
For the smaller penalty values, the empirical distances between
$\pi_\delta$ and $\pi_{\mathcal C}$ are closer to the distance between
two independent samples from $\pi_{\mathcal C}$.
Together with the convergence curves, these results illustrate the
benefit of nonreversible sampling when stronger penalization is used
to approximate the constrained distribution.

\subsection{Constrained Bayesian Neural Networks}
\label{subsec:numerical:bnn}
We next consider constrained Bayesian neural networks to examine the
use of PNSGLD in nonconvex statistical learning. We compare predictive
performance on Rice~\cite{uciRice2019} and MAGIC~\cite{bock2004magic}
using a network with eight hidden units,
\[
h_x(a)=w_2^\top\operatorname{softplus}(W_1a+b_1)+b_2.
\]
The parameter vector concatenates the rows of $W_1$, followed by
$b_1,w_2,b_2$. We use a uniform prior on
$\mathcal C=\{x:\|x\|_2\leq R\}$.
Outside the constraint set, the summed binary log loss is evaluated
through a smooth bounded radial extension $T_\rho(x)$, with
$\rho=0.1R$. This map equals $x$ for $\|x\|\leq R+\rho$ and has
constant radius $R+2\rho$ for $\|x\|\geq R+3\rho$.
The extension leaves the constrained posterior unchanged, and its
derivative is included in the stochastic gradient.

Each data set uses one stratified training--validation--test split
with proportions $60/20/20\%$. The parameter dimensions are 73 for
Rice and 97 for MAGIC. Parameters are selected in preliminary
validation runs. For Rice, we use
$(R,\delta,\eta,b)=(5.0011,7.5032\times10^{-6},3.7516\times10^{-7},64)$;
for MAGIC, we use
$(R,\delta,\eta,b)=(9.5892,9.1953\times10^{-5},9.1953\times10^{-7},128)$.
Each of 20 repetitions uses eight chains and $K=20000$ iterations.
The initial states are $x_0=0.25R\,u/\|u\|$, where the coordinates
of $u$ are independent $\mathrm{Unif}(-1,1)$ variables.
The state dependent perturbation acts on each coordinate triple as
\begin{equation}
[J(x)v]_{(r)}=\frac{\gamma}{\sqrt3}
\tanh(x_{(r)})\times v_{(r)}.
\label{eq:numerical:bnn:state:J}
\end{equation}
This is the construction in~\eqref{eq:numerical:state:J:template}
with $R_s=1$. Table~\ref{tab:bnn-summary} gives the skew coefficients.

We report test accuracy, NLL, and the area under the receiver operating
characteristic curve (AUROC). Figure~\ref{fig:bnn-results} averages
the individual chain scores over eight chains and 20 repetitions.
For Table~\ref{tab:bnn-summary}, each repetition forms a predictive
mixture from all eight chains, retaining one state every 50 iterations
after iteration 10000. The table reports the mean mixture scores over
repetitions and the fraction of these retained states outside
$\mathcal C$.
\begin{figure}[!htbp]
\centering
\refstepcounter{figure}\label{fig:bnn-results}
\begin{minipage}{0.46\textwidth}
\centering
\includegraphics[width=\textwidth]{\detokenize{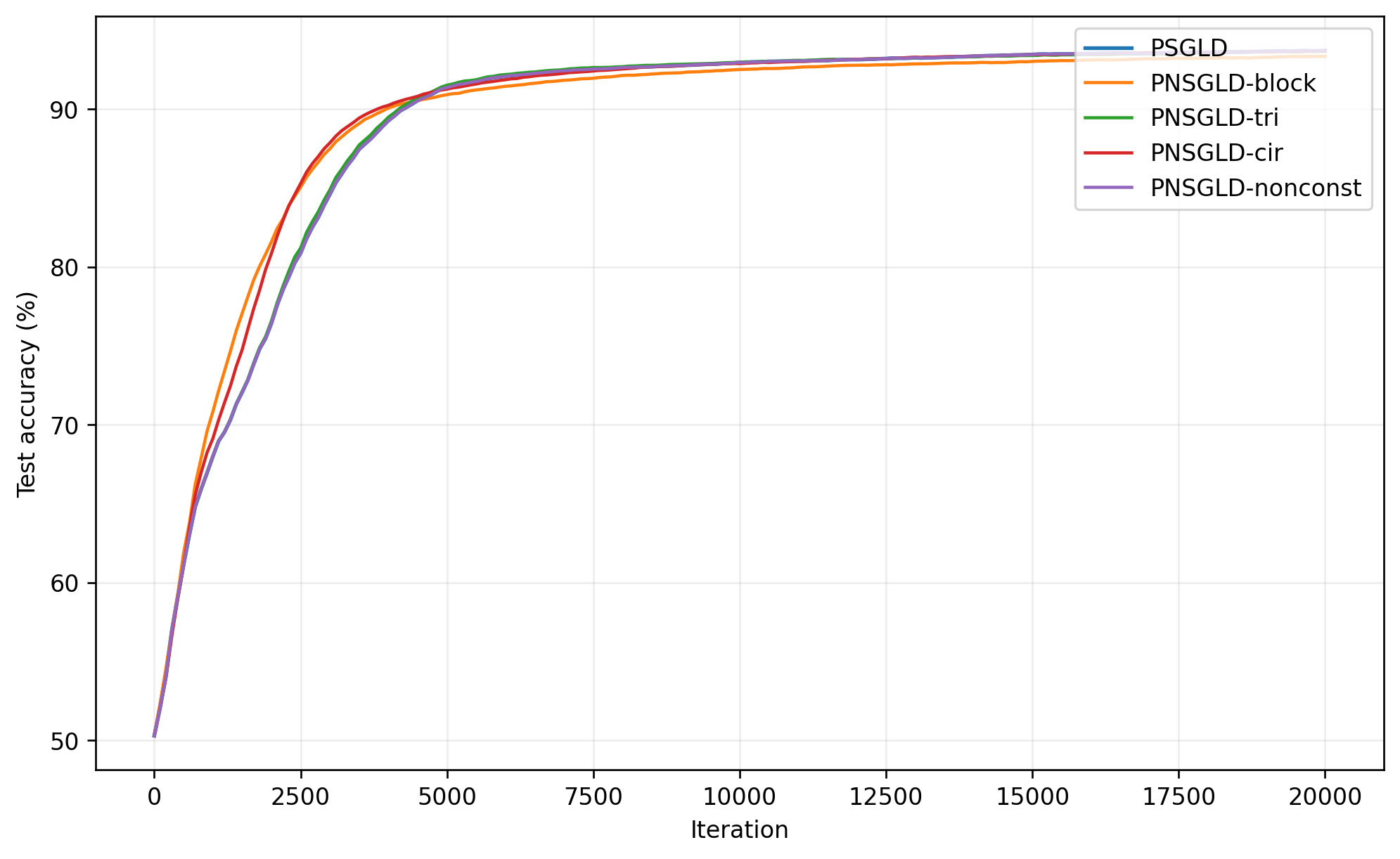}}\\[-0.2em]
{\small (a) Rice test accuracy}
\end{minipage}\hfill
\begin{minipage}{0.46\textwidth}
\centering
\includegraphics[width=\textwidth]{\detokenize{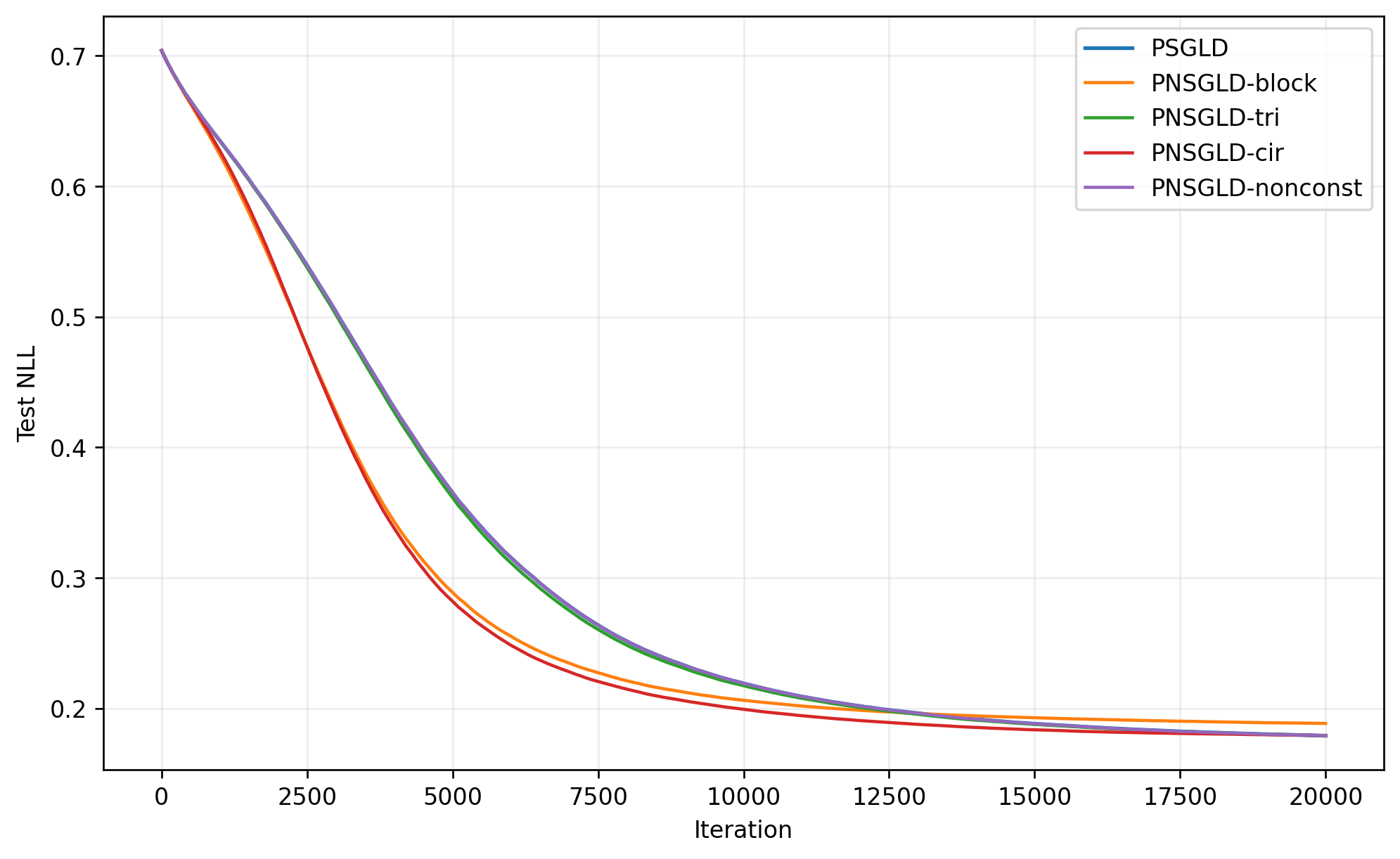}}\\[-0.2em]
{\small (b) Rice test NLL}
\end{minipage}\\[0.25em]
\begin{minipage}{0.46\textwidth}
\centering
\includegraphics[width=\textwidth]{\detokenize{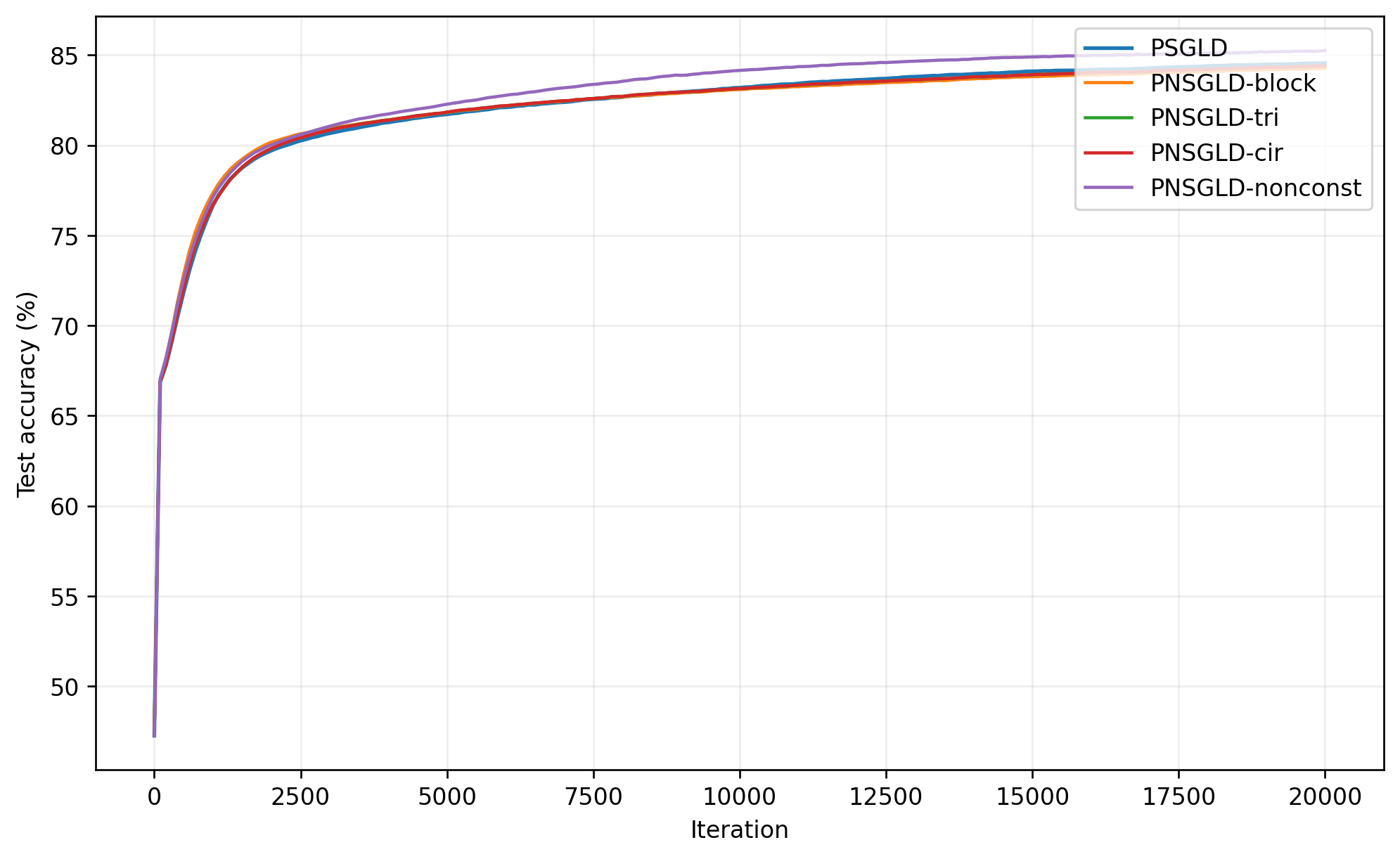}}\\[-0.2em]
{\small (c) MAGIC test accuracy}
\end{minipage}\hfill
\begin{minipage}{0.46\textwidth}
\centering
\includegraphics[width=\textwidth]{\detokenize{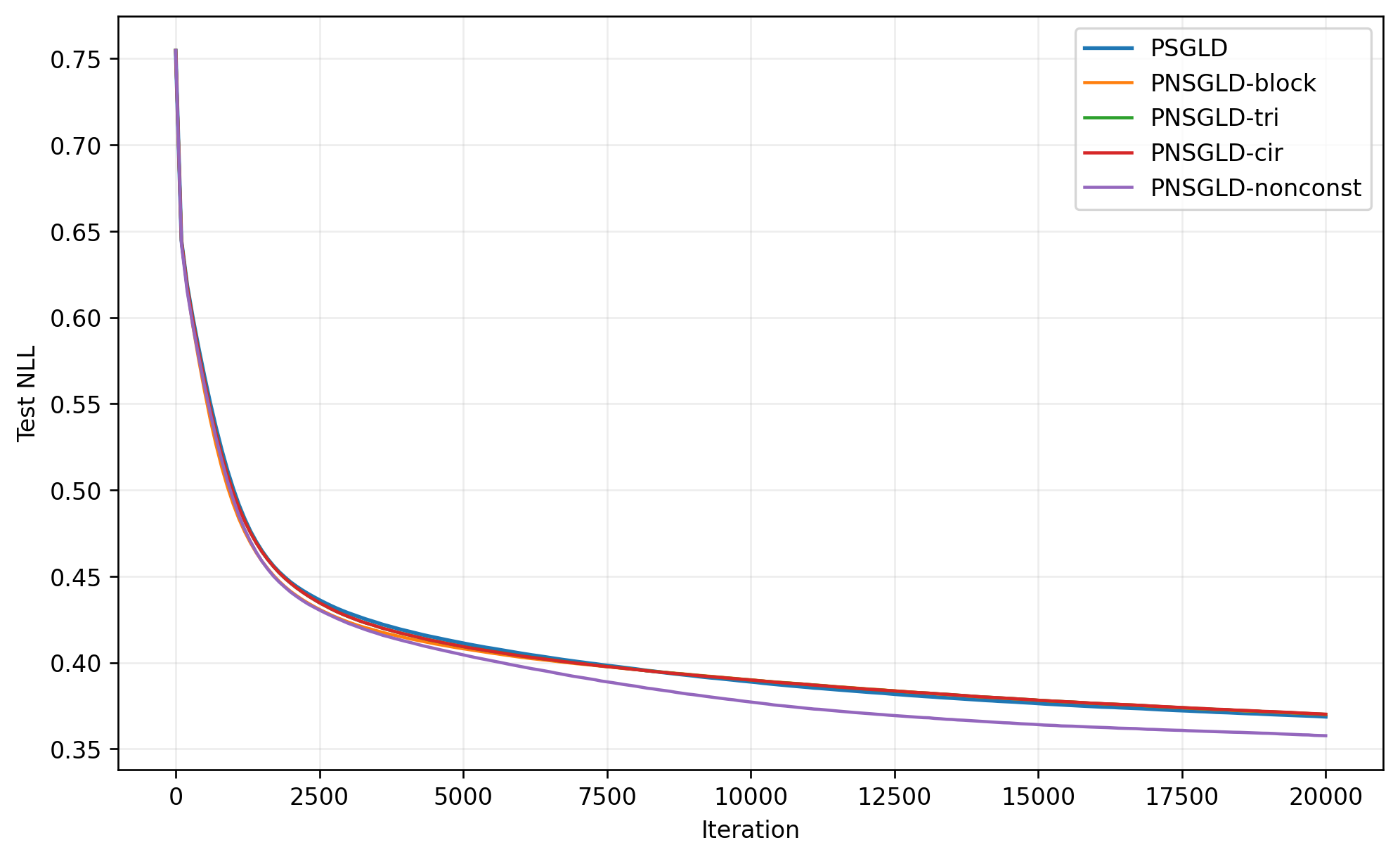}}\\[-0.2em]
{\small (d) MAGIC test NLL}
\end{minipage}\\[0.35em]
{\small \textbf{Figure~\thefigure.} Test accuracy and NLL during the
neural-network experiments, averaged over eight chains and 20
repetitions on a fixed data split. Each chain is scored separately.}
\end{figure}

\begin{table}[!htbp]
\centering
\caption{Predictive mixture scores averaged over 20 repetitions on a
fixed data split. Each mixture uses eight chains and 200 retained
states per chain. Accuracy and the empirical fraction of retained
states outside $\mathcal C$ are reported in percent.}
\label{tab:bnn-summary}
\scriptsize
\resizebox{\textwidth}{!}{%
\begin{tabular}{llccccc}
\hline
Data & Method & $\gamma$ & Test accuracy (\%) & Test NLL & Test AUROC & Outside (\%) \\
\hline
Rice & PSGLD & $0$ & $93.92$ & $0.1894$ & $0.9813$ & $0.00$ \\
 & PNSGLD (block) & $2$ & $93.54$ & $0.1841$ & $0.9794$ & $0.74$ \\
 & PNSGLD (tridiagonal) & $0.5$ & $93.67$ & $0.1887$ & $0.9812$ & $0.00$ \\
 & PNSGLD (circulant) & $4$ & $93.61$ & $0.1822$ & $0.9809$ & $0.54$ \\
 & PNSGLD (state dependent) & $-0.0625$ & $93.92$ & $0.1894$ & $0.9813$ & $0.00$ \\
\hline
MAGIC & PSGLD & $0$ & $84.48$ & $0.3724$ & $0.8977$ & $0.00$ \\
 & PNSGLD (block) & $1$ & $84.09$ & $0.3728$ & $0.8965$ & $0.00$ \\
 & PNSGLD (tridiagonal) & $1$ & $84.29$ & $0.3728$ & $0.8974$ & $0.00$ \\
 & PNSGLD (circulant) & $1$ & $84.28$ & $0.3728$ & $0.8974$ & $0.00$ \\
 & PNSGLD (state dependent) & $-12$ & $85.24$ & $0.3597$ & $0.9033$ & $0.00$ \\
\hline
\end{tabular}}
\end{table}

State dependent PNSGLD gives the best mean predictive scores on MAGIC
among the methods considered (Table~\ref{tab:bnn-summary}).
It increases accuracy from $84.48\%$ to $85.24\%$, reduces NLL from
$0.3724$ to $0.3597$, and improves AUROC from $0.8977$ to $0.9033$.
The improvement across all three criteria demonstrates the practical
benefit of the state dependent perturbation in this nonconvex model.

On Rice, the circulant perturbation gives the lowest NLL, reducing it
from $0.1894$ to $0.1822$, with accuracy $93.61\%$ compared with
$93.92\%$ for PSGLD. The empirical fraction of retained states outside
$\mathcal C$ is below $1\%$ for every configuration in the table.
Thus, the experiments show predictive gains from suitable constant
and state dependent perturbations, with the strongest overall
improvement obtained on MAGIC.

\subsection{Spectral Acceleration for Small Penalty Parameters}
\label{subsec:numerical:small-delta}
The preceding experiments illustrate the influence of the skew matrix
on sampling and predictive performance. We now examine the
penalty dependent choice prescribed by
Proposition~\ref{prop:PJ:penalty:hessian:block}.
A quadratic model allows us to compare both the spectral rates and
the number of iterations needed to attain the same sampling accuracy,
including the effect of stochastic gradient noise.

Let $m=(1,0)^\top$, $f(x)=\tfrac12\|x-m\|^2$, and $S(x)=x_2^2$.
Then
\[
H_\delta=\operatorname{diag}(1,\Lambda_\delta),\qquad
\Lambda_\delta=1+2/\delta,\qquad
\pi_\delta=\mathcal N(m,H_\delta^{-1}).
\]
The original potential $f$ is fixed, so varying $\delta$ changes only
the curvature introduced by the quadratic penalty.
We compare
\[
J_a=\begin{pmatrix}0&a\\-a&0\end{pmatrix},\qquad
a=0,\quad a=1,\quad
a=a_\delta=\frac{\Lambda_\delta-1}{2\sqrt{\Lambda_\delta}},
\]
using the eight penalty values in
Table~\ref{tab:small-delta-hitting}.
At each iteration, the stochastic gradient is obtained by adding the
average of $b=8$ independent $\mathcal N(0,I_2)$ vectors to $\nabla f$.
With $h_k=x_k-m$ and $B=I+J_a$, the Euler chain is
\[
h_{k+1}=Mh_k+\zeta_{k+1},\qquad
M=I-\eta BH_\delta,\qquad
\operatorname{Cov}(\zeta_{k+1})=Q
=2\eta I_2+\frac{\eta^2}{8}BB^\top.
\]
Its invariant covariance satisfies $C_\infty=MC_\infty M^\top+Q$,
which accounts for both discretization and stochastic gradient noise.
The states at recorded iterations are generated from the exact
Gaussian transitions of this Euler chain.

For each method, a single stepsize fraction $q$ is used across all
penalty values:
\[
\eta=q\min_i\frac{2\operatorname{Re}\mu_i}{|\mu_i|^2},
\qquad \mu_i\in\operatorname{spec}(BH_\delta).
\]
We select $q$ numerically to minimize the median iteration count,
subject to a stationary $W_2$ error at most $0.02$ for every penalty
value. The skew coefficient $a_\delta$ is prescribed by the theory.
Table~\ref{tab:small-delta-config} reports the selected fractions and
the resulting stationary errors.

\begin{figure}[!htbp]
\centering
\refstepcounter{figure}\label{fig:small-delta-mechanism}
\includegraphics[width=0.62\textwidth]{\detokenize{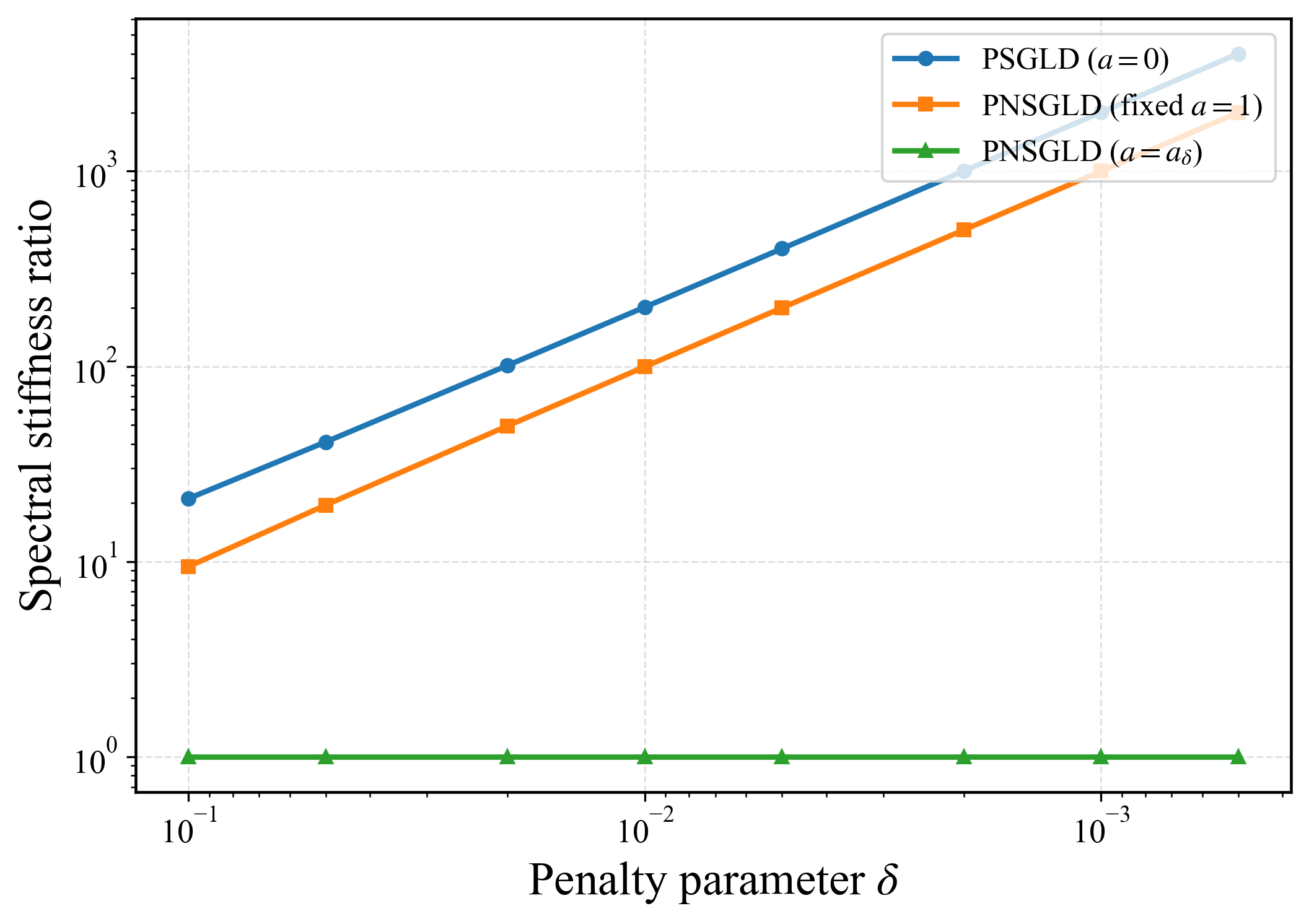}}\\[-0.1em]
{\small \textbf{Figure~\thefigure.} Spectral ratio
$\max_i\operatorname{Re}\mu_i/\min_i\operatorname{Re}\mu_i$
for $(I+J_a)H_\delta$. The penalty dependent choice gives ratio one,
as predicted by Proposition~\ref{prop:PJ:penalty:hessian:block}.}
\end{figure}

We use 12 repetitions of $N=2000$ particles initialized from
$h_0\sim\mathcal N((1,0)^\top,H_\delta^{-1})$.
The initial displacement lies in the direction of curvature one,
and the initial $W_2$ distance equals one for every $\delta$.
For repetition $r$, let $\widehat h_{r,k}$ and $\widehat C_{r,k}$
be the empirical mean and covariance of $h_k$. We compute
\[
(\widehat W^{\rm G}_{2,r,k})^2
=\|\widehat h_{r,k}\|^2+
\operatorname{tr}\!\left[
\widehat C_{r,k}+H_\delta^{-1}
-2(H_\delta^{-1/2}\widehat C_{r,k}H_\delta^{-1/2})^{1/2}
\right].
\]
Figure~\ref{fig:small-delta-w2} averages these estimates over
repetitions. The iteration counts in
Table~\ref{tab:small-delta-hitting} are computed from the exact
Gaussian mean and covariance recursion.

\begin{figure}[!htbp]
\centering
\refstepcounter{figure}\label{fig:small-delta-w2}
\includegraphics[width=\textwidth]{\detokenize{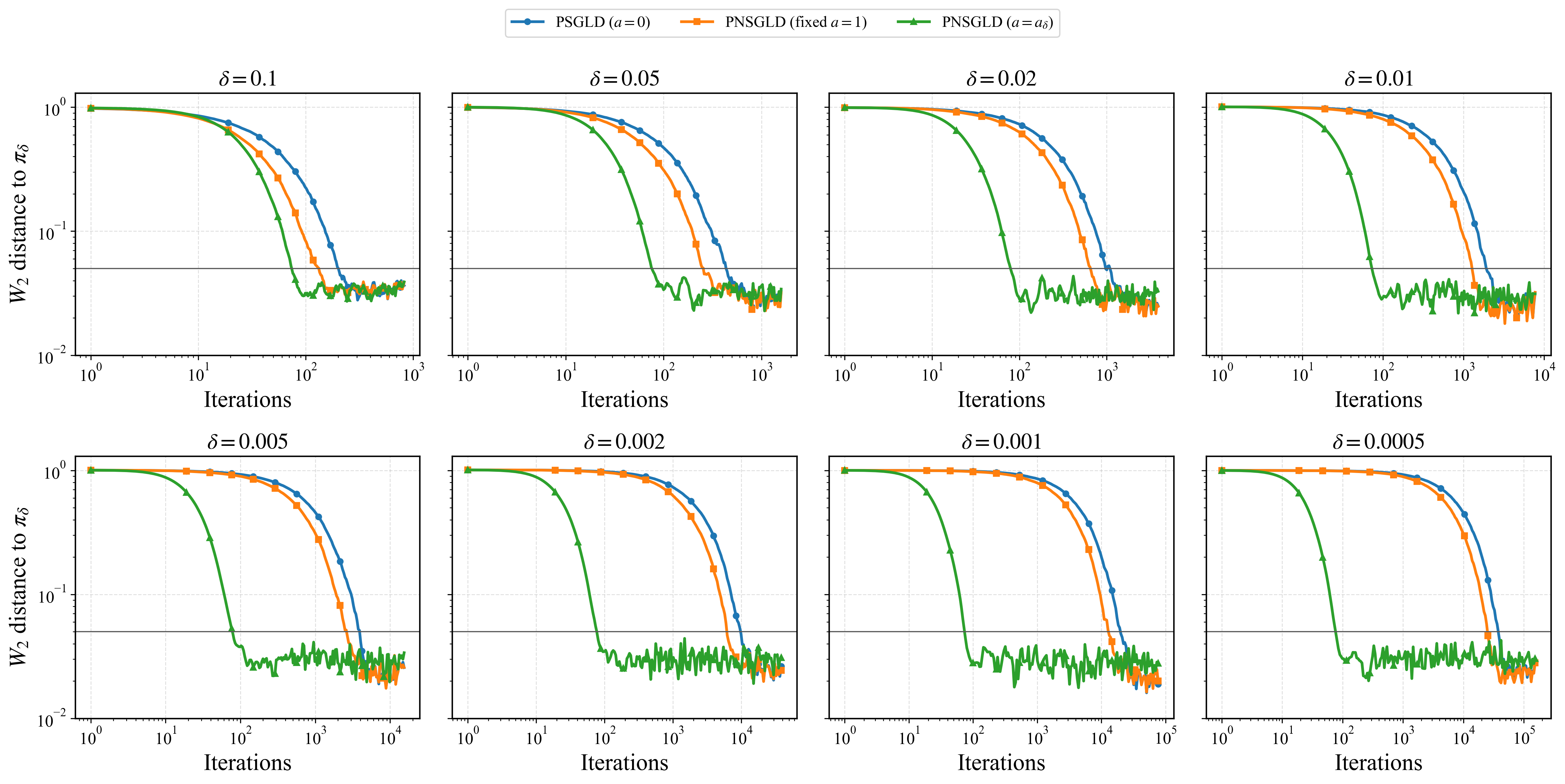}}\\[-0.1em]
{\small \textbf{Figure~\thefigure.} Gaussian $W_2$ estimates at eight
penalty values, averaged over 12 repetitions of 2000 particles.
The horizontal line marks the target accuracy $0.05$.}
\end{figure}

\begin{table}[!htbp]
\centering
\caption{Stepsize fractions for the eight penalty values. The spectral
radius and stationary $W_2$ error are reported as maxima over these
values. The stationary error includes discretization and stochastic
gradient noise.}
\label{tab:small-delta-config}
\resizebox{\textwidth}{!}{%
\begin{tabular}{lcccc}
\hline
Method & $a$ & $q$ & Max. spectral radius & Max. stationary $W_2$ bias \\
\hline
PSGLD & $0$ & $0.1564$ & $0.9999$ & $0.0199$ \\
PNSGLD (fixed) & $1$ & $0.1171$ & $0.9999$ & $0.0199$ \\
PNSGLD ($a=a_\delta$) & $a_\delta$ & $0.0313$ & $0.9374$ & $0.0198$ \\
\hline
\end{tabular}}
\end{table}

\begin{table}[!htbp]
\centering
\caption{First iteration
$K_{0.05}=\min\{k\geq0:W_2(\nu_k,\pi_\delta)\leq0.05\}$,
computed from the exact Gaussian mean and covariance recursion.}
\label{tab:small-delta-hitting}
\begin{tabular}{ccccc}
\hline
$\delta$ & $a_\delta$ & PSGLD & PNSGLD (fixed) & PNSGLD ($a=a_\delta$) \\
\hline
$0.1$ & $2.182$ & $206$ & $125$ & $75$ \\
$0.05$ & $3.123$ & $397$ & $253$ & $75$ \\
$0.02$ & $4.975$ & $971$ & $636$ & $75$ \\
$0.01$ & $7.053$ & $1929$ & $1275$ & $76$ \\
$0.005$ & $9.988$ & $3844$ & $2554$ & $76$ \\
$0.002$ & $15.803$ & $9589$ & $6390$ & $76$ \\
$0.001$ & $22.355$ & $19164$ & $12783$ & $76$ \\
$0.0005$ & $31.619$ & $38315$ & $25570$ & $76$ \\
\hline
\end{tabular}
\end{table}

Figure~\ref{fig:small-delta-mechanism} illustrates the spectral
prediction of Proposition~\ref{prop:PJ:penalty:hessian:block}.
For the fixed coefficient $a=1$, the slow rate approaches two as
$\delta$ decreases. The penalty-dependent choice instead makes both
rates equal to $(1+\Lambda_\delta)/2$, eliminating the increasing
separation between the slow and fast spectral rates.

The iteration counts show the computational benefit of this choice.
For PSGLD and fixed-coefficient PNSGLD, the counts grow approximately
linearly with $\Lambda_\delta$. The fixed coefficient reduces the
required count by a factor approaching $1.50$; this agrees with the
spectral gain after accounting for the selected stepsizes, since
$2q_{\rm fixed}/q_{\rm rev}\approx1.50$.
In contrast, PNSGLD with $a=a_\delta$ reaches the prescribed accuracy
in only 75--76 iterations across all eight penalty values.
At $\delta=5\times10^{-4}$, it requires 76 iterations compared with
38315 for PSGLD, a reduction by a factor of approximately 500
(Table~\ref{tab:small-delta-hitting}).

All methods satisfy the same stationary error tolerance
(Table~\ref{tab:small-delta-config}). The improvement therefore
results from a more effective choice of the nonreversible perturbation,
without relaxing the accuracy requirement.
These results support the favorable penalty dependence established in
Proposition~\ref{prop:PJ:penalty:Euler:acceleration} and demonstrate
that the acceleration remains effective in this quadratic experiment
with stochastic gradients. Together with the preceding experiments,
they show the value of choosing the skew perturbation to suit the
sampling problem, particularly by adapting it to the curvature
introduced by penalization.

\section{Conclusion}

We have developed penalized nonreversible Langevin algorithms for
constrained sampling and established nonasymptotic error bounds for
their full gradient and stochastic gradient discretizations.
The analysis combines convergence to the penalized distribution with
approximation of the constrained target. For quadratic potentials, adapting the skew perturbation to the penalty induced curvature
improves the sufficient iteration bound.

The numerical experiments illustrate faster initial convergence in
constrained linear regression, improved sampling performance under
strong penalization, and predictive benefits from suitable constant
and state dependent perturbations in statistical learning.
In the quadratic example, the penalty dependent choice substantially
reduces the iterations needed to reach a common accuracy, including
in the presence of stochastic gradient noise. These findings emphasize
the importance of designing the nonreversible perturbation together
with the penalty parameter.






\bibliographystyle{alpha} \bibliography{Langevin}


\clearpage

\appendix

\section{Technical Details}

We use the following lemmas in the main results.


\begin{lemma}
\label{lemma:uniform:bound}
Let $V_\delta=f+S/\delta$ have an $L_\delta$ Lipschitz gradient and be $(m_\delta,b_\delta)$ dissipative with $m_\delta>0$. Assume $J$ is Borel measurable, $J(x)^\top=-J(x)$ and $\|J(x)\|_{\mathrm{op}}\leq M_J$. If $X\sim\pi_\delta$, then
\begin{align}
\label{eq:moment:target:bounds}
\mathbb E\|X\|^2&\leq M_{2,\delta}, & \mathbb E\|X\|^4&\leq M_{4,\delta},
\\
\label{eq:moment:target:gradient:bounds}
\mathbb E\|\nabla V_\delta(X)\|^2 &\leq 2L_\delta^2M_{2,\delta}+2g_\delta^2, & \mathbb E\|\nabla V_\delta(X)\|^4 &\leq 8L_\delta^4M_{4,\delta}+8g_\delta^4.
\end{align}
Here $g_\delta$ and $(M_{2,\delta},M_{4,\delta})$ are defined in \eqref{eq:moment:quad:constants} and \eqref{eq:moment:target:constants}, respectively. For Euler iterates with $X_0\sim\nu_0$, if $\mathbb E_{\nu_0}U_\delta(X_0)<\infty$ and $0<\eta\leq1\wedge[L_\delta(1+M_J)^2]^{-1}$,
\begin{equation}
\label{eq:2nd:uniform:bound:general}
\sup_{k\geq0}\mathbb E\|X_{k\eta}\|^2\leq C_x, \qquad \sup_{k\geq0}\mathbb E\|\nabla V_\delta(X_{k\eta})\|^2\leq C_{\nabla,2}.
\end{equation}
$C_x$ and $C_{\nabla,2}$ are defined in \eqref{eq:moment:euler:second:constants}. If, in addition, $\mathbb E_{\nu_0}\|X_0\|^4<\infty$ and
\begin{equation}
\label{eq:moment:eta4:explicit}
0<\eta\leq 1\wedge\frac{1}{L_\delta(1+M_J)^2} \wedge\frac{q_{U,\delta}}{64L_\delta^2} \wedge\left( \frac{q_{U,\delta}} {128L_\delta^4(1+M_J)^4} \right)^{1/3},
\end{equation}
then
\begin{equation}
\label{eq:euler:fourth:moment:bound}
\sup_{k\geq0}\mathbb{E}\|X_{k\eta}\|^4\leq C_{4,\delta}^{\mathrm{E}}, \qquad \sup_{k\geq0}\mathbb{E}\|\nabla V_\delta(X_{k\eta})\|^4\leq C_{\nabla,4}^{\mathrm{E}}.
\end{equation}
Here $q_{U,\delta}$, $C_{4,\delta}^{\mathrm{E}}$, and $C_{\nabla,4}^{\mathrm{E}}$ are defined in \eqref{eq:moment:euler:qr:definition}, \eqref{eq:moment:euler:fourth:state}, and \eqref{eq:moment:euler:fourth:constants}, respectively.
\end{lemma}

\begin{proof}
See Appendix~\ref{proof:fourth:moment}.
\end{proof}

\begin{lemma}
\label{lemma:PJ:sg:second:moment}
For the PNSGLD iterates in~\eqref{eq:pnsgld}, assume the smoothness, dissipativity, and skew matrix bounds of Lemma~\ref{lemma:uniform:bound}, and assume $J$ is Borel measurable. Suppose Assumption~\ref{assumption:sg} holds, $\mathbb E_{\nu_0}U_\delta(x_0)<\infty$, and
\begin{equation}
0<\eta\leq1 \wedge\frac1{L_\delta(1+M_J^2)} \wedge \frac{q_{U,\delta}c_{U,1}} {4L_\delta(1+M_J^2)\sigma^2L^2}.
\label{eq:PJ:sg:moment:stepsize}
\end{equation}
The last upper bound is interpreted as $+\infty$ if $\sigma L=0$. Then
\begin{align}
\sup_{k\geq0}\mathbb E\|x_k\|^2 &\leq C_x^{\mathrm{sg}}<\infty.
\label{eq:PJ:sg:second:moment}
\end{align}
Here $c_{U,1}$ and $(q_{U,\delta},r_{U,\delta})$ are defined in \eqref{eq:cU1:explicit} and~\eqref{eq:moment:euler:qr:definition}, $C_x^{\mathrm{sg}}$ is defined in~\eqref{eq:Cx:sg:constant}.
\end{lemma}

\begin{proof}
See Appendix~\ref{proof:lemma:PJ:sg:second:moment}.
\end{proof}

\begin{lemma}
\label{lemma:dissipative:S}
Let $\mathcal K\subseteq B(0,R)$ be a nonempty closed convex set and $S_{\mathcal K}(x):=\mathrm{dist}(x,\mathcal K)^2$. Then $S_{\mathcal K}$ is continuously differentiable, has a $4$ Lipschitz gradient, and is $(1,R^2)$ dissipative:
\begin{equation}
\langle x,\nabla S_{\mathcal K}(x)\rangle \geq\|x\|^2-R^2, \qquad x\in\mathbb R^d.
\label{eq:penalty:dissipative:corrected}
\end{equation}
In particular, under Assumption~\ref{assumption:C}, one may take $\ell=4$, $m_S=1$, and $b_S=R^2$ for $S(x)=\delta_{\mathcal C}(x)^2$.
\end{lemma}

\begin{proof}
See Appendix~\ref{proof:lemma:dissipative:S}.
\end{proof}

The projection identity gives the pair $(m_S,b_S)=(1,R^2)$ used throughout this paper. In particular, we do not use the value $b_S=R^2/4$ in \cite[Lemma~C.1]{gurbuzbalaban2024penalized}.
\begin{lemma}
\label{lemma:dissipative}
If Assumptions~\ref{assumption:smooth} and~\ref{assumption:C} hold, then $V_\delta=f+S/\delta$ has an $L_\delta$ Lipschitz gradient and is $(m_\delta,b_\delta)$ dissipative, where
\begin{equation}
\label{eq:diss:const}
m_\delta:=-L-\frac12+\frac{m_S}{\delta}, \qquad b_\delta:=\frac12\|\nabla f(0)\|^2+\frac{b_S}{\delta}, \qquad L_\delta:=L+\frac{\ell}{\delta}.
\end{equation}
Here $(\ell,m_S,b_S)=(4,1,R^2)$ may be taken from Lemma~\ref{lemma:dissipative:S}, and $m_\delta>0$ whenever $0<\delta<m_S/(L+1/2)$.
\end{lemma}

\begin{proof}
See Appendix~\ref{proof:lemma:dissipative}.
\end{proof}

\begin{lemma}
\label{lemma:min}
Under Assumptions~\ref{assumption:mu:L} and~\ref{assumption:C}, the minimizer $x_*$ of $V_\delta=f+S/\delta$ satisfies $\|x_*\|\leq\|\nabla f(0)\|/\mu\leq(1+c)R$ for every $\delta>0$, where $c:=\max\{0,\|\nabla f(0)\|/(\mu R)-1\}$.
\end{lemma}
\begin{proof}
Since $0\in\mathcal C$, $\nabla S(0)=0$. Strong convexity and $\nabla V_\delta(x_*)=0$ give
$$
\mu\|x_*\|^2\leq\langle x_*,\nabla V_\delta(x_*)-\nabla V_\delta(0)\rangle=-\langle x_*,\nabla f(0)\rangle\leq\|x_*\|\|\nabla f(0)\|.
$$
Division by $\mu\|x_*\|$ proves the bound when $x_*\neq0$, and the case $x_*=0$ is immediate.
\end{proof}

\begin{lemma}
\label{lemma:ghz:transfer}
Under Assumptions~\ref{assumption:S},~\ref{assumption:C}, and \ref{assumption:smooth}, assume the representation \eqref{eq:ghz:constraint:representation}, as in \cite[Proposition~2.11]{gurbuzbalaban2024penalized}, and, in the merely convex case, additionally assume \eqref{eq:ghz:strict:feasibility}. No initialization hypothesis from that proposition is imposed. Let $\delta=\ep^4$ and take $\alpha=0$ when $h$ is strongly convex and $\alpha=\ep^2$ when $h$ is merely convex. For each fixed dimension and fixed problem data, for sufficiently small $\ep>0$,
\begin{equation}
\mathrm{TV}(\pi_\delta^\alpha,\pi) =\widetilde{\mathcal O}_d(\ep^2), \qquad \rho_*^{-1}=\mathcal O(1),
\label{eq:ghz:transfer}
\end{equation}
where $\rho_*$ is a log Sobolev constant of $\pi_\delta^\alpha$ in the convention~\eqref{eq:LSI:convention}. In $\mathcal O_d(\ep^2)$, the implicit constant may depend on the dimension $d$, but is independent of $\ep$.
\end{lemma}

\begin{proof}
See Appendix~\ref{proof:lemma:ghz:transfer}.
\end{proof}

\begin{lemma}
\label{lemma:complexity:penalty:orders}
Under Assumptions~\ref{assumption:S},~\ref{assumption:C}, and \ref{assumption:smooth}, let $\mathcal C^\alpha$, $S^\alpha$, and $V_\delta^\alpha=f+S^\alpha/\delta$ be defined as in \eqref{eq:C:alpha} with $\alpha\geq0$. Use the notation defined in \eqref{eq:diss:const}, \eqref{eq:moment:quad:constants}, \eqref{eq:cU2:explicit}, and~\eqref{eq:moment:euler:qr:definition}, with $(S,V_\delta,U_\delta)$ replaced by $(S^\alpha,V_\delta^\alpha,U_\delta^\alpha)$. If $0<\delta\leq 1 \wedge \frac{1}{2(L+1/2)}$, then
\begin{align}
L_\delta&=\mathcal O(\delta^{-1}), \, m_\delta=\Theta(\delta^{-1}), \, b_\delta=\mathcal O(\delta^{-1}), \, c_\delta=\mathcal O(\delta^{-1}), \, c_{U,2}=\mathcal O(\delta^{-1}),
\nonumber\\
q_{U,\delta}&=\Omega(\delta^{-1}), \, r_{U,\delta}=\mathcal O(\delta^{-2}).
\label{eq:complexity:penalty:orders:statement}
\end{align}
The constants depend only on $R,L,|f(0)|$ and $\|\nabla f(0)\|$.
\end{lemma}
\begin{proof}
The proof is provided in Appendix~\ref{proof:complexity:penalty:orders}.
\end{proof}


\subsection{Proof of Lemma~\ref{lemma:pi}}
\label{proof:lemma:pi}
We denote $\rho_\delta=Z_\delta^{-1}e^{-V_\delta}$. Since $\nabla V_\delta$ is globally Lipschitz, $\nabla\rho_\delta=-\rho_\delta\nabla V_\delta$ almost everywhere. We first verify the asserted regularity of the state dependent drift rather than assuming it. Let
\begin{equation}
M_J:=\sup_{x\in\mathbb R^d}\|J(x)\|_{\mathrm{op}}<\infty
\end{equation}
and let $L_J$ be a Lipschitz constant of $J$ with respect to the operator norm. For $x,y\in B(0,R)$,
\begin{align}
\|B_J(x)-B_J(y)\| & = \|(I+J(x))\bigl(\nabla V_\delta(x)-\nabla V_\delta(y)\bigr) +\bigl(J(x)-J(y)\bigr)\nabla V_\delta(y)\|
\nonumber\\
&\leq(1+M_J)L_\delta\|x-y\| +L_J\|x-y\|\|\nabla V_\delta(y)\|
\nonumber\\
&\leq \left[(1+M_J)L_\delta +L_J\bigl(L_\delta R+\|\nabla V_\delta(0)\|\bigr)\right]\|x-y\|.
\end{align}
Thus $B_J$ is Lipschitz on every ball and hence locally Lipschitz. Moreover,
\begin{equation}
\|B_J(x)\| \leq(1+M_J) \bigl(L_\delta\|x\|+\|\nabla V_\delta(0)\|\bigr),
\label{eq:pi:drift:linear:growth}
\end{equation}
Hence the SDE has a unique non explosive strong solution, and its martingale problem is well posed. On $C_c^\infty(\mathbb R^d)$, its generator and distributional forward operator are
\begin{equation}
\mathcal L\varphi=-\langle B_J,\nabla\varphi\rangle+\Delta\varphi, \qquad \mathcal L^*\rho =\nabla\cdot\big((I+J)\nabla V_\delta\rho\big)+\Delta\rho.
\end{equation}
Substituting $\rho=\rho_\delta$ and using the product rule in weak form gives
\begin{align}
\mathcal{L}^*\rho_\delta &=\nabla\cdot(\rho_\delta\nabla V_\delta)+\Delta\rho_\delta +\nabla\cdot(\rho_\delta J\nabla V_\delta)
\nonumber\\
&=(\nabla\rho_\delta)^\top J\nabla V_\delta +\rho_\delta\nabla\cdot(J\nabla V_\delta)
\nonumber\\
&=-\rho_\delta(\nabla V_\delta)^\top J\nabla V_\delta +\rho_\delta\nabla\cdot(J\nabla V_\delta)=0,
\end{align}
where we use the fact $\nabla\cdot(\rho_\delta\nabla V_\delta)+\Delta\rho_\delta=0$. The last line uses skew symmetry and~\eqref{eq:J:compatibility}. Hence, for every $\varphi\in C_c^\infty(\mathbb R^d)$,
\begin{equation}
\int_{\mathbb R^d}\mathcal L\varphi(x)\,\pi_\delta(dx) =\langle\mathcal L^*\rho_\delta,\varphi\rangle=0.
\label{eq:pi:infinitesimal:invariance}
\end{equation}
Equation~\eqref{eq:pi:infinitesimal:invariance} makes the constant curve $\mu_t=\pi_\delta$ a probability solution of $\partial_t\mu_t=\mathcal L^*\mu_t$. By~\eqref{eq:pi:drift:linear:growth}, for every $T>0$,
$$
\int_0^T\!\int_{\mathbb R^d}\frac{1+|\langle B_J(x),x\rangle|}{(1+\|x\|)^2}\,\pi_\delta(dx)\,dt\leq T\left[1+(1+M_J)\left(L_\delta+\|\nabla V_\delta(0)\|\right)\right]<\infty.
$$
The superposition principle~\cite[Theorem 1.1]{bogachev2021superposition} therefore gives a solution of the $\mathcal L$ martingale problem with marginal $\pi_\delta$ at every time. Uniqueness in law identifies it with~\eqref{eq:sde}, so $\pi_\delta P_t=\pi_\delta$ for every $t\geq0$.

Finally, if each column of $J$ has zero divergence, then, almost everywhere,
\begin{equation}
\nabla\cdot(J\nabla V_\delta) =\sum_{i,j}(\partial_iJ_{ij})\partial_jV_\delta+ \sum_{i,j}J_{ij}\partial_{ij}^2V_\delta=0.
\end{equation}
The first sum vanishes by hypothesis. The second vanishes because the almost everywhere Hessian of a $C^{1,1}$ function is symmetric and $J$ is skew symmetric. The proof is complete.

\subsection{Proof of Theorem~\ref{thm:nu:pi:delta}}
\label{proof:thm:nu:pi:delta}
\begin{proof}[Proof of Theorem~\ref{thm:nu:pi:delta}]
For $k\eta<t\leq(k+1)\eta$, write $B_\delta(y):=(I+J(y))\nabla V_\delta(y)$ and define the continuous time interpolation by
$$
dX_t=-B_\delta(X_{k\eta})\,dt+\sqrt2\,dW_t.
$$
This interpolation agrees with~\eqref{eq:alg} at the grid points. For $s:=t-k\eta\in(0,\eta]$, integration gives
\begin{equation}
X_t=X_{k\eta}-sB_\delta(X_{k\eta}) +\sqrt2(W_t-W_{k\eta}).
\end{equation}
Lemma~\ref{lemma:uniform:bound} gives $\sup_k\mathbb E\|X_{k\eta}\|^2\leq C_x$ and $\sup_k\mathbb E\|B_\delta(X_{k\eta})\|^2\leq(1+M_J)^2C_{\nabla,2}$. Hence, on each interval,
$$
\mathbb E\|X_t\|^2\leq2C_x+2\eta^2(1+M_J)^2C_{\nabla,2}+2d\eta<\infty.
$$
Starting from $F_0<\infty$, apply Lemma~\ref{lem:entropy:identity} on each grid intervals. Its endpoint bound gives $F(p_{k\eta})<\infty$ by induction, and its closed interval conclusion gives
$$
F(p_\cdot)\in AC([k\eta,(k+1)\eta]),\qquad \int_{k\eta}^{(k+1)\eta}\!\int p_t\left\|\nabla\log\frac{p_t}{p_\delta^*}\right\|^2dx\,dt<\infty.
$$
For almost every $t\in(k\eta,(k+1)\eta)$, the differential identity in Lemma~\ref{lem:entropy:identity} gives
$$
\frac d{dt}F(p_t)=-\int_{\mathbb R^d}p_t(x)\left\|\nabla\log\frac{p_t(x)}{p_\delta^*(x)}\right\|^2dx+\mathbb E\left\langle\nabla\log\frac{p_t(X_t)}{p_\delta^*(X_t)},B_\delta(X_t)-B_\delta(X_{k\eta})\right\rangle.
$$
\textbf{Constant $J$.} In this case,
$$
B_\delta(X_t)-B_\delta(X_{k\eta})=(I+J)(\nabla V_\delta(X_t)-\nabla V_\delta(X_{k\eta})).
$$
Therefore
\begin{align}
\frac{d}{dt}F(p_t) &\leq-\frac12\int_{\mathbb R^d}p_t(x)\left\|\nabla\log\frac{p_t(x)}{p_\delta^*(x)}\right\|^2dx+\frac12\mathbb E\left\|(I+J)(\nabla V_\delta(X_t)-\nabla V_\delta(X_{k\eta}))\right\|^2
\nonumber\\
&\leq-\frac12\int_{\mathbb R^d}p_t(x)\left\|\nabla\log\frac{p_t(x)}{p_\delta^*(x)}\right\|^2dx+\frac12(1+M_J)^2L_\delta^2\mathbb E\|X_t-X_{k\eta}\|^2.
\end{align}
The second inequality uses $-\|u\|^2+\langle u,v\rangle \leq-\frac12\|u\|^2+\frac12\|v\|^2$, and the last uses $\|I+J\|_{\mathrm{op}}\leq1+M_J$ and the $L_\delta$ Lipschitz continuity of $\nabla V_\delta$. For $s=t-k\eta$,
\begin{align}
\mathbb E\|X_t-X_{k\eta}\|^2 &=s^2\mathbb E\|(I+J(X_{k\eta})) \nabla V_\delta(X_{k\eta})\|^2+2ds \leq s^2(1+M_J)^2C_{\nabla,2}+2ds.
\label{eq:increment:second:compact}
\end{align}
Here the cross term is zero after conditioning on $X_{k\eta}$, since $W_t-W_{k\eta}$ is centered and independent of $X_{k\eta}$. Lemma~\ref{lemma:uniform:bound} applies provided that
\begin{equation}
0<\eta\leq1\wedge\frac1{L_\delta(1+M_J)^2}.
\label{eq:stepsize:constant:moment:condition}
\end{equation}
The log Sobolev inequality gives
\begin{equation}
\int_{\mathbb R^d} \left\|\nabla\log\frac{p_t(x)}{p_\delta^*(x)}\right\|^2p_t(x)\,dx \geq2\rho_*F(p_t).
\end{equation}
Together with $s\leq\eta$, this yields
\begin{equation}
\frac{d}{dt}F(p_t) \leq-\rho_*F(p_t) +d(1+M_J)^2L_\delta^2\eta +\frac12(1+M_J)^4L_\delta^2C_{\nabla,2}\eta^2.
\label{eq:entropy:constant:differential}
\end{equation}

\textbf{State dependent $J$.} The drift difference satisfies
\begin{align}
&B_\delta(X_t)-B_\delta(X_{k\eta})
\nonumber\\
&\qquad=(I+J(X_{k\eta}))(\nabla V_\delta(X_t)-\nabla V_\delta(X_{k\eta}))+(J(X_t)-J(X_{k\eta}))\nabla V_\delta(X_t).
\end{align}
Hence Lemma~\ref{lem:entropy:identity} gives
\begin{align}
\frac{d}{dt}F(p_t) &=-\int_{\mathbb R^d}p_t(x)\left\|\nabla\log\frac{p_t(x)}{p_\delta^*(x)}\right\|^2dx
\nonumber\\
&\qquad+\mathbb E\left\langle\nabla\log\frac{p_t(X_t)}{p_\delta^*(X_t)},(I+J(X_{k\eta}))(\nabla V_\delta(X_t)-\nabla V_\delta(X_{k\eta}))\right\rangle
\nonumber\\
&\qquad+\mathbb E\left\langle\nabla\log\frac{p_t(X_t)}{p_\delta^*(X_t)},(J(X_t)-J(X_{k\eta}))\nabla V_\delta(X_t)\right\rangle.
\end{align}
Applying $|\langle u,v\rangle|\leq\frac14\|u\|^2+\|v\|^2$ to each of the last two terms, followed by the Lipschitz bounds for $\nabla V_\delta$ and $J$, gives
\begin{align}
\frac{d}{dt}F(p_t) & \leq -\frac12\int_{\mathbb R^d} \left\|\nabla\log\frac{p_t(x)}{p_\delta^*(x)}\right\|^2p_t(x)\,dx +\mathbb E\left\|(I+J(X_{k\eta})) (\nabla V_\delta(X_t)-\nabla V_\delta(X_{k\eta}))\right\|^2
\nonumber\\
& \qquad +\mathbb E\left\|(J(X_t)-J(X_{k\eta}))\nabla V_\delta(X_t)\right\|^2
\nonumber\\
&\leq-\frac12\int_{\mathbb R^d}\left\|\nabla\log\frac{p_t(x)}{p_\delta^*(x)}\right\|^2p_t(x)\,dx +(1+M_J)^2L_\delta^2\mathbb E\|X_t-X_{k\eta}\|^2
\nonumber\\
& \qquad +L_J^2\mathbb E\left[\|X_t-X_{k\eta}\|^2\|\nabla V_\delta(X_t)\|^2\right].
\label{eq:entropy:state:prebound}
\end{align}
Moreover,
\begin{equation}
\|\nabla V_\delta(X_t)\|^2 \leq2L_\delta^2\|X_t-X_{k\eta}\|^2 +2\|\nabla V_\delta(X_{k\eta})\|^2.
\label{eq:gradient:increment:compact}
\end{equation}
Since
\begin{equation}
X_t-X_{k\eta}=-s(I+J(X_{k\eta}))\nabla V_\delta(X_{k\eta}) +\sqrt2(W_t-W_{k\eta}),
\end{equation}
the inequalities $\|a+b\|^4\leq8\|a\|^4+8\|b\|^4$ and $\mathbb E\|\sqrt2(W_t-W_{k\eta})\|^4=4d(d+2)s^2$ give
\begin{align}
\mathbb E\|X_t-X_{k\eta}\|^4 &\leq8(1+M_J)^4s^4 \mathbb E\|\nabla V_\delta(X_{k\eta})\|^4+32d(d+2)s^2
\nonumber\\
&\leq8(1+M_J)^4C_{\nabla,4}^{\mathrm E}\eta^4 +32d(d+2)\eta^2,
\end{align}
and also,
\begin{align}
\mathbb E\bigl[\|X_t-X_{k\eta}\|^2 \|\nabla V_\delta(X_{k\eta})\|^2\bigr] &=\mathbb E\left[ \mathbb E(\|X_t-X_{k\eta}\|^2| X_{k\eta}) \|\nabla V_\delta(X_{k\eta})\|^2\right]
\nonumber\\
&=\mathbb E\left[ \left(s^2\|(I+J(X_{k\eta}))\nabla V_\delta(X_{k\eta})\|^2+2ds\right) \|\nabla V_\delta(X_{k\eta})\|^2\right]
\nonumber\\
&\leq(1+M_J)^2C_{\nabla,4}^{\mathrm E}\eta^2 +2dC_{\nabla,2}\eta.
\end{align}
Therefore~\eqref{eq:gradient:increment:compact} gives
\begin{align}
\mathbb E\left[ \|X_t-X_{k\eta}\|^2\|\nabla V_\delta(X_t)\|^2\right] &\leq2L_\delta^2\mathbb E\|X_t-X_{k\eta}\|^4 +2\mathbb E\left[ \|X_t-X_{k\eta}\|^2 \|\nabla V_\delta(X_{k\eta})\|^2\right]
\nonumber\\
&\leq16L_\delta^2(1+M_J)^4 C_{\nabla,4}^{\mathrm E}\eta^4 +64L_\delta^2d(d+2)\eta^2
\nonumber\\
&\qquad +2(1+M_J)^2C_{\nabla,4}^{\mathrm E}\eta^2 +4dC_{\nabla,2}\eta.
\label{eq:increment:gradient:mixed:red}
\end{align}
The moment bounds follow from Lemma~\ref{lemma:uniform:bound} provided that $\eta$ satisfies~\eqref{eq:moment:eta4:explicit}. Substituting \eqref{eq:increment:second:compact} and \eqref{eq:increment:gradient:mixed:red} in \eqref{eq:entropy:state:prebound}, and using
\begin{equation}
\int_{\mathbb R^d} \left\|\nabla\log\frac{p_t(x)}{p_\delta^*(x)}\right\|^2p_t(x)\,dx \geq2\rho_*F(p_t),
\end{equation}
yields
\begin{align}
\frac{d}{dt}F(p_t) & \leq-\rho_*F(p_t)+\left(2d(1+M_J)^2L_\delta^2+4dL_J^2C_{\nabla,2}\right)\eta
\nonumber\\
& \qquad +\left((1+M_J)^4L_\delta^2C_{\nabla,2}+64L_J^2L_\delta^2d(d+2) +2L_J^2(1+M_J)^2C_{\nabla,4}^{\mathrm E}\right)\eta^2
\nonumber\\
& \qquad +16L_J^2L_\delta^2(1+M_J)^4 C_{\nabla,4}^{\mathrm E}\eta^4.
\label{eq:entropy:state:differential}
\end{align}

Let $R_\eta$ be the sum of the positive terms on the right hand side of \eqref{eq:entropy:constant:differential} or \eqref{eq:entropy:state:differential}. For almost every $t\in(k\eta,(k+1)\eta)$,
\begin{equation}
\frac{d}{dt}\left(e^{\rho_*(t-k\eta)}F(p_t)\right) \leq e^{\rho_*(t-k\eta)}R_\eta.
\end{equation}
Integrating from $k\eta$ to $(k+1)\eta$ and then iterating the resulting recursion give
\begin{align}
F(p_{K\eta}) &\leq e^{-\rho_*K\eta}F_0 +\frac{1-e^{-\rho_*K\eta}}{\rho_*}R_\eta \leq e^{-\rho_*K\eta}F_0+\frac{R_\eta}{\rho_*}.
\label{eq:entropy:iteration:compact}
\end{align}

If $J$ is constant, the two terms in \eqref{eq:entropy:constant:differential} satisfy
\begin{align}
\frac{d(1+M_J)^2L_\delta^2\eta}{\rho_*} &\leq\frac{\ep^2}{4} \quad\text{provided}\quad 0<\eta\leq \frac{\rho_*\ep^2}{4d(1+M_J)^2L_\delta^2},
\\
\frac{(1+M_J)^4L_\delta^2C_{\nabla,2}\eta^2}{2\rho_*} &\leq\frac{\ep^2}{4}\quad\text{provided}\quad0<\eta\leq\left(\frac{\rho_*\ep^2}{2(1+M_J)^4L_\delta^2C_{\nabla,2}} \right)^{1/2}.
\label{eq:stepsize:constant:quadratic}
\end{align}
Combining~\eqref{eq:stepsize:constant:moment:condition}-\eqref{eq:stepsize:constant:quadratic}, define
\begin{equation}
\label{eq:stepsize:const:explicit}
\eta_{\mathrm{const}}(\ep):= 1\wedge\frac{1}{L_\delta(1+M_J)^2} \wedge\frac{\rho_*\ep^2}{4d(1+M_J)^2L_\delta^2} \wedge \left( \frac{\rho_*\ep^2} {2(1+M_J)^4L_\delta^2C_{\nabla,2}} \right)^{1/2}.
\end{equation}
Therefore, if $0<\eta\leq\eta_{\mathrm{const}}(\ep)$,
\begin{equation}
\frac{R_\eta}{\rho_*} \leq\frac{\ep^2}{4}+\frac{\ep^2}{4} =\frac{\ep^2}{2}.
\end{equation}

If $J$ is state dependent, the three positive terms in \eqref{eq:entropy:state:differential}, after division by $\rho_*$, are bounded by $\ep^2/6$ when
\begin{align}
0<\eta &\leq \frac{\rho_*\ep^2}{12d(1+M_J)^2L_\delta^2+24dL_J^2C_{\nabla,2}},
\label{eq:stepsize:state:linear}\\
0<\eta &\leq \left[\frac{\rho_*\ep^2}{6(1+M_J)^4L_\delta^2C_{\nabla,2}+384L_J^2L_\delta^2d(d+2)+12L_J^2(1+M_J)^2C_{\nabla,4}^{\mathrm E}}\right]^{1/2},
\\
0<\eta &\leq \left[\frac{\rho_*\ep^2}{96L_J^2L_\delta^2(1+M_J)^4C_{\nabla,4}^{\mathrm E}}\right]^{1/4}.
\label{eq:stepsize:state:quartic}
\end{align}
The last condition is unnecessary when $L_J=0$. Combining \eqref{eq:moment:eta4:explicit} and \eqref{eq:stepsize:state:linear}-\eqref{eq:stepsize:state:quartic}, define
\begin{align}
\label{eq:stepsize:sd:explicit}
\eta_{\mathrm{sd}}(\ep)&:= 1\wedge\frac{1}{L_\delta(1+M_J)^2} \wedge\frac{q_{U,\delta}}{64L_\delta^2} \wedge \left(\frac{q_{U,\delta}}{128L_\delta^4(1+M_J)^4}\right)^{1/3}\wedge \frac{\rho_*\ep^2}{12d(1+M_J)^2L_\delta^2+24dL_J^2C_{\nabla,2}}
\nonumber\\
&\qquad \wedge \left(\frac{\rho_*\ep^2}{6}\right)^{1/2} \left( (1+M_J)^4L_\delta^2C_{\nabla,2} +64L_J^2L_\delta^2d(d+2)+2L_J^2(1+M_J)^2C_{\nabla,4}^{\mathrm E} \right)^{-1/2}
\nonumber\\
&\qquad \wedge \left[\frac{\rho_*\ep^2}{96L_J^2L_\delta^2(1+M_J)^4C_{\nabla,4}^{\mathrm E}}\right]^{1/4},
\end{align}
where the last term is omitted if $L_J=0$. Thus, if $0<\eta\leq\eta_{\mathrm{sd}}(\ep)$,
\begin{equation}
\frac{R_\eta}{\rho_*} \leq\frac{\ep^2}{6}+\frac{\ep^2}{6}+\frac{\ep^2}{6} =\frac{\ep^2}{2}.
\end{equation}
Here $C_{\nabla,2}$, $q_{U,\delta}$, and $C_{\nabla,4}^{\mathrm E}$ are defined in \eqref{eq:moment:euler:second:constants}, \eqref{eq:moment:euler:qr:definition}, and \eqref{eq:moment:euler:fourth:constants}, respectively. For every $K$ allowed in Theorem~\ref{thm:nu:pi:delta},
\begin{equation}
e^{-\rho_*K\eta}F_0\leq\frac{\ep^2}{2}.
\end{equation}
Indeed, this follows from $2F_0\leq\ep^2$ for arbitrary $K\in\mathbb N_0$, and from~\eqref{eq:K:explicit:theorem} otherwise. Hence~\eqref{eq:entropy:iteration:compact} gives $F(p_{K\eta})\leq\ep^2$. Since $p_{K\eta}$ and $p_\delta^*$ are the densities of $\nu_K$ and $\pi_\delta$, Pinsker's inequality gives
\begin{equation}
\mathrm{TV}(\nu_K,\pi_\delta) \leq\sqrt{2F(p_{K\eta})} \leq\sqrt2\ep.
\end{equation}
The proof is complete.
\end{proof}

\subsection{Proof of Proposition~\ref{main:full:grad}}
\label{proof:main:full:grad}
Put $\delta=\ep^4$. For sufficiently small $\ep>0$, Lemma~\ref{lemma:complexity:penalty:orders} applies uniformly for $\alpha\in\{0,\ep^2\}$. Applying Theorem~\ref{thm:nu:pi:delta} with $(V_\delta,\pi_\delta)$ replaced by $(V_\delta^\alpha,\pi_\delta^\alpha)$ gives
\begin{align}
\mathrm{TV}(\nu_K,\pi) &\leq \mathrm{TV}(\nu_K,\pi_\delta^\alpha) +\mathrm{TV}(\pi_\delta^\alpha,\pi)\leq\sqrt2\ep+\mathrm{TV}(\pi_\delta^\alpha,\pi) =\widetilde{\mathcal O}(\ep),
\end{align}
where the last equality follows from~\eqref{eq:external:penalty:TV}.

\paragraph{Constant $J$.} Equations~\eqref{eq:moment:euler:first:uniform:constant}, \eqref{eq:moment:euler:second:constants}, and \eqref{eq:complexity:penalty:orders:statement} give
\begin{equation}
C_{\nabla,2} =\mathcal O\left( \delta^{-1}(\mathbb E_{\nu_0}U_\delta^\alpha(X_0)+\delta^{-1}+d)\right).
\end{equation}
Hence the entries in~\eqref{eq:stepsize:const:explicit} have the following lower bounds:
\begin{align}
1\wedge\frac1{L_\delta(1+M_J)^2} &=\Omega\left(\frac{\delta}{(1+M_J)^2}\right),
\nonumber\\
\frac{\rho_*\ep^2} {4d(1+M_J)^2(L_\delta)^2} &=\Omega\left(\frac{\delta^{5/2}} {d(1+M_J)^2}\right),
\nonumber\\
\left( \frac{\rho_*\ep^2} {2(1+M_J)^4(L_\delta)^2C_{\nabla,2}} \right)^{1/2} &=\Omega\left(\frac{\delta^{7/4}} {(1+M_J)^2 \sqrt{\mathbb E_{\nu_0}U_\delta^\alpha(X_0)+\delta^{-1}+d}}\right).
\end{align}
Since $\delta^{5/2}/d\leq\delta$, the minimum in \eqref{eq:stepsize:const:explicit} satisfies
\begin{equation}
\eta_{\mathrm{const}}(\ep)^{-1} =\mathcal O\left((1+M_J)^2 \max\left\{ \frac d{\ep^{10}}, \frac{\sqrt{\mathbb E_{\nu_0}U_\delta^\alpha(X_0)+\ep^{-4}+d}} {\ep^7} \right\}\right).
\label{eq:complexity:const:general:eta}
\end{equation}

\paragraph{State dependent $J$.} Suppose now that $J$ is state dependent and $\nu_0$ has finite fourth moment. From \eqref{eq:moment:euler:first:uniform:constant}, \eqref{eq:moment:euler:second:constants}, and \eqref{eq:complexity:penalty:orders:statement},
\begin{equation}
C_{\nabla,2} =\mathcal O\left( \delta^{-1}(\mathbb E_{\nu_0}U_\delta^\alpha(X_0)+\delta^{-1}+d)\right).
\end{equation}
The definition of $R_{U,\delta}$ in \eqref{eq:moment:euler:lyapunov:drift} gives
\begin{align}
\frac{4R_{U,\delta}}{q_{U,\delta}} &=\frac{2(r_{U,\delta}+2L_\delta d+8L_\delta)^2} {q_{U,\delta}^2} +\frac{64L_\delta^2d(d+2)}{q_{U,\delta}} =\mathcal O\left((\delta^{-1}+d)^2 +\delta^{-1}d(d+2)\right)
\nonumber\\
&=\mathcal O\left(\delta^{-2}+\delta^{-1}d^2\right),
\nonumber\\
C_{\nabla,4}^{\rm E} &=\mathcal O\left( \delta^{-2}\mathbb E_{\nu_0}U_\delta^\alpha(X_0)^2 +\delta^{-4}+\delta^{-3}d^2\right).
\end{align}
Therefore, the last three denominators in restrictions of \eqref{eq:stepsize:sd:explicit} satisfy the following bounds. Here we use $d\leq d^2$, $0<\delta\leq1$, and
\begin{align}
(1+M_J)^4 &\leq\left((1+M_J)^2\vee L_J^2\right)^2,\quad L_J^2(1+M_J)^2 \leq\left((1+M_J)^2\vee L_J^2\right)^2,
\nonumber\\
L_J^2(1+M_J)^4 &\leq\left((1+M_J)^2\vee L_J^2\right)^3.
\end{align}
Consequently,
\begin{align}
&2d(1+M_J)^2L_\delta^2+4dL_J^2C_{\nabla,2} =\mathcal O\left( \left((1+M_J)^2\vee L_J^2\right) d\delta^{-1}(\mathbb E_{\nu_0}U_\delta^\alpha(X_0)+\delta^{-1}+d)\right),
\label{eq:complexity:sd:general:B0}\\
&(1+M_J)^4L_\delta^2C_{\nabla,2} +64L_J^2L_\delta^2d(d+2) +2L_J^2(1+M_J)^2C_{\nabla,4}^{\rm E}
\nonumber\\
&\qquad =\mathcal O\left( \left((1+M_J)^2\vee L_J^2\right)^2 \delta^{-3}\left( \mathbb E_{\nu_0}U_\delta^\alpha(X_0)+\delta^{-1}+d^2 +\delta\mathbb E_{\nu_0}U_\delta^\alpha(X_0)^2 \right)\right),
\label{eq:complexity:sd:general:B1}\\
&96L_J^2L_\delta^2(1+M_J)^4C_{\nabla,4}^{\rm E} =\mathcal O\left( \left((1+M_J)^2\vee L_J^2\right)^3 \delta^{-4}\left( \mathbb E_{\nu_0}U_\delta^\alpha(X_0)^2 +\delta^{-2}+\delta^{-1}d^2 \right)\right).
\label{eq:complexity:sd:general:B3}
\end{align}
We next simplify the components of~\eqref{eq:stepsize:sd:explicit}. For the first four components in~\eqref{eq:stepsize:sd:explicit}, by \eqref{eq:complexity:penalty:orders:statement} in Lemma~\ref{lemma:complexity:penalty:orders}, $0<\delta\leq1$, and $1+M_J\geq1$, we obtain
\begin{align}
\label{eq:stepsize:sd:4}
1\wedge\frac1{L_\delta(1+M_J)^2} & =\Omega\left(\frac{\delta}{(1+M_J)^2}\right), \quad \frac{q_{U,\delta}}{64L_\delta^2}=\Omega\left(\frac{\delta^{-1}}{\delta^{-2}}\right) =\Omega(\delta),
\nonumber\\
\left( \frac{q_{U,\delta}} {128L_\delta^4(1+M_J)^4} \right)^{1/3} &=\Omega\left( \left(\frac{\delta^{-1}}{\delta^{-4}(1+M_J)^4}\right)^{1/3} \right) =\Omega\left(\frac{\delta}{(1+M_J)^{4/3}}\right).
\end{align}
Thus each of the first four entries is $\Omega(\delta/(1+M_J)^2)$, hence $\Omega(\ep^4/(1+M_J)^2)$ after letting $\delta=\ep^4$. For the fifth component in~\eqref{eq:stepsize:sd:explicit}, $\rho_*^{-1}=\mathcal O(1)$ and \eqref{eq:complexity:sd:general:B0} give
\begin{align}
&\frac{\rho_*\ep^2} {12d(1+M_J)^2L_\delta^2+24dL_J^2C_{\nabla,2}} =\Omega\left( \frac{\ep^2\delta} {d\left((1+M_J)^2\vee L_J^2\right)} \left( \mathbb E_{\nu_0}U_\delta^\alpha(X_0)+\delta^{-1}+d \right)^{-1} \right)
\nonumber\\
&\quad =\Omega\left( \frac{\ep^6} {d\left((1+M_J)^2\vee L_J^2\right)} \left( \mathbb E_{\nu_0}U_\delta^\alpha(X_0)+\ep^{-4}+d \right)^{-1} \right).
\label{eq:complexity:sd:general:threshold:linear}
\end{align}
For the sixth component in~\eqref{eq:stepsize:sd:explicit}, \eqref{eq:complexity:sd:general:B1} gives
\begin{align}
&\left(\frac{\rho_*\ep^2}{6}\right)^{1/2} \left( (1+M_J)^4L_\delta^2C_{\nabla,2} +64L_J^2L_\delta^2d(d+2) +2L_J^2(1+M_J)^2C_{\nabla,4}^{\rm E} \right)^{-1/2}
\nonumber\\
&\qquad=\Omega\left( \frac{\ep\delta^{3/2}} {\left((1+M_J)^2\vee L_J^2\right)\left[ \mathbb E_{\nu_0}U_\delta^\alpha(X_0)+\delta^{-1}+d^2 +\delta\mathbb E_{\nu_0}U_\delta^\alpha(X_0)^2 \right]^{1/2}} \right)
\nonumber\\
&\qquad=\Omega\left( \frac{\ep^7} {\left((1+M_J)^2\vee L_J^2\right)\left[ \mathbb E_{\nu_0}U_\delta^\alpha(X_0)+\ep^{-4}+d^2 +\ep^4\mathbb E_{\nu_0}U_\delta^\alpha(X_0)^2 \right]^{1/2}} \right).
\end{align}
If $L_J>0$, \eqref{eq:complexity:sd:general:B3} gives for the seventh component in~\eqref{eq:stepsize:sd:explicit}
\begin{align}
\left[ \frac{\rho_*\ep^2} {96L_J^2L_\delta^2(1+M_J)^4C_{\nabla,4}^{\rm E}} \right]^{1/4} & =\Omega\left( \frac{\ep^{1/2}\delta} {\left((1+M_J)^2\vee L_J^2\right)^{3/4}\left[ \mathbb E_{\nu_0}U_\delta^\alpha(X_0)^2 +\delta^{-2}+\delta^{-1}d^2 \right]^{1/4}} \right)
\nonumber\\
&=\Omega\left( \frac{\ep^{9/2}} {\left((1+M_J)^2\vee L_J^2\right)^{3/4}\left[ \mathbb E_{\nu_0}U_\delta^\alpha(X_0)^2 +\ep^{-8}+\ep^{-4}d^2 \right]^{1/4}} \right).
\label{eq:complexity:sd:general:threshold:quartic}
\end{align}
Since $\eta_{\mathrm{sd}}(\ep)$ is the minimum in \eqref{eq:stepsize:sd:explicit}, the preceding four moment bounds in~\eqref{eq:stepsize:sd:4} and \eqref{eq:complexity:sd:general:threshold:linear}- \eqref{eq:complexity:sd:general:threshold:quartic}, together with $(a\wedge b)^{-1}=a^{-1}\vee b^{-1}$, show that
\begin{align}
\eta_{\mathrm{sd}}(\ep)^{-1} &=\mathcal O\Bigg( (1+M_J)^2\ep^{-4} \vee \frac{d\left((1+M_J)^2\vee L_J^2\right)}{\ep^6}\left( \mathbb E_{\nu_0}U_\delta^\alpha(X_0)+\ep^{-4}+d \right)
\nonumber\\
&\qquad \vee \frac{(1+M_J)^2\vee L_J^2}{\ep^7}\left( \mathbb E_{\nu_0}U_\delta^\alpha(X_0)+\ep^{-4}+d^2 +\ep^4\mathbb E_{\nu_0}U_\delta^\alpha(X_0)^2 \right)^{1/2}
\nonumber\\
&\qquad \vee \frac{\left((1+M_J)^2\vee L_J^2\right)^{3/4}}{\ep^{9/2}}\left( \mathbb E_{\nu_0}U_\delta^\alpha(X_0)^2 +\ep^{-8}+\ep^{-4}d^2 \right)^{1/4} \Bigg).
\label{eq:complexity:sd:inverse:max}
\end{align}
Cauchy-Schwarz gives
\begin{align}
\mathbb E_{\nu_0}U_\delta^\alpha(X_0) &\leq\left( \mathbb E_{\nu_0}U_\delta^\alpha(X_0)^2 \right)^{1/2},
\nonumber\\
\frac d{\ep^6}\left( \mathbb E_{\nu_0}U_\delta^\alpha(X_0)+\ep^{-4}+d \right) &\leq\frac d{\ep^6}\left(\left(\mathbb E_{\nu_0}U_\delta^\alpha(X_0)^2\right)^{1/2} +\ep^{-4}+d\right).
\label{eq:complexity:sd:inverse:linear}
\end{align}
The square root inequality, Cauchy-Schwarz, and Young's inequality give
\begin{align}
&\frac1{\ep^7}\left[ \mathbb E_{\nu_0}U_\delta^\alpha(X_0)+\ep^{-4}+d^2 +\ep^4\mathbb E_{\nu_0}U_\delta^\alpha(X_0)^2 \right]^{1/2}
\nonumber\\
&\qquad\leq \frac1{\ep^7} \left(\mathbb E_{\nu_0}U_\delta^\alpha(X_0)\right)^{1/2} +\ep^{-9}+d\ep^{-7} +\ep^{-5}\left( \mathbb E_{\nu_0}U_\delta^\alpha(X_0)^2 \right)^{1/2}
\nonumber\\
&\qquad\leq \frac1{\ep^7} \left(\mathbb E_{\nu_0}U_\delta^\alpha(X_0)^2\right)^{1/4} +\ep^{-9}+d\ep^{-7} +\ep^{-5}\left( \mathbb E_{\nu_0}U_\delta^\alpha(X_0)^2 \right)^{1/2}
\nonumber\\
&\qquad\leq \frac d{2\ep^6} \left( \mathbb E_{\nu_0}U_\delta^\alpha(X_0)^2 \right)^{1/2} +\frac1{2\ep^8d}+\ep^{-9}+d\ep^{-7} +\ep^{-5}\left( \mathbb E_{\nu_0}U_\delta^\alpha(X_0)^2 \right)^{1/2}
\nonumber\\
&\qquad=\mathcal O\left( \frac d{\ep^6}\left( \left( \mathbb E_{\nu_0}U_\delta^\alpha(X_0)^2 \right)^{1/2} +\ep^{-4}+d \right)\right),
\label{eq:complexity:sd:inverse:quadratic}
\end{align}
where we used, for $0<\ep\leq1$ and $d\geq1$,
\begin{equation}
\frac1{\ep^8d}+\ep^{-9}+d\ep^{-7} \leq3d\ep^{-10}, \qquad \ep^{-5}\left( \mathbb E_{\nu_0}U_\delta^\alpha(X_0)^2 \right)^{1/2} \leq\frac d{\ep^6}\left( \mathbb E_{\nu_0}U_\delta^\alpha(X_0)^2 \right)^{1/2}.
\end{equation}
Moreover, since $U_\delta^\alpha\geq1$, we obtain
\begin{align}
\frac1{\ep^{9/2}}\left( \mathbb E_{\nu_0}U_\delta^\alpha(X_0)^2 +\ep^{-8}+\ep^{-4}d^2 \right)^{1/4} &\leq \ep^{-9/2}\left( \mathbb E_{\nu_0}U_\delta^\alpha(X_0)^2 \right)^{1/2} +\ep^{-13/2}+\ep^{-11/2}d^{1/2}
\nonumber\\
&=\mathcal O\left( \frac d{\ep^6}\left( \left(\mathbb E_{\nu_0}U_\delta^\alpha(X_0)^2\right)^{1/2} +\ep^{-4}+d\right)\right).
\label{eq:complexity:sd:inverse:quartic}
\end{align}
Here we used
\begin{equation}
\ep^{-9/2}\leq d\ep^{-6}, \qquad \ep^{-13/2}\leq d\ep^{-10}, \qquad \ep^{-11/2}d^{1/2}\leq d^2\ep^{-6}.
\end{equation}
Moreover,
\begin{equation}
\ep^{-4} \leq\frac d{\ep^6}\left( \left( \mathbb E_{\nu_0}U_\delta^\alpha(X_0)^2 \right)^{1/2} +\ep^{-4}+d \right).
\end{equation}
Combining \eqref{eq:complexity:sd:inverse:max}, \eqref{eq:complexity:sd:inverse:linear}, \eqref{eq:complexity:sd:inverse:quadratic}, and \eqref{eq:complexity:sd:inverse:quartic}, and using $(1+M_J)^2\leq(1+M_J)^2\vee L_J^2$ and $((1+M_J)^2\vee L_J^2)^{3/4} \leq(1+M_J)^2\vee L_J^2$, yields
\begin{equation}
\eta_{\mathrm{sd}}(\ep)^{-1} =\mathcal O\left( \left((1+M_J)^2\vee L_J^2\right) \frac d{\ep^6}\left( \left(\mathbb E_{\nu_0}U_\delta^\alpha(X_0)^2\right)^{1/2} +\ep^{-4}+d\right)\right).
\label{eq:complexity:sd:general:eta}
\end{equation}
For both constant and state dependent $J$, if $2F_0>\ep^2$, then
\begin{align}
K &=\left\lceil\frac1{\rho_*\eta} \log\frac{2F_0}{\ep^2}\right\rceil =\mathcal O\left(\eta^{-1} \left(1+\log\frac{2F_0}{\ep^2}\right)\right),
\label{eq:complexity:K:choice}
\end{align}
where $\rho_*^{-1}=\mathcal O(1)$ and $\eta^{-1}\geq1$ is either $\eta_{\mathrm{sd}}(\ep)^{-1}$ in~\eqref{eq:complexity:sd:general:eta} or $\eta_{\mathrm{const}}(\ep)^{-1}$in~\eqref{eq:complexity:const:general:eta}.
The proof is complete.


\subsection{Proof of Proposition~\ref{prop:PJ:metric:construction}}
\label{proof:prop:PJ:metric:construction}
The matrix $A_J$ is similar to $\widetilde A_J=H_*^{1/2}A_JH_*^{-1/2}=H_*+H_*^{1/2}J(x_*)H_*^{1/2}$. Since $J(x_*)^\top=-J(x_*)$, every eigenpair $\widetilde A_Jv=\zeta v$ with $v\in\mathbb C^d\setminus\{0\}$ satisfies
$$
\mathrm{Re}(\zeta)=\frac{v^*(\widetilde A_J+\widetilde A_J^*)v}{2\|v\|^2}=\frac{v^*H_*v}{\|v\|^2}\geq\lambda_{\min}(H_*)>0.
$$
This proves~\eqref{eq:PJ:spectral:lower:bound}. For $0<\alpha_J<\beta_J$, all eigenvalues of $\alpha_JI-A_J$ have negative real parts. The continuous Lyapunov theorem, see~\cite{boyd2009lyapunov}, gives a unique $P_J\succ0$ satisfying
$$
(A_J-\alpha_JI)^\top P_J+P_J(A_J-\alpha_JI)=I,\qquad A_J^\top P_J+P_JA_J=2\alpha_JP_J+I\succeq2\alpha_JP_J.
$$
By using~\eqref{eq:PJ:linearized:dynamics},
$$
\frac{d}{dt}\|z_t\|_{P_J}^2=-z_t^\top(A_J^\top P_J+P_JA_J)z_t\leq-2\alpha_J\|z_t\|_{P_J}^2,\qquad \|z_t\|_{P_J}\leq e^{-\alpha_Jt}\|z_0\|_{P_J}.
$$
For $J(x_*)=0$, choosing $P_0=H_*$ gives
$$
A_0^\top P_0+P_0A_0=2H_*^2\succeq2\lambda_{\min}(H_*)H_*.
$$
If $\beta_J>\lambda_{\min}(H_*)$, any $\alpha_J\in(\lambda_{\min}(H_*),\beta_J)$ yields the strict improvement in the adapted metric. The proof is complete.

\subsection{Proof of Lemma~\ref{lemma:H}}
\label{proof:lemma:H}
Let $D=B_J(x)-B_J(y)$. Since $H_\eta(x)-H_\eta(y)=x-y-\eta D$, \eqref{eq:strong:full:drift} gives
\begin{align}
\|H_\eta(x)-H_\eta(y)\|_{P_J}^2 &=\|x-y\|_{P_J}^2 -2\eta\langle x-y,D\rangle_{P_J} +\eta^2\|D\|_{P_J}^2
\nonumber\\
&\leq\|x-y\|_{P_J}^2 -2\alpha_{P,J}\eta\|x-y\|_{P_J}^2 +L_{P,J}^2\eta^2\|x-y\|_{P_J}^2
\nonumber\\
&\leq \left(1-\frac32\alpha_{P,J}\eta\right) \|x-y\|_{P_J}^2,
\label{eq:strong:onestep:squared}
\end{align}
where the last inequality follows from $0<\eta\leq\frac{\alpha_{P,J}}{2L_{P,J}^2}$. For $x\neq y$, Cauchy-Schwarz and~\eqref{eq:strong:full:drift} also imply
$$
\alpha_{P,J}\|x-y\|_{P_J}^2 \leq\langle x-y,D\rangle_{P_J} \leq\|x-y\|_{P_J}\|D\|_{P_J} \leq L_{P,J}\|x-y\|_{P_J}^2,
$$
and hence $\alpha_{P,J}\leq L_{P,J}$. Therefore,
\begin{equation}
\alpha_{P,J}\eta \leq\frac{\alpha_{P,J}^2}{2L_{P,J}^2} \leq\frac12, \qquad 1-\frac32\alpha_{P,J}\eta\geq\frac14.
\end{equation}
Taking square roots in~\eqref{eq:strong:onestep:squared} proves \eqref{eq:strong:onestep}. The proof is complete.


\subsection{Proof of Lemma~\ref{lemma:1:step:error}}
\label{proof:lemma:1:step:error}

Note that $L_J=0$ when $J$ is constant. Boundedness of $J$ and the $L_\delta$ Lipschitz continuity of $\nabla V_\delta$ give
\begin{equation}
\|(I+J(x))\nabla V_\delta(x)\| \leq(1+M_J)\bigl(L_\delta\|x\|+\|\nabla V_\delta(0)\|\bigr).
\end{equation}
Moreover, if $\|x\|\vee\|y\|\leq R_0$, then
\begin{align}
&\|(I+J(x))\nabla V_\delta(x)-(I+J(y))\nabla V_\delta(y)\| \leq \left((1+M_J)L_\delta+L_J\bigl(L_\delta R_0+\|\nabla V_\delta(0)\|\bigr)\right)\|x-y\|.
\end{align}
Thus the drift is locally Lipschitz with linear growth, so the SDE has a unique nonexplosive strong solution. Consider the continuous time nonreversible diffusion
\begin{equation}
dY_t = -\left(I + J(Y_t)\right)\nabla V_\delta(Y_t)dt + \sqrt{2}dW_t, \qquad Y_0 \sim \pi_\delta.
\end{equation}
We define the one step error by
\begin{equation}
\mathcal{R}_0:= \int_0^\eta \left(\left(I + J(Y_t)\right)\nabla V_\delta(Y_t) - \left(I + J(Y_0)\right)\nabla V_\delta(Y_0)\right)dt .
\end{equation}
All $L^2$ and $L^4$ norms in this proof are the random vector norms defined in~\eqref{eq:Lp:PJ:norms}. First, we decompose the drift difference as
\begin{align}
&\left(I + J(Y_t)\right)\nabla V_\delta(Y_t) - \left(I + J(Y_0)\right)\nabla V_\delta(Y_0) \nonumber
\\
&\qquad = \left(I + J(Y_t)\right) \left(\nabla V_\delta(Y_t)-\nabla V_\delta(Y_0)\right) + \left(J(Y_t)-J(Y_0)\right)\nabla V_\delta(Y_0).
\end{align}
By Minkowski's inequality,
\begin{align}
\|\mathcal R_0\|_{L^2} &\leq \int_0^\eta \left\| \left(I + J(Y_t)\right) \left(\nabla V_\delta(Y_t)-\nabla V_\delta(Y_0)\right) \right\|_{L^2}dt + \int_0^\eta \left\| \left(J(Y_t)-J(Y_0)\right)\nabla V_\delta(Y_0) \right\|_{L^2}dt .
\label{eq:decomposition}
\end{align}
We bound the two terms separately. For the first term, by the boundedness of $J$ and the $L_\delta$ smoothness of $V_\delta$, we obtain
\begin{equation}
\left\| \left(I+J(Y_t)\right) \left(\nabla V_\delta(Y_t)-\nabla V_\delta(Y_0)\right) \right\|_{L^2} \leq (1+M_J)L_\delta\|Y_t-Y_0\|_{L^2}.
\end{equation}
For constant $J$, the second term in the decomposition~\eqref{eq:decomposition} is zero. For state dependent $J$, use the $L_J$ Lipschitz continuity of $J$ and the definition of $G_4$ to obtain
\begin{align}
\left\|\left(J(Y_t)-J(Y_0)\right)\nabla V_\delta(Y_0)\right\|_{L^2} & \leq L_J \left(\mathbb E\left[|Y_t-Y_0|^2|\nabla V_\delta(Y_0)|^2\right] \right)^{1/2} \nonumber
\\
& \leq L_J \left(\mathbb E|Y_t-Y_0|^4\right)^{1/4}\left(\mathbb E|\nabla V_\delta(Y_0)|^4\right)^{1/4} \nonumber
\\
&=L_JG_4\|Y_t-Y_0\|_{L^4},
\end{align}
where we applied Cauchy-Schwarz to $|Y_t-Y_0|^2$ and $|\nabla V_\delta(Y_0)|^2$. Then, we bound the increments $Y_t-Y_0$ in $L^2$ and $L^4$. We have
$$
Y_t-Y_0 = -\int_0^t\left(I+J(Y_s)\right)\nabla V_\delta(Y_s)ds + \sqrt{2}W_t .
$$
Using Minkowski's inequality in $L^2$, we obtain
\begin{align}
\|Y_t-Y_0\|_{L^2} & \leq \int_0^t \left\| \left(I+J(Y_s)\right)\nabla V_\delta(Y_s)\right\|_{L^2}ds + \|\sqrt{2}W_t\|_{L^2} \nonumber
\\
& \leq (1+M_J) \int_0^t \|\nabla V_\delta(Y_s)\|_{L^2}ds + \sqrt{2dt}.
\end{align}
Since $Y_0\sim \pi_\delta$ and the diffusion is stationary, $Y_s\sim \pi_\delta$ for all $s\geq 0$. We get
\begin{equation}
\label{eq:L2:increment}
\|Y_t-Y_0\|_{L^2} \leq (1+M_J)G_2t + \sqrt{2dt},
\end{equation}
where we define $G_2$ in~\eqref{eq:strong:local:error:constant}. Consequently,
\begin{align}
\int_0^\eta \|Y_t-Y_0\|_{L^2}dt & \leq (1+M_J)G_2\int_0^\eta t\,dt + \sqrt{2d}\int_0^\eta t^{1/2}dt \nonumber
\\
& = \frac{1}{2}(1+M_J)G_2\eta^2 + \frac{2}{3}\sqrt{2d}\eta^{3/2}.
\end{align}
Using Minkowski's inequality in $L^4$, we have
\begin{align}
\|Y_t-Y_0\|_{L^4} & \leq \int_0^t \left\|\left(I+J(Y_s)\right)\nabla V_\delta(Y_s)\right\|_{L^4}ds + \|\sqrt{2}W_t\|_{L^4} \nonumber
\\
& \leq (1+M_J) \int_0^t \|\nabla V_\delta(Y_s)\|_{L^4}ds + \sqrt{2}\|W_t\|_{L^4}.
\end{align}
Since $Y_s\sim\pi_\delta$ for all $s\geq0$, $\|\nabla V_\delta(Y_s)\|_{L^4}=G_4$ where we define $G_4$ in~\eqref{eq:strong:local:error:sd:constant}. Moreover, $\|\sqrt{2}W_t\|_{L^4}=\sqrt{2}[d(d+2)]^{1/4}t^{1/2}$. Thus,
\begin{equation}
\label{eq:L4:increment}
\|Y_t-Y_0\|_{L^4} \leq (1+M_J)G_4t + \sqrt{2}[d(d+2)]^{1/4}t^{1/2}.
\end{equation}
Integrating over $t\in[0,\eta]$ gives
\begin{align}
\int_0^\eta \|Y_t-Y_0\|_{L^4}dt & \leq (1+M_J)G_4\int_0^\eta t\,dt + \sqrt{2}[d(d+2)]^{1/4} \int_0^\eta t^{1/2}dt \nonumber
\\
& = \frac12(1+M_J)G_4\eta^2 + \frac{2\sqrt2}{3}[d(d+2)]^{1/4}\eta^{3/2}.
\end{align}
\paragraph{Constant $J$.} In this case $J(Y_t)-J(Y_0)=0$. Hence the second integral in the decomposition of $\mathcal R_0$ vanishes, and
\begin{align}
\|\mathcal R_0\|_{L^2} &\leq(1+M_J)L_\delta\int_0^\eta\|Y_t-Y_0\|_{L^2}dt \nonumber
\\
&\leq(1+M_J)L_\delta\left( \frac12(1+M_J)G_2\eta^2 +\frac23\sqrt{2d}\eta^{3/2} \right) \nonumber
\\
&=\eta^{3/2}(1+M_J)L_\delta\left( \frac12(1+M_J)G_2\eta^{1/2} +\frac23\sqrt{2d} \right)
\nonumber\\
&\leq\eta^{3/2}(1+M_J)L_\delta\left( \frac12(1+M_J)G_2 +\frac23\sqrt{2d} \right) :=C_{\mathcal R}^{\rm const}\eta^{3/2},
\end{align}
where the last inequality uses $\eta^{1/2}\leq1$ and the last identity is the definition in~\eqref{eq:strong:local:error:constant}.

\paragraph{State dependent $J$.} In this case both integrals in the decomposition of $\mathcal R_0$ remain. Using~\eqref{eq:L2:increment} and~\eqref{eq:L4:increment}, we obtain
\begin{align}
\|\mathcal R_0\|_{L^2} &\leq(1+M_J)L_\delta\int_0^\eta\|Y_t-Y_0\|_{L^2}dt +L_JG_4\int_0^\eta\|Y_t-Y_0\|_{L^4}dt
\nonumber\\
&\leq\eta^{3/2}(1+M_J)L_\delta\left( \frac12(1+M_J)G_2\eta^{1/2} +\frac23\sqrt{2d} \right)
\nonumber\\
&\qquad +\eta^{3/2}L_JG_4\left( \frac12(1+M_J)G_4\eta^{1/2} +\frac{2\sqrt2}{3}[d(d+2)]^{1/4} \right)
\nonumber\\
&\leq\eta^{3/2}\Bigg( (1+M_J)L_\delta\left( \frac12(1+M_J)G_2 +\frac23\sqrt{2d} \right)
\nonumber\\
&\qquad +L_JG_4\left( \frac12(1+M_J)G_4 +\frac{2\sqrt2}{3}[d(d+2)]^{1/4} \right) \Bigg):=C_{\mathcal R}^{\rm sd}\eta^{3/2},
\end{align}
where the inequality again uses $\eta^{1/2}\leq1$ and the last identity is exactly~\eqref{eq:strong:local:error:sd:constant}. The proof is complete.

\subsection{Proof of Lemma~\ref{lemma:L:2}}
\label{proof:lemma:L:2}
We adapt the moment argument in the proof of \cite[Proposition~2.21]{gurbuzbalaban2024penalized} to the nonreversible case. Let $x_*$ be the unique minimizer of $V_{\delta}=f+S/\delta$. Lemma~\ref{lemma:min} gives $\|x_*\|\leq(1+c)R$, with $c$ defined there, and optimality gives $\nabla V_{\delta}(x_*)=0$. The potential $V_{\delta}$ is strongly convex with parameter $\mu$ and has an $L_{\delta}$ Lipschitz gradient. The initial potential assumption~\eqref{eq:strong:initial:potential} and strong convexity imply
\begin{align}
\mathbb E_{\nu_0}\|x_0\|^2 &\leq2\mathbb E_{\nu_0}\|x_0-x_*\|^2+2\|x_*\|^2 \leq\frac4\mu \mathbb E_{\nu_0} \left[V_\delta(x_0)-V_\delta(x_*)\right] +2\|x_*\|^2<\infty.
\label{eq:strong:initial:second:from:potential}
\end{align}
We next use $L_\delta$ smoothness to obtain
\begin{equation}
V_\delta\left(x_{k+1}\right) \leq V_\delta\left(x_k\right)+\left\langle\nabla V_\delta\left(x_k\right), x_{k+1}-x_k\right\rangle+\frac{L_\delta}{2}\left\|x_{k+1}-x_k\right\|^2.
\end{equation}
Taking conditional expectation given $x_k$ gives
\begin{align}
\mathbb{E}\left[\left\langle\nabla V_\delta\left(x_k\right), x_{k+1}-x_k\right\rangle | x_k\right] &=-\eta\left\langle\nabla V_\delta\left(x_k\right),\left(I+J\left(x_k\right)\right) \nabla V_\delta\left(x_k\right)\right\rangle \nonumber
\\
& =-\eta\left\|\nabla V_\delta\left(x_k\right)\right\|^2-\eta\left\langle\nabla V_\delta\left(x_k\right), J\left(x_k\right) \nabla V_\delta\left(x_k\right)\right\rangle \nonumber
\\
& = -\eta\left\|\nabla V_\delta\left(x_k\right)\right\|^2,
\end{align}
because $J(x_k)$ is skew symmetric and hence
$$
\left\langle\nabla V_\delta\left(x_k\right), J\left(x_k\right) \nabla V_\delta\left(x_k\right)\right\rangle=0.
$$
The conditional second moment of the increment satisfies
\begin{align}
\mathbb{E}\left[\left\|x_{k+1}-x_k\right\|^2 | x_k\right] & = \eta^2\left\|\left(I+J\left(x_k\right)\right) \nabla V_\delta\left(x_k\right)\right\|^2 \nonumber
\\
& \qquad +\eta^2 \mathbb{E}\left[\left\|\left(I+J\left(x_k\right)\right)\left(\widetilde\nabla f\left(x_k\right)-\nabla f\left(x_k\right)\right)\right\|^2 | x_k\right]+2 \eta d \nonumber
\\
& \leq \eta^2\left(1+M_J^2\right)\left\|\nabla V_\delta\left(x_k\right)\right\|^2 \nonumber
\\
& \qquad +2\eta^2\left(1+M_J^2\right) \sigma^2\left(L^2\left\|x_k\right\|^2+\|\nabla f(0)\|^2\right) +2 \eta d,
\end{align}
where skew symmetry implies, for every $v\in\mathbb R^d$,
$$
\left\|\left(I+J\left(x_k\right)\right) v\right\|^2=\|v\|^2+\left\|J\left(x_k\right) v\right\|^2 \leq\left(1+M_J^2\right)\|v\|^2.
$$
Taking expectation and using the stochastic gradient bound gives
\begin{align}
\mathbb{E} V_\delta\left(x_{k+1}\right) & \leq \mathbb{E} V_\delta\left(x_k\right)-\eta\left(1-\frac{\eta L_\delta\left(1+M_J^2\right)}{2}\right) \mathbb{E}\left\|\nabla V_\delta\left(x_k\right)\right\|^2 \nonumber
\\
& \qquad +\eta^2 L_\delta\left(1+M_J^2\right) \sigma^2\left(L^2 \mathbb{E}\left\|x_k\right\|^2+\|\nabla f(0)\|^2\right)+\eta L_\delta d.
\end{align}
Strong convexity of $V_\delta$ and the gradient domination inequality give
$$
\left\|\nabla V_\delta\left(x_k\right)\right\|^2 \geq 2 \mu\left(V_\delta\left(x_k\right)-V_\delta\left(x_*\right)\right),
$$
and strong convexity also gives
\begin{equation}
\label{eq:error:norm}
\left\|x_k-x_*\right\|^2 \leq \frac{2}{\mu}\left(V_\delta\left(x_k\right)-V_\delta\left(x_*\right)\right) .
\end{equation}
Therefore,
$$
\left\|x_k\right\|^2 \leq 2\left\|x_k-x_*\right\|^2+2\left\|x_*\right\|^2 \leq \frac{4}{\mu}\left(V_\delta\left(x_k\right)-V_\delta\left(x_*\right)\right)+2\left\|x_*\right\|^2 .
$$
Thus,
\begin{align}
& \mathbb{E} \left[V_\delta\left(x_{k+1}\right)-V_\delta\left(x_*\right)\right] \nonumber
\\
& \qquad \leq \left[1-2 \eta \mu\left(1-\frac{\eta L_\delta\left(1+M_J^2\right)}{2}\right)+\frac{4 \eta^2 L_\delta L^2\left(1+M_J^2\right) \sigma^2}{\mu}\right] \mathbb{E}\left[V_\delta\left(x_k\right)-V_\delta\left(x_*\right)\right] \nonumber
\\
& \qquad +2\eta^2 L^2L_{\delta}\left(1+M_J^2\right) \sigma^2\left\|x_*\right\|^2+\eta^2 L_\delta\left(1+M_J^2\right) \sigma^2\|\nabla f(0)\|^2+\eta L_\delta d.
\end{align}
Noticing $L\leq L_\delta$, it shows that the coefficient in the preceding recursion satisfies
\begin{align}
&1-2\eta\mu\left(1-\frac{\eta L_\delta(1+M_J^2)}2\right) +\frac{4\eta^2L_\delta L^2(1+M_J^2)\sigma^2}{\mu}
\nonumber\\
&\quad\leq 1-2\eta\mu +\eta^2\mu L_\delta(1+M_J^2) +\frac{4\eta^2L_\delta^3(1+M_J^2)\sigma^2}{\mu} =1-\beta_\eta,
\end{align}
where $\beta_{\eta} := 2\eta\mu\left(1-\eta L_\delta\left(1+M_J^2\right)\left(\frac{1}{2}+ \frac{2L_\delta^2\sigma^2}{\mu^2}\right)\right)$. We choose the stepsize $\eta$ such that
\begin{equation}
\eta \leq \frac{1}{L_\delta\left(1+M_J^2\right)\left(1 + \frac{4L_\delta^2\sigma^2}{\mu^2}\right)} \wedge \frac{1}{2\mu}.
\end{equation}
Therefore,
\begin{align}
\eta L_\delta(1+M_J^2) \left(\frac12+\frac{2L_\delta^2\sigma^2}{\mu^2}\right) &\leq\frac12, \quad \eta\mu\leq\beta_\eta<2\eta\mu\leq1.
\end{align}
With $C_{\mathrm{sg},\delta}$ defined in~\eqref{eq:strong:sg:constant}, the recursion becomes
$$
\mathbb{E}\left[V_\delta\left(x_{k+1}\right)-V_\delta\left(x_*\right)\right] \leq (1-\beta_\eta) \mathbb{E}\left[V_\delta\left(x_k\right)-V_\delta\left(x_*\right)\right]+\eta \left(\eta C_{\mathrm{sg},\delta} + L_\delta d\right).
$$
Iterating from $x_0\sim\nu_0$ yields
\begin{align}
\mathbb{E}\left[V_\delta\left(x_k\right)-V_\delta\left(x_*\right)\right] &\leq (1-\beta_{\eta})^k \mathbb E_{\nu_0} \left[V_\delta\left(x_0\right)-V_\delta\left(x_*\right)\right] +\eta(\eta C_{\mathrm{sg},\delta}+L_\delta d)\sum_{j=0}^{k-1}(1-\beta_\eta)^j
\nonumber\\
&\leq (1-\beta_{\eta})^k \mathbb E_{\nu_0} \left[V_\delta\left(x_0\right)-V_\delta\left(x_*\right)\right] +\frac{\eta C_{\mathrm{sg},\delta}+L_\delta d}{\mu},
\end{align}
where the last inequality uses $\beta_\eta\geq\eta\mu$. Therefore,~\eqref{eq:error:norm} gives, for all $k=1,2,\ldots$,
\begin{equation}
\mathbb{E}\left\|x_k-x_*\right\|^2 \leq \frac{2}{\mu}(1-\beta_{\eta})^k\mathbb E_{\nu_0} \left[V_\delta\left(x_0\right)-V_\delta\left(x_*\right)\right]+ \frac{2(\eta C_{\mathrm{sg},\delta} + L_\delta d)}{\mu^2} < \infty.
\end{equation}
Consequently,
\begin{equation}
\sup_k\mathbb{E}\left\| x_k \right\|^2 \leq 2\sup_k\mathbb{E}\left\|x_k-x_*\right\|^2 + 2(1+c)^2R^2.
\end{equation}
The proof is complete.

\subsection{Proof of Theorem~\ref{thm:wasserstein:mu}}
\label{proof:thm:wasserstein:mu}
We first apply Lemma~\ref{lemma:1:step:error} with $f=V_\delta$ and $\pi=\pi_\delta$, and let $(Y_t)_{t\geq0}$ be stationary. By \eqref{eq:strong:initial:second:from:potential} and \eqref{eq:moment:target:bounds}, we get
\begin{equation}
\mathcal W_{2,P_J}^2(\nu_0,\pi_\delta) \leq2\lambda_{\max}(P_J) \left(\mathbb E_{\nu_0}\|x_0\|^2+M_{2,\delta}\right)<\infty.
\end{equation}
Fix a finite cost initial coupling $(x_0,Y_0)$. On its product with the Brownian and oracle probability spaces, take $W$, $(w_k)_{k\geq0}$, and $(x_0,Y_0)$ mutually independent. Use $\xi_{k+1}:=\eta^{-1/2}(W_{(k+1)\eta}-W_{k\eta})$ in both processes. Denote
\begin{align}
H_\eta(x)&:=x-\eta B_J(x),\quad \zeta_k:=(I+J(x_k)) \left(\widetilde\nabla f(x_k)-\nabla f(x_k)\right),
\nonumber\\
\mathcal R_k&:=\int_{k\eta}^{(k+1)\eta} \left(B_J(Y_t)-B_J(Y_{k\eta})\right)dt, \qquad \Delta_k:=x_k-Y_{k\eta}.
\end{align}
Then
\begin{align}
x_{k+1} &=H_\eta(x_k)-\eta\zeta_k+\sqrt{2\eta}\xi_{k+1}, \quad Y_{(k+1)\eta}=H_\eta(Y_{k\eta})-\mathcal R_k+\sqrt{2\eta}\xi_{k+1},
\nonumber\\
\Delta_{k+1} &=H_\eta(x_k)-H_\eta(Y_{k\eta})+\mathcal R_k-\eta\zeta_k.
\label{eq:strong:coupling:delta:update}
\end{align}
Since $Y_{k\eta}\sim\pi_\delta$, Lemma~\ref{lemma:1:step:error} gives $\mathbb E\|\mathcal R_k\|^2\leq C_{\mathcal R}^2\eta^3$. Consequently,
\begin{align}
\mathbb E\|\mathcal R_k\|_{P_J}^2 &=\mathbb E\langle\mathcal R_k,P_J\mathcal R_k\rangle \leq\lambda_{\max}(P_J)\mathbb E\|\mathcal R_k\|^2 \leq\lambda_{\max}(P_J)C_{\mathcal R}^2\eta^3.
\label{eq:strong:coupling:remainder}
\end{align}
Let $ \mathcal F_k=\sigma\bigl(x_0,Y_0,w_0,\ldots,w_{k-1},W_s:0\leq s\leq(k+1)\eta\bigr). $ Then $x_k$, $Y_{k\eta}$, and $\mathcal R_k$ are $\mathcal F_k$ measurable, while $w_k$ is independent of $\mathcal F_k$. Assumption~\ref{assumption:sg} and~\eqref{eq:sigma:V} yield
\begin{align}
\mathbb E[\zeta_k|\mathcal F_k]=0\,, \quad \mathbb E\|\zeta_k\|_{P_J}^2\leq\lambda_{\max}(P_J)(1+M_J)^2\sigma_V^2d =\sigma_{P,J}^2d\,.
\label{eq:strong:sg:error:moment}
\end{align}
By~\eqref{eq:strong:full:drift} and Cauchy-Schwarz, we have $\alpha_{P,J}\leq L_{P,J}$ and $0<\alpha_{P,J}\eta \leq\frac{\alpha_{P,J}^2}{2L_{P,J}^2} \leq\frac12$. Thus $ 1-\frac32\alpha_{P,J}\eta\geq\frac14, $ so the parameter $\alpha_{P,J}\eta/[2(1-3\alpha_{P,J}\eta/2)]$ is positive. Then Cauchy-Schwarz and Young's inequality show that
\begin{align}
&2\mathbb E\left\langle H_\eta(x_k)-H_\eta(Y_{k\eta}),\mathcal R_k \right\rangle_{P_J} \nonumber
\\
&\quad\leq 2\mathbb E\left[ \|H_\eta(x_k)-H_\eta(Y_{k\eta})\|_{P_J} \|\mathcal R_k\|_{P_J}\right]
\nonumber\\
&\quad\leq \frac{\alpha_{P,J}\eta} {2(1-3\alpha_{P,J}\eta/2)} \mathbb E\|H_\eta(x_k)-H_\eta(Y_{k\eta})\|_{P_J}^2 +\frac{2(1-3\alpha_{P,J}\eta/2)} {\alpha_{P,J}\eta} \mathbb E\|\mathcal R_k\|_{P_J}^2.
\label{eq:strong:coupling:young:cross}
\end{align}
By~\eqref{eq:strong:onestep} and $\Delta_k=x_k-Y_{k\eta}$,
\begin{align}
\mathbb E\|H_\eta(x_k)-H_\eta(Y_{k\eta})\|_{P_J}^2 &\leq \left(1-\frac32\alpha_{P,J}\eta\right) \mathbb E\|x_k-Y_{k\eta}\|_{P_J}^2 = \left(1-\frac32\alpha_{P,J}\eta\right) \mathbb E\|\Delta_k\|_{P_J}^2.
\label{eq:strong:coupling:H:contraction}
\end{align}
Expanding the square and using \eqref{eq:strong:coupling:young:cross}, \eqref{eq:strong:coupling:H:contraction}, and \eqref{eq:strong:coupling:remainder}, we obtain
\begin{align}
& \mathbb E\|H_\eta(x_k)-H_\eta(Y_{k\eta}) +\mathcal R_k\|_{P_J}^2
\nonumber\\
&\quad= \mathbb E\|H_\eta(x_k)-H_\eta(Y_{k\eta})\|_{P_J}^2 + 2\mathbb E\left\langle H_\eta(x_k)-H_\eta(Y_{k\eta}),\mathcal R_k \right\rangle_{P_J} +\mathbb E\|\mathcal R_k\|_{P_J}^2
\nonumber\\
&\quad\leq \left(1+\frac{\alpha_{P,J}\eta} {2(1-3\alpha_{P,J}\eta/2)}\right) \mathbb E\|H_\eta(x_k)-H_\eta(Y_{k\eta})\|_{P_J}^2 +\left(1+\frac{2(1-3\alpha_{P,J}\eta/2)} {\alpha_{P,J}\eta}\right) \mathbb E\|\mathcal R_k\|_{P_J}^2
\nonumber\\
&\quad\leq \left(1+\frac{\alpha_{P,J}\eta} {2(1-3\alpha_{P,J}\eta/2)}\right) \left(1-\frac32\alpha_{P,J}\eta\right) \mathbb E\|\Delta_k\|_{P_J}^2+\left(1+\frac{2(1-3\alpha_{P,J}\eta/2)} {\alpha_{P,J}\eta}\right)\lambda_{\max}(P_J)C_{\mathcal R}^2\eta^3
\nonumber\\
&\quad=(1-\alpha_{P,J}\eta) \mathbb E\|\Delta_k\|_{P_J}^2 +\left(\frac{2}{\alpha_{P,J}\eta}-2\right) \lambda_{\max}(P_J)C_{\mathcal R}^2\eta^3
\nonumber\\
&\quad\leq(1-\alpha_{P,J}\eta) \mathbb E\|\Delta_k\|_{P_J}^2 +\frac{2\lambda_{\max}(P_J)C_{\mathcal R}^2}{\alpha_{P,J}}\eta^2.
\label{eq:strong:coupling:drift:remainder}
\end{align}
Since $H_\eta(x_k)-H_\eta(Y_{k\eta})+\mathcal R_k$ is $\mathcal F_k$ measurable, the tower property and \eqref{eq:strong:sg:error:moment} give
\begin{align}
&\mathbb E\left\langle H_\eta(x_k)-H_\eta(Y_{k\eta})+\mathcal R_k,\zeta_k \right\rangle_{P_J} =\mathbb E\left[ \left\langle H_\eta(x_k)-H_\eta(Y_{k\eta})+\mathcal R_k, \mathbb E[\zeta_k| \mathcal F_k] \right\rangle_{P_J}\right] =0.
\label{eq:strong:coupling:cross:term}
\end{align}
Expanding the square in~\eqref{eq:strong:coupling:delta:update} and using~\eqref{eq:strong:coupling:cross:term}, \eqref{eq:strong:coupling:drift:remainder}, and \eqref{eq:strong:sg:error:moment}, we obtain
\begin{align}
\mathbb E\|\Delta_{k+1}\|_{P_J}^2 &=\mathbb E\|H_\eta(x_k)-H_\eta(Y_{k\eta}) +\mathcal R_k\|_{P_J}^2 -2\eta\mathbb E\left\langle H_\eta(x_k)-H_\eta(Y_{k\eta})+\mathcal R_k,\zeta_k \right\rangle_{P_J} +\eta^2\mathbb E\|\zeta_k\|_{P_J}^2
\nonumber\\
&\leq(1-\alpha_{P,J}\eta)\mathbb E\|\Delta_k\|_{P_J}^2 +\left(\frac{2\lambda_{\max}(P_J)C_{\mathcal R}^2}{\alpha_{P,J}} +\sigma_{P,J}^2d\right)\eta^2.
\label{eq:strong:coupling:recursion}
\end{align}
For any initial coupling of $x_0\sim\nu_0$ and $Y_0\sim\pi_\delta$, iterating~\eqref{eq:strong:coupling:recursion} gives, for $K\geq1$,
\begin{align}
\mathbb E\|\Delta_K\|_{P_J}^2 &\leq(1-\alpha_{P,J}\eta)^K \mathbb E\|\Delta_0\|_{P_J}^2 +\left(\frac{2\lambda_{\max}(P_J)C_{\mathcal R}^2}{\alpha_{P,J}} +\sigma_{P,J}^2d\right)\eta^2 \sum_{j=0}^{K-1}(1-\alpha_{P,J}\eta)^j
\nonumber\\
&\leq(1-\alpha_{P,J}\eta)^K \mathbb E\|\Delta_0\|_{P_J}^2 +\eta\left( \frac{2\lambda_{\max}(P_J)C_{\mathcal R}^2}{\alpha_{P,J}^2} +\frac{\sigma_{P,J}^2d}{\alpha_{P,J}}\right),
\label{eq:strong:coupling:iteration}
\end{align}
where we use $ \sum_{j=0}^{K-1}(1-\alpha_{P,J}\eta)^j \leq\frac1{\alpha_{P,J}\eta}. $ By construction, $x_K\sim\nu_K$. Since $(Y_t)_{t\geq0}$ is stationary with invariant law $\pi_\delta$, $Y_{K\eta}\sim\pi_\delta$. Hence the joint law of $(x_K,Y_{K\eta})$ is a coupling of $\nu_K$ and $\pi_\delta$, and
\begin{align}
\mathcal W_{2,P_J}^2(\nu_K,\pi_\delta) &=\inf_{\gamma\in\Gamma(\nu_K,\pi_\delta)} \int_{\mathbb R^d\times\mathbb R^d} \|x-y\|_{P_J}^2\gamma(dx,dy) \leq\mathbb E\|x_K-Y_{K\eta}\|_{P_J}^2 =\mathbb E\|\Delta_K\|_{P_J}^2.
\label{eq:strong:coupling:W2:cost}
\end{align}
Combining~\eqref{eq:strong:coupling:iteration} and \eqref{eq:strong:coupling:W2:cost}, and then taking the infimum over all initial couplings of $\nu_0$ and $\pi_\delta$, we obtain
\begin{align}
\mathcal W_{2,P_J}^2(\nu_K,\pi_\delta) &\leq(1-\alpha_{P,J}\eta)^K \mathcal W_{2,P_J}^2(\nu_0,\pi_\delta) +\eta\left( \frac{2\lambda_{\max}(P_J)C_{\mathcal R}^2}{\alpha_{P,J}^2} +\frac{\sigma_{P,J}^2d}{\alpha_{P,J}}\right).
\label{eq:strong:coupling:W2:squared}
\end{align}
The same display is immediate for $K=0$. Moreover,
$$
(1-\alpha_{P,J}\eta)^{K/2} \leq e^{-\alpha_{P,J}K\eta/2}.
$$
Taking square roots gives
\begin{align}
\mathcal W_{2,P_J}(\nu_K,\pi_\delta) &\leq(1-\alpha_{P,J}\eta)^{K/2} \mathcal W_{2,P_J}(\nu_0,\pi_\delta) +\sqrt\eta\left( \frac{\sqrt{2\lambda_{\max}(P_J)}C_{\mathcal R}}{\alpha_{P,J}} +\frac{\sigma_{P,J}\sqrt d}{\sqrt{\alpha_{P,J}}}\right)
\nonumber\\
&\leq e^{-\alpha_{P,J}K\eta/2} \mathcal W_{2,P_J}(\nu_0,\pi_\delta) +\sqrt\eta\left( \frac{\sqrt{2\lambda_{\max}(P_J)}C_{\mathcal R}}{\alpha_{P,J}} +\frac{\sigma_{P,J}\sqrt d}{\sqrt{\alpha_{P,J}}}\right),
\label{eq:W:P:J}
\end{align}
Norm equivalence gives
\begin{equation}
\sqrt{\lambda_{\min}(P_J)}\mathcal W_2(\mu_1,\mu_2) \leq\mathcal W_{2,P_J}(\mu_1,\mu_2) \leq\sqrt{\lambda_{\max}(P_J)}\mathcal W_2(\mu_1,\mu_2),
\label{eq:strong:W2:norm:equivalence}
\end{equation}
and therefore
\begin{equation}
\mathcal W_2(\nu_K,\pi) \leq\frac{\mathcal W_{2,P_J}(\nu_K,\pi_\delta)} {\sqrt{\lambda_{\min}(P_J)}}+\mathcal W_2(\pi_\delta,\pi).
\end{equation}
Substituting~\eqref{eq:W:P:J} into the preceding display gives \eqref{eq:strong:W2:constrained}. The proof is complete.

\subsection{Proof of Theorem~\ref{thm:PJ:wasserstein:convergence}}
\label{proof:thm:PJ:wasserstein:convergence}
Lemma~\ref{lemma:PJ:sg:second:moment} gives $\sup_k\mathbb E\|x_k\|^2\leq C_x^{\mathrm{sg}}$, and Lemma~\ref{lemma:uniform:bound} gives the target moments required by Lemma~\ref{lemma:1:step:error}. Thus the synchronous coupling in the proof of Theorem~\ref{thm:wasserstein:mu} applies, with an arbitrary finite cost initial coupling of $\nu_0$ and $\pi_\delta$. Retain its notation $\Delta_k=x_k-Y_{k\eta}$, $\mathcal R_k$, and $\mathcal F_k$, and write
$$
\mathcal E_{k+1}=(I+J(x_k))(\widetilde\nabla f(x_k,w_k)-\nabla f(x_k)).
$$
The same independence and conditional unbiasedness give $\mathbb E[\mathcal E_{k+1}\mid\mathcal F_k]=0$. Here Assumption~\ref{assumption:sg} and skew symmetry yield
\begin{align}
\mathbb E\|\mathcal E_{k+1}\|_{P_J}^2&\leq2\lambda_{\max}(P_J)(1+M_J^2)\sigma^2\left(L^2\mathbb E\|x_k\|^2+\|\nabla f(0)\|^2\right)
\nonumber\\
&\leq2\lambda_{\max}(P_J)(1+M_J^2)\sigma^2\left(L^2C_x^{\mathrm{sg}}+\|\nabla f(0)\|^2\right).
\label{eq:PJ:coupling:sg:moment}
\end{align}
The drift and remainder estimate~\eqref{eq:strong:coupling:drift:remainder} uses only Lemmas~\ref{lemma:H} and~\ref{lemma:1:step:error}, so it remains valid under the present assumptions. Substituting~\eqref{eq:PJ:coupling:sg:moment} into the squared coupling update gives
\begin{align}
\mathbb E\|\Delta_{k+1}\|_{P_J}^2&=\mathbb E\|H_\eta(x_k)-H_\eta(Y_{k\eta})+\mathcal R_k\|_{P_J}^2+\eta^2\mathbb E\|\mathcal E_{k+1}\|_{P_J}^2
\nonumber\\
&\leq(1-\alpha_{P,J}\eta)\mathbb E\|\Delta_k\|_{P_J}^2+\frac{2\lambda_{\max}(P_J)C_{\mathcal R}^2}{\alpha_{P,J}}\eta^2
\nonumber\\
&\qquad+2\lambda_{\max}(P_J)(1+M_J^2)\sigma^2\left(L^2C_x^{\mathrm{sg}}+\|\nabla f(0)\|^2\right)\eta^2.
\end{align}
Iteration and minimization over the initial coupling, as in~\eqref{eq:strong:coupling:iteration}-\eqref{eq:strong:coupling:W2:squared}, give, for $K\in\mathbb N_0$,
\begin{align}
\mathcal W_{2,P_J}^2(\nu_K,\pi_\delta)&\leq e^{-\alpha_{P,J}K\eta}\mathcal W_{2,P_J}^2(\nu_0,\pi_\delta)
\nonumber\\
&\qquad+\eta\left[\frac{2\lambda_{\max}(P_J)C_{\mathcal R}^2}{\alpha_{P,J}^2}+\frac{2\lambda_{\max}(P_J)(1+M_J^2)\sigma^2}{\alpha_{P,J}}\left(L^2C_x^{\mathrm{sg}}+\|\nabla f(0)\|^2\right)\right].
\end{align}
Taking square roots and applying the norm equivalence~\eqref{eq:strong:W2:norm:equivalence} proves~\eqref{eq:PJ:W2:convergence}.

\subsection{Proof of Corollary~\ref{cor:PJ:complexity}}
\label{proof:cor:PJ:complexity}
To derive the complexity bound, write the two terms in~\eqref{eq:PJ:W2:convergence} as
\begin{align}
T_1(K,\eta) &:= \sqrt{\kappa(P_J)} \exp\left(-\frac12\alpha_{P,J}K\eta\right) \mathcal W_2(\nu_0,\pi_\delta),
\nonumber\\
T_2(\eta) &:=\frac{\sqrt{2\lambda_{\max}(P_J)}C_{\mathcal R}} {\alpha_{P,J}\sqrt{\lambda_{\min}(P_J)}}\sqrt\eta +\frac{\sqrt\eta}{\sqrt{\lambda_{\min}(P_J)}} \left( \frac{2\lambda_{\max}(P_J)(1+M_J^2)\sigma^2} {\alpha_{P,J}} \left( L^2C_x^{\mathrm{sg}} +\|\nabla f(0)\|^2 \right) \right)^{1/2}.
\end{align}
Using $(a+b)^2\leq2a^2+2b^2$ and the stepsize in \eqref{eq:PJ:complexity:accuracy:stepsize} gives
\begin{align}
T_2(\eta)^2 &\leq \frac{4\lambda_{\max}(P_J)\eta}{\lambda_{\min}(P_J)} \Bigg[ \frac{C_{\mathcal R}^2}{\alpha_{P,J}^2} +\frac{(1+M_J^2)\sigma^2}{\alpha_{P,J}} \left( L^2C_x^{\mathrm{sg}} +\|\nabla f(0)\|^2 \right) \Bigg] \leq\frac{\varepsilon^2}{4}.
\end{align}
Hence $T_2(\eta)\leq\varepsilon/2$. It remains to make $T_1(K,\eta)\leq\varepsilon/2$. If $2\sqrt{\kappa(P_J)}\mathcal W_2(\nu_0,\pi_\delta)\leq\varepsilon$, then
\begin{equation}
T_1(0,\eta) =\sqrt{\kappa(P_J)}\mathcal W_2(\nu_0,\pi_\delta) \leq\frac\varepsilon2.
\end{equation}
Otherwise, $2\sqrt{\kappa(P_J)}\mathcal W_2(\nu_0,\pi_\delta)>\varepsilon$, and the requirement $T_1(K,\eta)\leq\varepsilon/2$ is equivalent to
\begin{align}
\sqrt{\kappa(P_J)} \exp\left(-\frac12\alpha_{P,J}K\eta\right) \mathcal W_2(\nu_0,\pi_\delta) &\leq\frac{\varepsilon}{2}.
\end{align}
Equivalently,
\begin{align}
\frac12\alpha_{P,J}K\eta &\geq \log\left( \frac{2\sqrt{\kappa(P_J)} \mathcal W_2(\nu_0,\pi_\delta)} {\varepsilon} \right).
\end{align}
This is exactly~\eqref{eq:PJ:complexity:K:lower}. Thus $T_1(K,\eta)\leq\varepsilon/2$ in both cases. In the nontrivial case~\eqref{eq:PJ:complexity:initial:large}, choose
\begin{equation}
K=\left\lceil \frac{2}{\alpha_{P,J}\eta} \log\left( \frac{2\sqrt{\kappa(P_J)}\mathcal W_2(\nu_0,\pi_\delta)} {\varepsilon} \right) \right\rceil.
\end{equation}
For the stated choice of $\eta$,
\begin{align}
\frac{1}{\alpha_{P,J}\eta} &=\mathcal O\Bigg( \frac1{\alpha_{P,J}}\vee \frac{L_{P,J}^2}{\alpha_{P,J}^2}\vee \frac{L_\delta(1+M_J^2)}{\alpha_{P,J}} \vee \frac{L_\delta(1+M_J^2)\sigma^2L^2} {\alpha_{P,J}q_{U,\delta}c_{U,1}}
\nonumber\\
&\qquad \vee \frac{\lambda_{\max}(P_J)} {\alpha_{P,J}\lambda_{\min}(P_J)\varepsilon^2} \left[ \frac{C_{\mathcal R}^2}{\alpha_{P,J}^2} +\frac{(1+M_J^2)\sigma^2}{\alpha_{P,J}} \left( L^2C_x^{\mathrm{sg}} +\|\nabla f(0)\|^2 \right) \right] \Bigg).
\label{eq:PJ:complexity:inverse:stepsize}
\end{align}
Substitution of~\eqref{eq:PJ:complexity:inverse:stepsize} with the logarithmic factor hidden by $\widetilde{\mathcal O}$ proves \eqref{eq:PJ:complexity:K:order}. Finally, combining the two estimates yields
\begin{equation}
\mathcal W_2(\nu_K,\pi_\delta) \leq T_1(K,\eta)+T_2(\eta) \leq \varepsilon.
\end{equation}
And for $\mathcal W_2(\pi_\delta,\pi)\leq\varepsilon$, the triangle inequality then gives
\begin{equation}
\mathcal W_2(\nu_K,\pi) \leq \mathcal W_2(\nu_K,\pi_\delta)+\mathcal W_2(\pi_\delta,\pi) \leq 2\varepsilon.
\end{equation}
The proof is complete.

\subsection{Proof of Proposition~\ref{prop:PJ:penalty:hessian:block}}
\label{proof:prop:PJ:penalty:hessian:block}
We obtain from~\eqref{eq:PJ:penalty:block:definition},
$$
(I+J_a)H_\delta =\begin{pmatrix}
\lambda&a\Lambda_\delta\\
-a\lambda&\Lambda_\delta
\end{pmatrix}.
$$
Therefore,
\begin{align}
\det\left(zI-(I+J_a)H_\delta\right) &=z^2-(\lambda+\Lambda_\delta)z +(1+a^2)\lambda\Lambda_\delta,
\nonumber\\
\mu_{\delta,\pm}(a) &=\frac{\lambda+\Lambda_\delta \pm\sqrt{(\Lambda_\delta-\lambda)^2 -4a^2\lambda\Lambda_\delta}}2.
\label{eq:PJ:penalty:block:eigenvalues}
\end{align}
For fixed $a$,
\begin{equation}
\frac{(\Lambda_\delta-\lambda)^2 -4a^2\lambda\Lambda_\delta}{\Lambda_\delta^2} =1-\frac{2(1+2a^2)\lambda}{\Lambda_\delta} +\frac{\lambda^2}{\Lambda_\delta^2}\to1,\quad\Lambda_\delta\rightarrow\infty.
\end{equation}
The discriminant is positive for all sufficiently small $\delta$. The trace and determinant in \eqref{eq:PJ:penalty:block:eigenvalues} are positive, so both roots are positive. Rationalizing the smaller root gives
$$
\mu_{\delta,-}(a) =\frac{2(1+a^2)\lambda\Lambda_\delta} {\lambda+\Lambda_\delta +\sqrt{(\Lambda_\delta-\lambda)^2 -4a^2\lambda\Lambda_\delta}} \to(1+a^2)\lambda,\quad\Lambda_\delta\rightarrow\infty.
$$
Since the sum of the two roots is $\lambda+\Lambda_\delta$, the second limit in \eqref{eq:PJ:penalty:fixed:J:rates} follows from the preceding limit. Substitution of~\eqref{eq:PJ:penalty:tuned:J} into the discriminant in~\eqref{eq:PJ:penalty:block:eigenvalues} gives
$$
(\Lambda_\delta-\lambda)^2 -4a_\delta^2\lambda\Lambda_\delta=0.
$$
Both roots consequently equal $(\lambda+\Lambda_\delta)/2$. Moreover,
$$
a_\delta^2 =\frac{\Lambda_\delta}{4\lambda}-\frac12 +\frac{\lambda}{4\Lambda_\delta},
$$
which proves~\eqref{eq:PJ:penalty:tuned:rate}.

\subsection{Proof of Proposition~\ref{prop:PJ:penalty:Euler:invariant:bias}}
\label{proof:prop:PJ:penalty:Euler:invariant:bias}
Write $M=I-\eta(I+J_a)H_\delta$, with $a$ equal to the fixed coefficient, $0$, or $a_\delta$. Proposition~\ref{prop:PJ:penalty:hessian:block} and~\eqref{eq:PJ:penalty:Euler:bias:stepsize} give $0<\mu_{\delta,\pm}(a)<\lambda+\Lambda_\delta$ and $|1-\eta\mu_{\delta,\pm}(a)|<1$, so $\rho(M)<1$. The discrete Lyapunov theorem~\cite{boyd2009lyapunov} gives
\begin{equation}
\Sigma_{\delta,a}=2\eta\sum_{j=0}^\infty M^j(M^j)^\top=M\Sigma_{\delta,a}M^\top+2\eta I.
\end{equation}
Thus $\mathcal N(0,\Sigma_{\delta,a})$ is invariant for the centered recursion. Iterating that recursion, $M^kZ_0\to0$ in probability and the Gaussian noise covariance tends to $\Sigma_{\delta,a}$, so every initial law converges to this Gaussian. Consequently, $\gamma_\delta^a=\mathcal N(x_\delta^*,\Sigma_{\delta,a})$ is the unique invariant law. Since $\pi_\delta=\mathcal N(x_\delta^*,H_\delta^{-1})$ and
$$
MH_\delta^{-1}M^\top+2\eta I=H_\delta^{-1}+\eta^2(I+J_a)H_\delta(I-J_a),
$$
the same Lyapunov formula applied to the covariance difference gives
\begin{equation}
\Sigma_{\delta,a}-H_\delta^{-1}=\eta^2\sum_{j=0}^\infty M^j(I+J_a)H_\delta(I-J_a)(M^j)^\top\succeq0.
\label{eq:PJ:penalty:Euler:bias:series}
\end{equation}
Coupling $X\sim\pi_\delta$ with $X+G$, where $G\sim\mathcal N(0,\Sigma_{\delta,a}-H_\delta^{-1})$ is independent of $X$, yields
$$
\mathcal W_2^2(\gamma_\delta^a,\pi_\delta)\leq\mathbb E\|G\|^2=\mathrm{Tr}(\Sigma_{\delta,a}-H_\delta^{-1}).
$$

\paragraph{Fixed coefficient $a\neq0$.} The eigenvector matrix
$$
R_{\delta,a}=\begin{pmatrix}1&\dfrac{a\Lambda_\delta}{\mu_{\delta,+}(a)-\lambda}\\ \dfrac{\mu_{\delta,-}(a)-\lambda}{a\Lambda_\delta}&1\end{pmatrix},\quad (I+J_a)H_\delta R_{\delta,a}=R_{\delta,a}\operatorname{diag}(\mu_{\delta,-}(a),\mu_{\delta,+}(a)),
$$
satisfies $R_{\delta,a}\to\left(\begin{smallmatrix}1&a\\0&1\end{smallmatrix}\right)$ as $\delta\to0$ by~\eqref{eq:PJ:penalty:fixed:J:rates}. Hence, for all sufficiently small $\delta$,
$$
\|R_{\delta,a}\|_{\mathrm{op}}\|R_{\delta,a}^{-1}\|_{\mathrm{op}}\leq2(2+a^2),\qquad \mu_{\delta,-}(a)\geq\frac{(1+a^2)\lambda}{2}.
$$
Since $\eta(\mu_{\delta,-}(a)+\mu_{\delta,+}(a))<2$, we have $0<\eta\mu_{\delta,-}(a)<1$ and $|1-\eta\mu_{\delta,+}(a)|\leq1-\eta\mu_{\delta,-}(a)$. Diagonalization therefore gives
\begin{equation}
\|M^j\|_{\mathrm{op}}\leq\|R_{\delta,a}\|_{\mathrm{op}}\|R_{\delta,a}^{-1}\|_{\mathrm{op}}(1-\eta\mu_{\delta,-}(a))^j\leq2(2+a^2)(1-\eta\mu_{\delta,-}(a))^j.
\label{eq:PJ:penalty:Euler:fixed:power:bound}
\end{equation}
Using~\eqref{eq:PJ:penalty:Euler:bias:series} and $\mathrm{Tr}((I+J_a)H_\delta(I-J_a))=(1+a^2)(\lambda+\Lambda_\delta)$,
$$
\begin{aligned}
\mathcal W_2^2(\gamma_\delta^a,\pi_\delta)&\leq\eta^2(1+a^2)(\lambda+\Lambda_\delta)\sum_{j=0}^\infty\|M^j\|_{\mathrm{op}}^2\\
&\leq\frac{4(2+a^2)^2\eta^2(1+a^2)(\lambda+\Lambda_\delta)}{1-(1-\eta\mu_{\delta,-}(a))^2}\\
&\leq\frac{4(2+a^2)^2\eta(1+a^2)(\lambda+\Lambda_\delta)}{\mu_{\delta,-}(a)}\leq8(2+a^2)^2\frac{\eta(\lambda+\Lambda_\delta)}\lambda.
\end{aligned}
$$

\paragraph{Reversible case $a=0$.} The diagonal covariance equation and $\eta<2/(\lambda+\Lambda_\delta)$ give
$$
\begin{aligned}
\mathcal W_2^2(\gamma_\delta^0,\pi_\delta)&\leq\mathrm{Tr}(\Sigma_{\delta,0}-H_\delta^{-1})=\frac{\eta}{2-\eta\lambda}+\frac{\eta}{2-\eta\Lambda_\delta}<\frac{\eta(\lambda+\Lambda_\delta)}{2\Lambda_\delta}+\frac{\eta(\lambda+\Lambda_\delta)}{2\lambda}\leq\frac{\eta(\lambda+\Lambda_\delta)}\lambda.
\end{aligned}
$$

\paragraph{Tuned coefficient $a_\delta$.} Put $c=1-\eta(\lambda+\Lambda_\delta)/2\in(0,1)$. By~\eqref{eq:PJ:penalty:tuned:J},
$$
H_\delta^{1/2}MH_\delta^{-1/2}=cI-\frac{\eta(\Lambda_\delta-\lambda)}2\begin{pmatrix}-1&1\\-1&1\end{pmatrix},\qquad \begin{pmatrix}-1&1\\-1&1\end{pmatrix}^{\!2}=0.
$$
Thus, for $j\geq1$,
$$
\begin{aligned}
H_\delta^{1/2}M^jH_\delta^{-1/2}&=c^jI-\frac{j\eta(\Lambda_\delta-\lambda)}2c^{j-1}\begin{pmatrix}-1&1\\-1&1\end{pmatrix},\\
\|H_\delta^{1/2}M^jH_\delta^{-1/2}\|_{\mathrm{op}}^2&\leq2c^{2j}+2j^2\eta^2(\Lambda_\delta-\lambda)^2c^{2j-2}.
\end{aligned}
$$
Since $1-c^2\geq\eta(\lambda+\Lambda_\delta)/2$, summing gives
$$
\begin{aligned}
\sum_{j=0}^\infty\|H_\delta^{1/2}M^jH_\delta^{-1/2}\|_{\mathrm{op}}^2&\leq\frac2{1-c^2}+\frac{2\eta^2(\Lambda_\delta-\lambda)^2(1+c^2)}{(1-c^2)^3}\\
&\leq\frac4{\eta(\lambda+\Lambda_\delta)}+\frac{32}{\eta(\lambda+\Lambda_\delta)}=\frac{36}{\eta(\lambda+\Lambda_\delta)}.
\end{aligned}
$$
Finally,
$$
\|H_\delta^{1/2}(I+J_{a_\delta})H_\delta^{1/2}\|_{\mathrm F}^2=\lambda^2+\Lambda_\delta^2+\frac{(\Lambda_\delta-\lambda)^2}{2}\leq\frac32(\lambda+\Lambda_\delta)^2,
$$
so~\eqref{eq:PJ:penalty:Euler:bias:series} yields
$$
\begin{aligned}
\mathcal W_2^2(\gamma_\delta^{a_\delta},\pi_\delta)&\leq\eta^2\sum_{j=0}^\infty\|M^j(I+J_{a_\delta})H_\delta^{1/2}\|_{\mathrm F}^2\\
&\leq\frac{\eta^2}{\lambda}\|H_\delta^{1/2}(I+J_{a_\delta})H_\delta^{1/2}\|_{\mathrm F}^2\sum_{j=0}^\infty\|H_\delta^{1/2}M^jH_\delta^{-1/2}\|_{\mathrm{op}}^2\\
&\leq\frac{\eta^2}{\lambda}\frac{3(\lambda+\Lambda_\delta)^2}{2}\frac{36}{\eta(\lambda+\Lambda_\delta)}=54\frac{\eta(\lambda+\Lambda_\delta)}\lambda.
\end{aligned}
$$
The proof is complete.

\subsection{Proof of Proposition~\ref{prop:PJ:penalty:Euler:acceleration}}
\label{proof:prop:PJ:penalty:Euler:acceleration}
For each of the three coefficients, put $M=I-\eta(I+J_a)H_\delta$ and use the corresponding stepsize in~\eqref{eq:PJ:penalty:Euler:accuracy:stepsizes}. Couple the centered chain with a stationary copy using the same Gaussian increments. Then
$$
Z_k-\widetilde Z_k=M^k(Z_0-\widetilde Z_0),\qquad \mathcal W_2(\nu_k^a,\gamma_\delta^a)\leq\|M^k\|_{\mathrm{op}}\mathcal W_2(\nu_0^a,\gamma_\delta^a).
$$
Proposition~\ref{prop:PJ:penalty:Euler:invariant:bias} and the chosen stepsizes give $\mathcal W_2(\gamma_\delta^a,\pi_\delta)\leq\varepsilon/2$ in all three cases. Hence
\begin{equation}
\mathcal W_2(\nu_k^a,\pi_\delta)\leq\|M^k\|_{\mathrm{op}}\mathcal W_2(\nu_0^a,\gamma_\delta^a)+\frac\varepsilon2.
\label{eq:PJ:penalty:Euler:common:coupling}
\end{equation}

\paragraph{Fixed coefficient $a\neq0$.} By~\eqref{eq:PJ:penalty:Euler:fixed:power:bound} and the bound $\mu_{\delta,-}(a)\geq(1+a^2)\lambda/2$ established there,
$$
\|M^k\|_{\mathrm{op}}\leq2(2+a^2)(1-\eta\mu_{\delta,-}(a))^k\leq2(2+a^2)e^{-k\eta(1+a^2)\lambda/2}.
$$

\paragraph{Reversible case $a=0$.} Since $\eta\leq(\lambda+\Lambda_\delta)^{-1}$,
$$
\|M^k\|_{\mathrm{op}}=\|(I-\eta H_\delta)^k\|_{\mathrm{op}}=(1-\eta\lambda)^k\leq e^{-k\eta\lambda}.
$$

\paragraph{Tuned coefficient $a_\delta$.} For $0<\eta\leq(\lambda+\Lambda_\delta)^{-1}$, put $c=1-\eta(\lambda+\Lambda_\delta)/2\in[1/2,1)$. The matrix power formula in the proof of Proposition~\ref{prop:PJ:penalty:Euler:invariant:bias} gives, for $k\geq1$,
$$
\|H_\delta^{1/2}M^kH_\delta^{-1/2}\|_{\mathrm{op}}\leq c^k+k\eta(\Lambda_\delta-\lambda)c^{k-1}.
$$
Since $-\log c\geq\eta(\lambda+\Lambda_\delta)/2$ and $c\geq1/2$,
$$
k\eta(\Lambda_\delta-\lambda)c^{k/2-1}\leq2k\eta(\lambda+\Lambda_\delta)e^{-k\eta(\lambda+\Lambda_\delta)/4}\leq\frac8e.
$$
Consequently, for $k\geq0$,
\begin{equation}
\begin{aligned}
\|M^k\|_{\mathrm{op}}&\leq\sqrt{\frac{\Lambda_\delta}{\lambda}}\|H_\delta^{1/2}M^kH_\delta^{-1/2}\|_{\mathrm{op}}\\
&\leq\sqrt{\frac{\Lambda_\delta}{\lambda}}\left(1+\frac8e\right)c^{k/2}\leq\sqrt{\frac{\Lambda_\delta}{\lambda}}\left(1+\frac8e\right)e^{-k\eta(\lambda+\Lambda_\delta)/4}.
\end{aligned}
\end{equation}
In particular,
\begin{equation}
\mathcal W_2(\nu_k^{a_\delta},\gamma_\delta^{a_\delta})\leq\sqrt{\frac{\Lambda_\delta}{\lambda}}\left(1+\frac8e\right)e^{-k\eta_{\mathrm{tuned}}(\lambda+\Lambda_\delta)/4}\mathcal W_2(\nu_0^{a_\delta},\gamma_\delta^{a_\delta}).
\end{equation}
Substituting the three matrix bounds into~\eqref{eq:PJ:penalty:Euler:common:coupling}, the respective conditions on $K$ in~\eqref{eq:PJ:penalty:Euler:fixed:complexity}, \eqref{eq:PJ:penalty:Euler:reversible:complexity}, and~\eqref{eq:PJ:penalty:Euler:tuned:complexity} make the first term at most $\varepsilon/2$. This proves all three assertions.

\section{Proof Details of Supporting Lemmas}

We provide proofs of supporting lemmas in this section.


\subsection{Proof of Lemma~\ref{lemma:uniform:bound}}
\label{proof:fourth:moment}
Lemma~\ref{lemma:target:moments:dissipative} proves~\eqref{eq:moment:target:bounds} and~\eqref{eq:moment:target:gradient:bounds}. Lemma~\ref{lemma:euler:lyapunov:bounds} gives the two Lyapunov recursions used for the Euler iterates. For every unit vector $u$ and $t>0$,
$$
m_\delta-\frac{b_\delta}{t^2}\leq\frac{\langle tu,\nabla V_\delta(tu)\rangle}{t^2}\leq L_\delta+\frac{g_\delta}{t}.
$$
Hence $m_\delta\leq L_\delta$, and the stepsize condition for the second moment gives
\begin{equation}
0<q_{U,\delta}=\frac{m_\delta^2}{c_{U,2}}\leq\frac{2L_\delta^2}{L_\delta+1}\leq2L_\delta,\qquad 0<\frac{q_{U,\delta}\eta}{2}\leq L_\delta\eta\leq\frac1{(1+M_J)^2}\leq1.
\label{eq:moment:first:coefficient}
\end{equation}
Iteration of~\eqref{eq:moment:euler:first:recursion} therefore yields
\begin{equation}
\sup_{k\geq0}\mathbb E U_\delta(X_{k\eta})\leq\max\left\{\mathbb E_{\nu_0}U_\delta(X_0),\frac{r_{U,\delta}+2L_\delta d}{q_{U,\delta}}\right\}=:C_{U,\delta}^{\rm E}.
\label{eq:moment:euler:first:uniform:constant}
\end{equation}
Define
\begin{equation}
C_x:=\frac{C_{U,\delta}^{\rm E}}{c_{U,1}},\qquad C_{\nabla,2}:=2L_\delta C_{U,\delta}^{\rm E}.
\label{eq:moment:euler:second:constants}
\end{equation}
The bounds in~\eqref{eq:cU1:explicit} and~\eqref{eq:euler:gradient:by:potential} prove~\eqref{eq:2nd:uniform:bound:general}.

Under~\eqref{eq:moment:eta4:explicit}, $\mathbb E_{\nu_0}U_\delta(X_0)^2\leq c_{U,2}^2\mathbb E_{\nu_0}(1+\|X_0\|^2)^2<\infty$, and~\eqref{eq:moment:first:coefficient} gives $0\leq1-q_{U,\delta}\eta/4<1$. Iteration of~\eqref{eq:moment:euler:lyapunov:drift} gives
\begin{equation}
\sup_{k\geq0}\mathbb E U_\delta(X_{k\eta})^2\leq\max\left\{\mathbb E_{\nu_0}U_\delta(X_0)^2,\frac{4R_{U,\delta}}{q_{U,\delta}}\right\}=:C_{U,2,\delta}^{\mathrm E}.
\label{eq:moment:euler:U:fourth:uniform}
\end{equation}
Consequently,
\begin{align}
\sup_{k\geq0}\mathbb E\|X_{k\eta}\|^4&\leq c_{U,1}^{-2}C_{U,2,\delta}^{\mathrm E}=:C_{4,\delta}^{\mathrm E},
\label{eq:moment:euler:fourth:state}\\
\sup_{k\geq0}\mathbb E\|\nabla V_\delta(X_{k\eta})\|^4&\leq4L_\delta^2C_{U,2,\delta}^{\mathrm E}=:C_{\nabla,4}^{\mathrm E}.
\label{eq:moment:euler:fourth:constants}
\end{align}
This proves~\eqref{eq:euler:fourth:moment:bound}.


\subsection{Proof of Lemma~\ref{lemma:PJ:sg:second:moment}}
\label{proof:lemma:PJ:sg:second:moment}
For the update~\eqref{eq:pnsgld}, write
\begin{align}
x_{k+1}&=x_k+\Delta_{k+1}^{\rm sg},
\nonumber\\
\Delta_{k+1}^{\rm sg} &=-\eta(I+J(x_k)) \left[ \nabla V_\delta(x_k) +\widetilde\nabla f(x_k)-\nabla f(x_k) \right] +\sqrt{2\eta}\xi_{k+1}.
\label{eq:moment:sg:increment}
\end{align}
Conditioning on $x_k=x$, the centering in Assumption~\ref{assumption:sg} and the skew symmetry of $J(x)$ give
\begin{align}
&\mathbb E\left[ \left\langle\nabla V_\delta(x), \Delta_{k+1}^{\rm sg}\right\rangle \vert x_k=x\right] =-\eta\|\nabla V_\delta(x)\|^2,
\label{eq:moment:sg:linear}\\
&\mathbb E\left[ \|\Delta_{k+1}^{\rm sg}\|^2 \vert x_k=x\right] =\eta^2\mathbb E\Bigl[ \bigl\|(I+J(x)) \bigl(\nabla V_\delta(x)+\widetilde\nabla f(x_k)-\nabla f(x)\bigr) \bigr\|^2\,\Bigm|\,x_k=x\Bigr] +2d\eta
\nonumber\\
&\qquad \leq\eta^2(1+M_J^2) \Bigl[ \|\nabla V_\delta(x)\|^2 +2\sigma^2 \left(L^2\|x\|^2+\|\nabla f(0)\|^2\right) \Bigr]+2d\eta.
\label{eq:moment:sg:quadratic}
\end{align}
Indeed, the mixed term containing $\widetilde\nabla f(x_k)-\nabla f(x_k)$ has conditional expectation zero, and
\begin{equation}
\|I+J(x)\|_{\mathrm{op}}^2 =1+\|J(x)\|_{\mathrm{op}}^2 \leq1+M_J^2.
\end{equation}
Applying $L_\delta$ smoothness of $U_\delta$ to \eqref{eq:moment:sg:increment}, taking the conditional expectation, and using~\eqref{eq:moment:sg:linear} and \eqref{eq:moment:sg:quadratic}, we obtain
\begin{align}
\mathbb E\left[U_\delta(x_{k+1})| x_k=x\right] &\leq U_\delta(x) -\eta\left(1-\frac{L_\delta\eta(1+M_J^2)}2\right) \|\nabla V_\delta(x)\|^2
\nonumber\\
&\qquad +L_\delta\eta^2(1+M_J^2)\sigma^2 \left(L^2\|x\|^2+\|\nabla f(0)\|^2\right) +L_\delta d\eta.
\end{align}
Under~\eqref{eq:PJ:sg:moment:stepsize}, the coefficient of $\|\nabla V_\delta(x)\|^2$ is at least $1/2$. Therefore \eqref{eq:grad:lower:by:U} and~\eqref{eq:cU1:explicit} give
\begin{align}
\mathbb E\left[U_\delta(x_{k+1})| x_k=x\right] &\leq U_\delta(x) -\frac\eta2\left(q_{U,\delta}U_\delta(x)-r_{U,\delta}\right)+\frac{L_\delta\eta^2(1+M_J^2)\sigma^2L^2} {c_{U,1}}U_\delta(x)
\nonumber\\
&\qquad+L_\delta\eta^2(1+M_J^2)\sigma^2 \|\nabla f(0)\|^2+L_\delta d\eta
\nonumber\\
&\leq\left(1-\frac{q_{U,\delta}\eta}{4}\right) U_\delta(x)+\left( \frac{r_{U,\delta}}2+L_\delta d +L_\delta(1+M_J^2)\sigma^2\|\nabla f(0)\|^2 \right)\eta.
\label{eq:moment:sg:potential:recursion}
\end{align}
Since $q_{U,\delta}\leq2L_\delta$ by \eqref{eq:moment:first:coefficient} and $\eta\leq[L_\delta(1+M_J^2)]^{-1}$, $ 0<\frac{q_{U,\delta}\eta}{4}\leq\frac12. $ Define
\begin{equation}
C_x^{\mathrm{sg}} :=\frac1{c_{U,1}}\max\Bigg\{ \mathbb E_{\nu_0}U_\delta(x_0), \frac4{q_{U,\delta}} \Bigg[ \frac{r_{U,\delta}}2+L_\delta d +L_\delta(1+M_J^2)\sigma^2\|\nabla f(0)\|^2 \Bigg]\Bigg\}<\infty.
\label{eq:Cx:sg:constant}
\end{equation}
Taking expectations in~\eqref{eq:moment:sg:potential:recursion} and iterating the resulting geometric recursion give
\begin{equation}
\sup_{k\geq0}\mathbb E U_\delta(x_k) \leq c_{U,1}C_x^{\mathrm{sg}}.
\end{equation}
Finally, $\|x\|^2\leq U_\delta(x)/c_{U,1}$ from \eqref{eq:cU1:explicit} proves~\eqref{eq:PJ:sg:second:moment}. The proof is complete.


\subsection{Proof of Lemma~\ref{lemma:dissipative:S}}
\label{proof:lemma:dissipative:S}
Let $\mathcal P_{\mathcal K}$ be the Euclidean projection onto $\mathcal K$. By Lemma~\ref{lemma:squared:distance:penalty}, the projection formula and nonexpansiveness give
\begin{align}
\nabla S_{\mathcal K}(x) &=2(x-\mathcal P_{\mathcal K}x),
\nonumber\\
\|\nabla S_{\mathcal K}(x)-\nabla S_{\mathcal K}(y)\| &\leq2\|x-y\|+2\|\mathcal P_{\mathcal K}x-\mathcal P_{\mathcal K}y\| \leq4\|x-y\|.
\end{align}
Since $\|\mathcal P_{\mathcal K}x\|\leq R$,
\begin{align}
\langle x,\nabla S_{\mathcal K}(x)\rangle &=2\|x\|^2-2\langle x,\mathcal P_{\mathcal K}x\rangle
\nonumber\\
&\geq2\|x\|^2-2R\|x\| =\|x\|^2+(\|x\|-R)^2-R^2 \geq\|x\|^2-R^2,
\end{align}
which proves~\eqref{eq:penalty:dissipative:corrected}.


\subsection{Proof of Lemma~\ref{lemma:dissipative}}
\label{proof:lemma:dissipative}
The smoothness constant follows from
\begin{equation}
\|\nabla V_\delta(x)-\nabla V_\delta(y)\| \leq\left(L+\frac\ell\delta\right)\|x-y\|.
\end{equation}
Moreover,
\begin{align}
\langle x,\nabla f(x)\rangle &=\langle x,\nabla f(x)-\nabla f(0)\rangle +\langle x,\nabla f(0)\rangle
\nonumber\\
&\geq-L\|x\|^2-\|x\|\|\nabla f(0)\|
\nonumber\\
&\geq-\left(L+\frac12\right)\|x\|^2 -\frac12\|\nabla f(0)\|^2.
\end{align}
Adding $\delta^{-1}\langle x,\nabla S(x)\rangle \geq\delta^{-1}(m_S\|x\|^2-b_S)$ gives
\begin{equation}
\langle x,\nabla V_\delta(x)\rangle \geq m_\delta\|x\|^2-b_\delta.
\end{equation}


\subsection{Proof of Lemma~\ref{lemma:ghz:transfer}}
\label{proof:lemma:ghz:transfer}
Let $\pi^\alpha$ be the Gibbs law with density proportional to $e^{-f}$ on $\mathcal C^\alpha$. We estimate separately the error from replacing $\mathcal C$ by $\mathcal C^\alpha$ and the error from replacing the constraint by the squared distance penalty. We then apply Lemma~\ref{lem:penalty:uniform:lsi} to bound the log Sobolev constant. Throughout the proof, write
$$
G_{\mathcal C}:=\sup_{x\in\mathcal C}\|\nabla f(x)\|,\qquad E_{\mathcal C}:=\exp\left\{\sup_{\mathcal C}f-\inf_{\mathcal C}f\right\}.
$$

\textbf{The change in the constraint set.} If $h$ is strongly convex, choose $\beta>0$ such that $h-\beta\|\cdot\|^2/2$ is convex. Otherwise, set $\beta=0$. When $\beta=0$, assumption~\eqref{eq:ghz:strict:feasibility} gives $h(0)<0$. When $\beta>0$, closedness of $\mathcal C$ and Assumption~\ref{assumption:C} give $\overline{B(0,r)}\subseteq\mathcal C$. Thus, for any unit vector $\mathbf e$, $h(r\mathbf e)\leq0$ and $h(-r\mathbf e)\leq0$, and strong convexity yields
$$
\begin{aligned}
h(0)&\leq\frac{h(r\mathbf e)+h(-r\mathbf e)}2-\frac\beta8\|r\mathbf e-(-r\mathbf e)\|^2\leq-\frac{\beta r^2}2<0.
\end{aligned}
$$
In both cases, $\tau:=-h(0)>0$. For $0<\alpha\leq\tau/R^2$ and $x\in\mathcal C$, convexity and $\|x\|\leq R$ give
\begin{equation}
\begin{aligned}
h\left(\left(1-\frac{\alpha R^2}{2\tau}\right)x\right)+\frac\alpha2\left\|\left(1-\frac{\alpha R^2}{2\tau}\right)x\right\|^2
&\leq\left(1-\frac{\alpha R^2}{2\tau}\right)h(x)+\frac{\alpha R^2}{2\tau}h(0)+\frac{\alpha R^2}{2}\\
&\leq-\frac{\alpha R^2}{2}+\frac{\alpha R^2}{2}=0.
\end{aligned}
\end{equation}
Since $1-\alpha R^2/(2\tau)\geq1/2$ and $\mathcal C^\alpha\subseteq\mathcal C$, this proves
$$
B(0,r/2)\subseteq\left(1-\frac{\alpha R^2}{2\tau}\right)\mathcal C\subseteq\mathcal C^\alpha\subseteq\mathcal C.
$$
The two Gibbs laws have the same weight $e^{-f}$ on $\mathcal C^\alpha$, so
$$
\pi^\alpha=\pi(\cdot\mid\mathcal C^\alpha),\qquad \frac{d\pi^\alpha}{d\pi}=\frac{\mathbf1_{\mathcal C^\alpha}}{\pi(\mathcal C^\alpha)}.
$$
Splitting the total variation integral over $\mathcal C^\alpha$ and its complement gives
$$
\begin{aligned}
\mathrm{TV}(\pi^\alpha,\pi)&=\frac12\int_{\mathcal C}\left|\frac{\mathbf1_{\mathcal C^\alpha}}{\pi(\mathcal C^\alpha)}-1\right|d\pi=\frac12\left[\left(\frac1{\pi(\mathcal C^\alpha)}-1\right)\pi(\mathcal C^\alpha)+\pi(\mathcal C\setminus\mathcal C^\alpha)\right]=\pi(\mathcal C\setminus\mathcal C^\alpha).
\end{aligned}
$$
Using the set inclusion above, volume scaling, and $1-(1-u)^d\leq du$ for $0\leq u\leq1$, we obtain
$$
\begin{aligned}
\mathrm{TV}(\pi^\alpha,\pi)&=\frac{\int_{\mathcal C\setminus\mathcal C^\alpha}e^{-f(x)}\,dx}{\int_{\mathcal C}e^{-f(x)}\,dx}\leq E_{\mathcal C}\frac{|\mathcal C\setminus\mathcal C^\alpha|}{|\mathcal C|}\leq E_{\mathcal C}\left[1-\left(1-\frac{\alpha R^2}{2\tau}\right)^d\right]\leq\frac{E_{\mathcal C}d\alpha R^2}{2\tau}.
\end{aligned}
$$
For $\alpha=0$, $\mathcal C^0=\mathcal C$ and $\pi^0=\pi$, so the same bound holds with both sides zero.

\textbf{The error from the distance penalty.} Put $D=\mathcal C^\alpha$. The preceding argument gives $B(0,r/2)\subseteq D\subseteq\mathcal C$ in both cases. Since $S^\alpha(x)=\operatorname{dist}(x,D)^2$, define the masses inside and outside $D$ by
$$
Z_D:=\int_D e^{-f(x)}\,dx,\qquad Z_{\rm out}:=\int_{D^c}e^{-f(x)-\operatorname{dist}(x,D)^2/\delta}\,dx.
$$
For $x\notin D$, let $p=\mathcal P_Dx$ and $s=\|x-p\|$. Smoothness, $p\in\mathcal C$, and $\delta L\leq1$ imply
$$
\begin{aligned}
f(x)+\frac{s^2}{\delta}&\geq f(p)+\langle\nabla f(p),x-p\rangle+\left(\frac1\delta-\frac L2\right)s^2\geq\inf_{\mathcal C}f-G_{\mathcal C}s+\frac{s^2}{2\delta}\geq\inf_{\mathcal C}f+\frac{s^2}{4\delta}-\delta G_{\mathcal C}^2.
\end{aligned}
$$
The last inequality uses $G_{\mathcal C}s\leq s^2/(4\delta)+\delta G_{\mathcal C}^2$. Together with $Z_D\geq|D|e^{-\sup_{\mathcal C}f}$, it gives
$$
\frac{Z_{\rm out}}{Z_D}\leq\frac{E_{\mathcal C}e^{\delta G_{\mathcal C}^2}}{|D|}\int_{D^c}e^{-\operatorname{dist}(x,D)^2/(4\delta)}\,dx.
$$
To bound the remaining integral, use $B(0,r/2)\subseteq D$ and convexity to obtain
$$
D+t\overline{B(0,1)}\subseteq(1+2t/r)D,\qquad F_D(t):=\frac{|[D+t\overline{B(0,1)}]\setminus D|}{|D|}\leq(1+2t/r)^d-1.
$$
Indeed, for $x\in D$ and $\|u\|\leq1$,
$$
\frac{x+tu}{1+2t/r}=\frac r{r+2t}x+\frac{2t}{r+2t}\frac r2u\in D.
$$
For $s\geq0$, $e^{-s^2/(4\delta)}=\int_s^\infty [t/(2\delta)]e^{-t^2/(4\delta)}\,dt$. Applying this identity with $s=\operatorname{dist}(x,D)$ and exchanging the nonnegative integrals by Tonelli's theorem gives
$$
\begin{aligned}
\frac1{|D|}\int_{D^c}e^{-\operatorname{dist}(x,D)^2/(4\delta)}\,dx&=\frac1{|D|}\int_{D^c}\int_0^\infty\frac{t}{2\delta}e^{-t^2/(4\delta)}\mathbf1_{\{\operatorname{dist}(x,D)\leq t\}}\,dt\,dx\\
&=\int_0^\infty\frac{t}{2\delta}e^{-t^2/(4\delta)}\frac{|\{x\notin D:\operatorname{dist}(x,D)\leq t\}|}{|D|}\,dt\\
&=\int_0^\infty\frac{t}{2\delta}e^{-t^2/(4\delta)}F_D(t)\,dt.
\end{aligned}
$$
The last equality follows from $\{x\notin D:\operatorname{dist}(x,D)\leq t\}=[D+t\overline{B(0,1)}]\setminus D$ and the definition of $F_D$. We now use $F_D(t)\leq(1+2t/r)^d-1$ and integrate by parts to obtain
$$
\begin{aligned}
\int_0^\infty\frac{t}{2\delta}e^{-t^2/(4\delta)}F_D(t)\,dt
&\leq\int_0^\infty\frac{t}{2\delta}e^{-t^2/(4\delta)}\bigl[(1+2t/r)^d-1\bigr]\,dt=\frac{2d}r\int_0^\infty(1+2t/r)^{d-1}e^{-t^2/(4\delta)}\,dt.
\end{aligned}
$$
Here the boundary term $e^{-t^2/(4\delta)}[(1+2t/r)^d-1]$ vanishes both at $t=0$ and as $t\to\infty$. Finally, $(1+2t/r)^{d-1}\leq e^{2(d-1)t/r}$ and completing the square yield
$$
\begin{aligned}
\int_0^\infty(1+2t/r)^{d-1}e^{-t^2/(4\delta)}\,dt
&\leq e^{4\delta(d-1)^2/r^2}\int_0^\infty e^{-(t-4\delta(d-1)/r)^2/(4\delta)}\,dt\leq2\sqrt{\pi\delta}\,e^{4\delta(d-1)^2/r^2}.
\end{aligned}
$$
In particular, $Z_{\rm out}<\infty$. Since $\pi_\delta^\alpha(\cdot\mid D)=\pi^\alpha$, the conditional distribution calculation above gives
\begin{equation}
\begin{aligned}
\mathrm{TV}(\pi_\delta^\alpha,\pi^\alpha)&=\frac{Z_{\rm out}}{Z_D+Z_{\rm out}}\leq\frac{Z_{\rm out}}{Z_D}\leq\frac{4\sqrt\pi E_{\mathcal C}d}r\sqrt\delta\exp\left[\delta\left(G_{\mathcal C}^2+\frac{4(d-1)^2}{r^2}\right)\right].
\end{aligned}
\end{equation}
Combining the two errors, for fixed dimension and problem data, gives
\begin{equation}
\begin{aligned}
\mathrm{TV}(\pi_\delta^\alpha,\pi)&\leq\mathrm{TV}(\pi_\delta^\alpha,\pi^\alpha)+\mathrm{TV}(\pi^\alpha,\pi)=\mathcal O_d(\sqrt\delta+\alpha)=\mathcal O_d(\ep^2),
\end{aligned}
\label{eq:ghz:penalty:triangle}
\end{equation}
since $\delta=\ep^4$ and $\alpha\in\{0,\ep^2\}$ as in~\eqref{eq:ghz:transfer}.

\textbf{The log Sobolev constant.} Fix $\varrho>0$, independently of $\ep$. With $B_h$ as defined in~\eqref{eq:ghz:subgradient:bound}, formula~\eqref{eq:ghz:mu:orders} becomes
$$
\mu_{\ep^4}^{\alpha}=\begin{cases}\displaystyle\frac{2\beta\varrho}{\ep^4(B_h+\beta\varrho)}-L,&\beta>0,\quad\alpha=0,\\\displaystyle\frac{2\varrho}{\ep^2[B_h+\ep^2(R+\varrho)]}-L,&\beta=0,\quad\alpha=\ep^2.
\end{cases}
$$
Both expressions tend to infinity as $\ep\to0$. Thus $\mu_{\ep^4}^{\alpha}\geq2$ for sufficiently small $\ep$, and Lemma~\ref{lem:penalty:uniform:lsi} yields
$$
\rho_*^{-1}\leq\exp\left\{(R+\varrho)^2\bigl((1+L)+(1+L)^2\bigr)\right\}=\mathcal O(1).
$$
Together with~\eqref{eq:ghz:penalty:triangle}, this proves~\eqref{eq:ghz:transfer}.

\subsection{Proof of Lemma~\ref{lemma:complexity:penalty:orders}}
\label{proof:complexity:penalty:orders}
Since $h(0)\leq0$ and $\alpha\geq0$, the closed convex set $\mathcal C^\alpha$ satisfies
$$
0\in\mathcal C^\alpha\subseteq\mathcal C\subseteq B(0,R),\qquad S^\alpha(0)=0,\qquad \nabla S^\alpha(0)=0.
$$
Lemma~\ref{lemma:dissipative:S}, applied to $\mathcal K=\mathcal C^\alpha$, gives the same constants $(\ell,m_S,b_S)=(4,1,R^2)$ for every $\alpha\geq0$. The calculation in Lemma~\ref{lemma:dissipative} therefore gives~\eqref{eq:diss:const} for $V_\delta^\alpha$ with these constants. Hence, whenever $0<\delta\leq1$ and $\delta(L+1/2)\leq1/2$, we may take
\begin{align}
L_\delta=L+\frac4\delta \leq\frac{L+4}{\delta}, \quad \frac1{2\delta}\leq m_\delta =\frac1\delta-L-\frac12 \leq\frac1\delta, \quad b_\delta=\frac12\|\nabla f(0)\|^2+\frac{R^2}{\delta} \leq\frac{\|\nabla f(0)\|^2/2+R^2}{\delta}.
\label{eq:complexity:Lmb:scaling}
\end{align}
Consequently, uniformly in $\alpha$,
\begin{equation}
L_\delta=\mathcal O(\delta^{-1}), \qquad m_\delta=\Theta(\delta^{-1}), \qquad b_\delta=\mathcal{O}(\delta^{-1}).
\end{equation}
Since $V_\delta^\alpha(0)=f(0)$ and $\|\nabla V_\delta^\alpha(0)\|=\|\nabla f(0)\|$, the definitions in \eqref{eq:moment:quad:constants}, followed by \eqref{eq:complexity:Lmb:scaling}, give
\begin{align}
c_\delta &=|f(0)|+2b_\delta +\frac{L_\delta b_\delta}{m_\delta} +\frac{\|\nabla f(0)\|^2}{2m_\delta}
\nonumber\\
&\leq |f(0)| +\frac{\|\nabla f(0)\|^2+2R^2}{\delta} +\frac{2(L+4)(\|\nabla f(0)\|^2/2+R^2)}{\delta} +\|\nabla f(0)\|^2\delta =\mathcal O(\delta^{-1}),
\\
c_{U,2} &=\max\left\{ 1+|f(0)|+c_\delta+\frac12\|\nabla f(0)\|^2, \frac{L_\delta+1}{2}\right\} =\mathcal O(\delta^{-1}).
\end{align}
By~\eqref{eq:moment:euler:qr:definition}, $q_{U,\delta}=(m_\delta)^2/c_{U,2}$ and $r_{U,\delta}=(m_\delta)^2+2m_\delta b_\delta$. Therefore
\begin{align}
q_{U,\delta} &=\Omega(\delta^{-1}), \quad r_{U,\delta} \leq\delta^{-2} +2\left(\frac12\|\nabla f(0)\|^2+R^2\right)\delta^{-2} =\mathcal O(\delta^{-2}).
\end{align}
The implied constants depend only on $R,L,|f(0)|$, and $\|\nabla f(0)\|$, so all bounds are uniform in $\alpha$. The proof is complete.

\section{Supporting Lemmas}

\subsection{Moment bounds under dissipativity}
\begin{lemma}
\label{lemma:target:moments:dissipative}
Let $V_\delta$ have an $L_\delta$ Lipschitz gradient and be $(m_\delta,b_\delta)$ dissipative with $m_\delta>0$. Define
\begin{equation}
g_\delta:=\|\nabla V_\delta(0)\|,\qquad c_\delta:=|V_\delta(0)|+2b_\delta+\frac{L_\delta b_\delta}{m_\delta}+\frac{g_\delta^2}{2m_\delta}.
\label{eq:moment:quad:constants}
\end{equation}
Then
\begin{equation}
V_\delta(x)\geq\frac{m_\delta}{4}\|x\|^2-c_\delta,\qquad x\in\mathbb R^d.
\label{eq:quad:lower:fourth}
\end{equation}
In particular, $Z_\delta:=\int_{\mathbb R^d}e^{-V_\delta(x)}dx<\infty$. If $X$ has density $Z_\delta^{-1}e^{-V_\delta}$, then
\begin{equation}
\mathbb E\|X\|^2\leq\frac{d+b_\delta}{m_\delta}=:M_{2,\delta},\qquad \mathbb E\|X\|^4\leq\frac{(d+b_\delta)(d+2+b_\delta)}{m_\delta^2}=:M_{4,\delta},
\label{eq:moment:target:constants}
\end{equation}
and
\begin{align}
\mathbb E\|\nabla V_\delta(X)\|^2\leq2L_\delta^2M_{2,\delta}+2g_\delta^2, \qquad \mathbb E\|\nabla V_\delta(X)\|^4\leq8L_\delta^4M_{4,\delta}+8g_\delta^4.
\end{align}
\end{lemma}
\begin{proof}
If $b_\delta=0$, radial integration gives $V_\delta(x)-V_\delta(0)\geq m_\delta\|x\|^2/2$. Suppose $b_\delta>0$, put $r_0=(2b_\delta/m_\delta)^{1/2}$, and first let $\|x\|\geq r_0$ and $y=r_0x/\|x\|$. Dissipativity and $\log u\leq u^2/2$ for $u\geq1$ give
$$
V_\delta(x)-V_\delta(y)\geq\int_{r_0/\|x\|}^1\left(m_\delta t\|x\|^2-\frac{b_\delta}{t}\right)dt\geq\frac{m_\delta}{4}\|x\|^2-b_\delta,
$$
it gives
\begin{equation}
V_\delta(y)\geq V_\delta(0)-g_\delta r_0-\frac{L_\delta r_0^2}{2}\geq V_\delta(0)-\frac{g_\delta^2}{2m_\delta}-b_\delta-\frac{L_\delta b_\delta}{m_\delta}.
\end{equation}
If $\|x\|<r_0$, smoothness and Young's inequality give
\begin{align}
V_\delta(x)&\geq V_\delta(0)-g_\delta\|x\|-\frac{L_\delta}{2}\|x\|^2
\geq\frac{m_\delta}{4}\|x\|^2+V_\delta(0)-\frac{g_\delta^2}{2m_\delta}-\frac{3b_\delta}{2}-\frac{L_\delta b_\delta}{m_\delta}.
\end{align}
These estimates prove~\eqref{eq:quad:lower:fourth}. Now Lipschitz continuity gives
\begin{equation}
\|\nabla V_\delta(x)\|\leq L_\delta\|x\|+g_\delta.
\label{eq:moment:gradient:pointwise}
\end{equation}
By~\eqref{eq:quad:lower:fourth}, $e^{-V_\delta(x)}\leq e^{c_\delta}e^{-m_\delta\|x\|^2/4}$. Together with~\eqref{eq:moment:gradient:pointwise}, this makes the following integrals absolutely convergent and ensures that the polynomial boundary terms in integration by parts vanish. Dissipativity and $\partial_j e^{-V_\delta}=-(\partial_jV_\delta)e^{-V_\delta}$ give
$$
\begin{aligned}
m_\delta\mathbb E\|X\|^2
&\leq b_\delta+\mathbb E\langle X,\nabla V_\delta(X)\rangle\\
&=b_\delta-\frac1{Z_\delta}\sum_{j=1}^d\int_{\mathbb R^d}x_j\partial_j e^{-V_\delta(x)}\,dx\\
&=b_\delta+\frac1{Z_\delta}\sum_{j=1}^d\int_{\mathbb R^d}e^{-V_\delta(x)}\,dx=d+b_\delta.
\end{aligned}
$$
Multiplying the dissipativity inequality by $\|x\|^2$ and using $\partial_j(\|x\|^2x_j)=\|x\|^2+2x_j^2$, we similarly obtain
\begin{equation}
\begin{aligned}
m_\delta\mathbb E\|X\|^4
&\leq b_\delta\mathbb E\|X\|^2+\mathbb E\left[\|X\|^2\langle X,\nabla V_\delta(X)\rangle\right]\\
&=b_\delta\mathbb E\|X\|^2-\frac1{Z_\delta}\sum_{j=1}^d\int_{\mathbb R^d}\|x\|^2x_j\partial_j e^{-V_\delta(x)}\,dx\\
&=b_\delta\mathbb E\|X\|^2+\frac1{Z_\delta}\sum_{j=1}^d\int_{\mathbb R^d}(\|x\|^2+2x_j^2)e^{-V_\delta(x)}\,dx\\
&=(d+2+b_\delta)\mathbb E\|X\|^2\leq\frac{(d+b_\delta)(d+2+b_\delta)}{m_\delta}.
\end{aligned}
\end{equation}
Dividing by $m_\delta>0$ proves~\eqref{eq:moment:target:constants}. Finally,~\eqref{eq:moment:gradient:pointwise} gives
\begin{align}
\mathbb E\|\nabla V_\delta(X)\|^2&\leq2L_\delta^2M_{2,\delta}+2g_\delta^2, \quad \mathbb E\|\nabla V_\delta(X)\|^4\leq8L_\delta^4M_{4,\delta}+8g_\delta^4.
\end{align}
The proof is complete.
\end{proof}


\subsection{Lyapunov bounds for the Euler scheme}
\begin{lemma}
\label{lemma:euler:lyapunov:bounds}
Assume the conditions of Lemma~\ref{lemma:target:moments:dissipative}, let $J$ be Borel measurable with $J(x)^\top=-J(x)$ and $\|J(x)\|_{\mathrm{op}}\leq M_J$, and set
$$
U_\delta(x):=1+V_\delta(x)+c_\delta.
$$
Define
\begin{equation}
c_{U,1}:=\min\left\{1,\frac{m_\delta}{4}\right\},\qquad U_\delta(x)\geq c_{U,1}(1+\|x\|^2),
\label{eq:cU1:explicit}
\end{equation}
\begin{equation}
c_{U,2}:=\max\left\{1+|V_\delta(0)|+c_\delta+\frac{g_\delta^2}{2},\frac{L_\delta+1}{2}\right\},\qquad U_\delta(x)\leq c_{U,2}(1+\|x\|^2),
\label{eq:cU2:explicit}
\end{equation}
and
\begin{equation}
q_{U,\delta}:=\frac{m_\delta^2}{c_{U,2}},\qquad r_{U,\delta}:=m_\delta^2+2m_\delta b_\delta.
\label{eq:moment:euler:qr:definition}
\end{equation}
For the Euler transition in~\eqref{eq:alg}, if $0<\eta\leq1\wedge[L_\delta(1+M_J)^2]^{-1}$, then
\begin{equation}
\mathbb E\left[U_\delta(X_{(k+1)\eta})\mid X_{k\eta}=x\right]\leq\left(1-\frac{q_{U,\delta}\eta}{2}\right)U_\delta(x)+\left(\frac{r_{U,\delta}}2+L_\delta d\right)\eta.
\label{eq:moment:euler:first:recursion}
\end{equation}
If~\eqref{eq:moment:eta4:explicit} also holds, then
\begin{equation}
\mathbb E\left[U_\delta(X_{(k+1)\eta})^2\mid X_{k\eta}=x\right]\leq\left(1-\frac{q_{U,\delta}\eta}{4}\right)U_\delta(x)^2+R_{U,\delta}\eta,
\label{eq:moment:euler:lyapunov:drift}
\end{equation}
where
$$
R_{U,\delta}:=\frac{(r_{U,\delta}+2L_\delta d+8L_\delta)^2}{2q_{U,\delta}}+16L_\delta^2d(d+2).
$$
\end{lemma}
\begin{proof}
The lower bound in~\eqref{eq:cU1:explicit} follows from~\eqref{eq:quad:lower:fourth}. Smoothness at the origin and Young's inequality give
$$
U_\delta(x)\leq1+|V_\delta(0)|+c_\delta+\frac{g_\delta^2}{2}+\frac{L_\delta+1}{2}\|x\|^2,
$$
which proves~\eqref{eq:cU2:explicit}. The descent inequality for an $L_\delta$ smooth function and $\inf U_\delta\geq1$ give
\begin{equation}
\|\nabla V_\delta(x)\|^2\leq2L_\delta U_\delta(x).
\label{eq:euler:gradient:by:potential}
\end{equation}
Moreover,
$$
m_\delta\|x\|^2-b_\delta\leq\|x\|\|\nabla V_\delta(x)\|\leq\frac{m_\delta}{2}\|x\|^2+\frac1{2m_\delta}\|\nabla V_\delta(x)\|^2.
$$
Using~\eqref{eq:cU2:explicit}, we obtain
\begin{equation}
\|\nabla V_\delta(x)\|^2\geq m_\delta^2\|x\|^2-2m_\delta b_\delta\geq q_{U,\delta}U_\delta(x)-r_{U,\delta}.
\label{eq:grad:lower:by:U}
\end{equation}
Condition on $X_{k\eta}=x$ and write
$$
B_\delta(x):=(I+J(x))\nabla V_\delta(x),\qquad \Delta=-\eta B_\delta(x)+\sqrt{2\eta}\xi,\qquad \xi\sim\mathcal N(0,I_d).
$$
Then $X_{(k+1)\eta}=x+\Delta$, and all expectations below are over $\xi$. The matrix $J(x)$ is fixed in this conditional calculation, even when $J$ depends on the state. The pointwise assumptions $J(x)^\top=-J(x)$ and $\|J(x)\|_{\mathrm{op}}\leq M_J$ give
$$
\begin{aligned}
\langle\nabla V_\delta(x),B_\delta(x)\rangle
&=\|\nabla V_\delta(x)\|^2+\langle\nabla V_\delta(x),J(x)\nabla V_\delta(x)\rangle=\|\nabla V_\delta(x)\|^2,\\
\|B_\delta(x)\|^2&\leq\|I+J(x)\|_{\mathrm{op}}^2\|\nabla V_\delta(x)\|^2\leq(1+M_J)^2\|\nabla V_\delta(x)\|^2.
\end{aligned}
$$
Thus the following estimates hold for both constant and state dependent $J$. By the smoothness inequality and the bound on $\eta$, we get
\begin{align}
\mathbb E[U_\delta(x+\Delta)]&\leq U_\delta(x)-\eta\|\nabla V_\delta(x)\|^2+\frac{L_\delta}{2}\left(\eta^2\|B_\delta(x)\|^2+2d\eta\right)
\nonumber\\
&\leq U_\delta(x)-\frac\eta2\|\nabla V_\delta(x)\|^2+L_\delta d\eta.
\label{eq:moment:euler:first:drift}
\end{align}
Combining this with~\eqref{eq:grad:lower:by:U} proves~\eqref{eq:moment:euler:first:recursion}. It remains to prove~\eqref{eq:moment:euler:lyapunov:drift}. By~\eqref{eq:euler:gradient:by:potential} and the standard Gaussian identity $\mathbb E\|\xi\|^4=d(d+2)$,
\begin{align}
\mathbb E\langle \nabla V_\delta(x),\Delta\rangle^2&=\eta^2\|\nabla V_\delta(x)\|^4+2\eta\|\nabla V_\delta(x)\|^2\leq4L_\delta^2\eta^2U_\delta(x)^2+4L_\delta\eta U_\delta(x),
\nonumber\\
\mathbb E\|\Delta\|^4&\leq8\eta^4\|B_\delta(x)\|^4+32\eta^2d(d+2)
\leq32L_\delta^2(1+M_J)^4\eta^4U_\delta(x)^2+32\eta^2d(d+2).
\end{align}
Since
$$
\left|U_\delta(x+\Delta)-U_\delta(x)-\langle \nabla V_\delta(x),\Delta\rangle\right|\leq\frac{L_\delta}{2}\|\Delta\|^2,
$$
the two additional stepsize bounds in~\eqref{eq:moment:eta4:explicit} imply
\begin{equation}
\mathbb E\left[(U_\delta(x+\Delta)-U_\delta(x))^2\right]\leq\frac{q_{U,\delta}\eta}{4}U_\delta(x)^2+8L_\delta\eta U_\delta(x)+16L_\delta^2d(d+2)\eta^2.
\label{eq:moment:euler:squared:increment}
\end{equation}
On the other hand,~\eqref{eq:moment:euler:first:drift} and~\eqref{eq:grad:lower:by:U} give
$$
2U_\delta(x)\mathbb E[U_\delta(x+\Delta)-U_\delta(x)]\leq-q_{U,\delta}\eta U_\delta(x)^2+(r_{U,\delta}+2L_\delta d)\eta U_\delta(x).
$$
Combining this inequality with~\eqref{eq:moment:euler:squared:increment} and using
$$
(r_{U,\delta}+2L_\delta d+8L_\delta)U_\delta(x)\leq\frac{q_{U,\delta}}2U_\delta(x)^2+\frac{(r_{U,\delta}+2L_\delta d+8L_\delta)^2}{2q_{U,\delta}}
$$
and $\eta\leq1$ proves~\eqref{eq:moment:euler:lyapunov:drift}. The proof is complete.
\end{proof}


\subsection{Entropy estimate for the Euler interpolation}
\begin{lemma}
\label{lem:entropy:identity}
Under the regularity, dissipativity, boundedness, and compatibility assumptions on $V_\delta$ and $J$ in Theorem~\ref{thm:nu:pi:delta}, write $B_\delta=(I+J)\nabla V_\delta$ and $F(\mu)=\mathrm{KL}(\mu\Vert\pi_\delta)$. Let $Z\sim\nu$ satisfy $\mathbb E\|Z\|^2<\infty$ and $F(\nu)<\infty$. For Brownian motion $W$ independent of $Z$, set $X_t=Z-tB_\delta(Z)+\sqrt2W_t$, $0\leq t\leq\eta$. Let $p_t$ be its density for $t>0$, let $p_\delta^*$ be the density of $\pi_\delta$, and set $F(p_0):=F(\nu)$ and
$$
I(t):=\int_{\mathbb R^d}p_t(x)\left\|\nabla_x\log\frac{p_t(x)}{p_\delta^*(x)}\right\|^2\,dx.
$$
Then
\begin{equation}
F(p_t)\leq F(p_0)+\frac14\int_0^t\mathbb E\|B_\delta(X_s)-B_\delta(Z)\|^2\,ds,
\label{eq:entropy:endpoint:bound}
\end{equation}
$F(p_\cdot)$ is absolutely continuous on $[0,\eta]$, $\int_0^\eta I(t)\,dt<\infty$, and, for almost every $t\in(0,\eta)$,
\begin{equation}
\frac d{dt}F(p_t)=-I(t)+\mathbb E\left\langle\nabla\log\frac{p_t(X_t)}{p_\delta^*(X_t)},B_\delta(X_t)-B_\delta(Z)\right\rangle.
\label{eq:entropy:closed:interval:identity}
\end{equation}
\end{lemma}
\begin{proof}
Write $B=B_\delta$ and $D(t)=\mathbb E\|B(X_t)-B(Z)\|^2$. Linear growth of $B$ and independence give
$$
\sup_{0\leq t\leq\eta}\mathbb E\|X_t\|^2\leq2\mathbb E\|Z\|^2+2\eta^2\mathbb E\|B(Z)\|^2+2d\eta<\infty,
\qquad
\sup_{0\leq t\leq\eta}D(t)<\infty.
$$
In particular, $\int_0^tD(s)\,ds\to0$ as $t\to0$. Fix $t\in(0,\eta]$. Let $\mathsf P$ be the path law of $X$ on $[0,t]$, and let $\mathsf Q$ be the stationary diffusion law from Lemma~\ref{lemma:pi}. Their initial laws are $\nu$ and $\pi_\delta$, and their drifts are $-B(\omega_0)$ and $-B(\omega_s)$. Local Lipschitz continuity and linear growth of $B$ ensure well posedness of the reference SDE. The energy condition holds under both laws because $\int_0^t\|B(\omega_s)-B(\omega_0)\|^2\,ds\leq4t\sup_{s\leq t}\|B(\omega_s)\|^2<\infty$ for every continuous path. Thus~\cite[Lemma 4.4(i) and Remark 4.5]{lacker2023hierarchies}, with diffusion coefficient $\sqrt2I$, gives
$$
\mathrm{KL}(\mathsf P\Vert\mathsf Q)\leq F(\nu)+\frac12\mathbb E\int_0^t\left\|\frac{B(X_s)-B(Z)}{\sqrt2}\right\|^2\,ds=F(\nu)+\frac14\int_0^tD(s)\,ds.
$$
The terminal laws are $p_t(x)\,dx$ and $\pi_\delta$. The chain rule for relative entropy~\cite[Section 4.3, equation (4.1)]{lacker2023hierarchies} decomposes the path entropy into the terminal entropy and a nonnegative conditional entropy:
$$
\begin{aligned}
\mathrm{KL}(\mathsf P\Vert\mathsf Q)
&=F(p_t)+\int_{\mathbb R^d}\mathrm{KL}\bigl(\mathsf P(\cdot\mid\omega_t=x)\Vert\mathsf Q(\cdot\mid\omega_t=x)\bigr)p_t(x)\,dx\geq F(p_t).
\end{aligned}
$$
Here the conditional laws describe the path given its terminal value. Combining the two inequalities proves~\eqref{eq:entropy:endpoint:bound}. Since $X_t-Z=-tB(Z)+\sqrt2W_t$, independence gives
$$
\mathcal W_2^2(\mathcal L(X_t),\nu)\leq\mathbb E\|X_t-Z\|^2=t^2\mathbb E\|B(Z)\|^2+2dt\rightarrow0.
$$
Lower semicontinuity of relative entropy and~\eqref{eq:entropy:endpoint:bound} yield
$$
F(\nu)\leq\liminf_{t\rightarrow0}F(p_t)\leq\limsup_{t\rightarrow0}F(p_t)\leq F(\nu),
$$
hence $F(p_t)\rightarrow F(p_0)$. For $t>0$, Gaussian smoothing gives a positive smooth density and
$$
\nabla\log p_t(x)=-\frac1{2t}\mathbb E[X_t-Z+tB(Z)\mid X_t=x],
\qquad
\int_{\mathbb R^d}p_t\|\nabla\log p_t\|^2dx\leq\frac d{2t}.
$$
It\^o's formula gives the forward equation
$$
\partial_t p_t(x)=\Delta p_t(x)+\nabla_x\cdot\left(p_t(x)\mathbb E[B(Z)\mid X_t=x]\right)
$$
in the distributional sense. On every $[a,b]\subset(0,\eta]$, the score estimate, the second moment bound, and conditional Jensen's inequality verify the finite Fisher information and kinetic energy hypotheses of the entropy chain rule in~\cite[Theorem 8.3.1, Lemma 8.4.2, Proposition 10.3.18, and Theorem 10.4.13]{ambrosio2005gradient}. Since $D^2V_\delta\succeq-L_\delta I$, it follows that $F(p_\cdot)\in AC([a,b])$ and, for almost every $t\in[a,b]$,
$$
\begin{aligned}
\frac d{dt}F(p_t)
&=-\int_{\mathbb R^d}\nabla\log\frac{p_t(x)}{p_\delta^*(x)}\cdot\left(\nabla p_t(x)+p_t(x)\mathbb E[B(Z)\mid X_t=x]\right)dx\\
&=-I(t)+\mathbb E\left\langle\nabla\log\frac{p_t(X_t)}{p_\delta^*(X_t)},\nabla V_\delta(X_t)-B(Z)\right\rangle.
\end{aligned}
$$
Compatibility and skew symmetry give $\nabla\cdot(J\nabla V_\delta)=0$ almost everywhere and $\langle\nabla V_\delta,J\nabla V_\delta\rangle=0$. The Gaussian score identity and conditional Gaussian integration by parts~\cite[Lemma 2]{stein1981estimation} therefore give
$$
\begin{aligned}
\mathbb E\left\langle\nabla\log\frac{p_t(X_t)}{p_\delta^*(X_t)},J(X_t)\nabla V_\delta(X_t)\right\rangle
&=\mathbb E\langle\nabla\log p_t(X_t),J(X_t)\nabla V_\delta(X_t)\rangle\\
&=-\frac1{2t}\mathbb E\langle X_t-Z+tB(Z),J(X_t)\nabla V_\delta(X_t)\rangle\\
&=-\mathbb E[\nabla\cdot(J\nabla V_\delta)(X_t)]=0.
\end{aligned}
$$
Since $B=\nabla V_\delta+J\nabla V_\delta$, this proves~\eqref{eq:entropy:closed:interval:identity} on $(0,\eta)$. Young's inequality gives, for $0<a<b\leq\eta$,
$$
F(p_b)+\frac12\int_a^bI(t)dt\leq F(p_a)+\frac12\int_a^bD(t)dt.
$$
Letting $a\rightarrow0$ and using $F(p_a)\rightarrow F(p_0)$ gives
$$
\int_0^\eta I(t)dt\leq2F(p_0)+\int_0^\eta D(t)dt<\infty.
$$
Moreover,
$$
\int_0^\eta\left|\mathbb E\left\langle\nabla\log\frac{p_t(X_t)}{p_\delta^*(X_t)},B(X_t)-B(Z)\right\rangle\right|dt\leq\left(\int_0^\eta I(t)dt\right)^{1/2}\left(\int_0^\eta D(t)dt\right)^{1/2}<\infty.
$$
Thus the integrated identity extends to $a=0$, and $F(p_\cdot)\in AC([0,\eta])$. The proof is complete.
\end{proof}

\subsection{A uniform log Sobolev bound for the penalized target}
\begin{lemma}[A uniform log Sobolev bound for $\pi_\delta^\alpha$]
\label{lem:penalty:uniform:lsi}
Under the geometric and smoothness assumptions and the notation $\mathcal C^\alpha$, $S^\alpha$, and $V_\delta^\alpha$ of Lemma~\ref{lemma:ghz:transfer}, let $\beta$ be a positive strong convexity modulus of $h$ when available, and set $\beta=0$ otherwise. Let $\alpha\geq0$ satisfy $\alpha+\beta>0$, and fix $\varrho>0$. Define
\begin{equation}
B_h:=\sup\{\|v\|:x\in\mathcal C,\ v\in\partial h(x)\}<\infty,
\label{eq:ghz:subgradient:bound}
\end{equation}
and
\begin{equation}
\mu_\delta^\alpha:=\frac{2(\alpha+\beta)\varrho}{\delta[B_h+\alpha R+(\alpha+\beta)\varrho]}-L.
\label{eq:ghz:mu:orders}
\end{equation}
If $\mu_\delta^\alpha\geq2$, then $e^{-V_\delta^\alpha}$ is integrable and $\pi_\delta^\alpha$ satisfies the log Sobolev inequality in~\eqref{eq:LSI:convention} with
\begin{equation}
\rho_*\geq e^{-C_*},\qquad C_*:=(R+\varrho)^2\bigl((1+L)+(1+L)^2\bigr).
\end{equation}
In particular, this bound is independent of $\alpha$ and $\delta$ whenever $\mu_\delta^\alpha\geq2$.
\end{lemma}
\begin{proof}
The compactness of $\mathcal C$ and local boundedness of the subdifferential of the finite convex function $h$ give $B_h<\infty$. Put $g=h+\alpha\|\cdot\|^2/2$. Since $\mathcal C^\alpha\subseteq\mathcal C\subseteq B(0,R)$,
$$
\sup_{x\in\mathcal C^\alpha}\sup_{v\in\partial g(x)}\|v\|\leq B_h+\alpha R.
$$
Applying~\cite[Lemma D.1 and Corollary D.2]{gurbuzbalaban2024penalized} to $\mathcal C^\alpha=\{g\leq0\}$ shows that $V_\delta^\alpha$ is $\mu_\delta^\alpha$ strongly convex outside $B(0,R+\varrho)$ and is globally $(-L)$ convex. If $\mu_\delta^\alpha\geq2$,~\cite[Lemma D.3]{gurbuzbalaban2024penalized}, applied with $m=1$, gives a $1$ strongly convex function $U$ such that
$$
\begin{aligned}
&\sup_{x\in\mathbb R^d}\bigl(U(x)-V_\delta^\alpha(x)\bigr)-\inf_{x\in\mathbb R^d}\bigl(U(x)-V_\delta^\alpha(x)\bigr)\\
&\qquad\qquad\leq2(R+\varrho)^2\left(\frac{1+L}{2}+\frac{(1+L)^2}{\mu_\delta^\alpha}\right)\leq(R+\varrho)^2\bigl((1+L)+(1+L)^2\bigr)=C_*.
\end{aligned}
$$
The bounded perturbation also implies that $e^{-V_\delta^\alpha}$ is integrable. The Bakry-\'Emery criterion gives a log Sobolev constant at least $1$ for the Gibbs law with potential $U$. The Holley-Stroock bounded perturbation argument used in \cite[Supplementary Information, Appendix B.1,
Lemma~3, equation~(10)]{ma2019sampling} then gives
$$
\rho_*\geq\exp\left\{-\sup_{x\in\mathbb R^d}\bigl(U(x)-V_\delta^\alpha(x)\bigr)+\inf_{x\in\mathbb R^d}\bigl(U(x)-V_\delta^\alpha(x)\bigr)\right\}\geq e^{-C_*}.
$$
The proof is complete.
\end{proof}


\clearpage
\begingroup \small
\renewcommand{\arraystretch}{1.55} \setlength{\tabcolsep}{3pt}
\begin{longtable}{p{0.15\textwidth}p{0.67\textwidth}>{\raggedleft\arraybackslash}p{0.13\textwidth}}
\emph{Symbol} & \emph{Definition} & \emph{Source}\\
\hline \endfirsthead
\emph{Symbol} & \emph{Definition} & \emph{Source}\\
\hline \endhead \hline \endfoot $M_J$ & $\displaystyle\sup_{x\in\mathbb R^d}\|J(x)\|_{\mathrm{op}}\leq M_J$ &
\eqref{eq:non-convex:J:bounds}\\
$\rho_*$ & $\mathrm{KL}(q\Vert\pi_\delta)\leq(2\rho_*)^{-1} \int\|\nabla\log(q/p_\delta^*)\|^2q$ &
\eqref{eq:LSI:convention}\\
$F_0$ & $F_0:=\mathrm{KL}(p_0\Vert\pi_\delta)$ &
\eqref{eq:init:entropy}\\
$m_\delta,b_\delta,L_\delta$ & $\displaystyle m_\delta:=-L-\frac12+\frac{m_S}{\delta},\quad b_\delta:=\frac12\|\nabla f(0)\|^2+\frac{b_S}{\delta},\quad L_\delta:=L+\frac\ell\delta$ &
\eqref{eq:diss:const}\\
$g_\delta$ & $g_\delta:=\|\nabla V_\delta(0)\|$ &
\eqref{eq:moment:quad:constants}\\
$c_\delta$ & $\displaystyle c_\delta:=|V_\delta(0)|+2b_\delta +\frac{L_\delta b_\delta}{m_\delta} +\frac{g_\delta^2}{2m_\delta}$ &
\eqref{eq:moment:quad:constants}\\
$U_\delta$ & $U_\delta:=1+c_\delta+V_\delta$ &
\eqref{eq:init:potential:moment}\\
$M_{2,\delta}$ & $\displaystyle M_{2,\delta}:=\frac{d+b_\delta}{m_\delta}$ &
\eqref{eq:moment:target:constants}\\
$M_{4,\delta}$ & $\displaystyle M_{4,\delta}:=\frac{(d+b_\delta)(d+2+b_\delta)}{m_\delta^2}$ &
\eqref{eq:moment:target:constants}\\
$c_{U,1}$ & $\displaystyle c_{U,1}:=\min\left\{1,\frac{m_\delta}{4}\right\}$ &
\eqref{eq:cU1:explicit}\\
$c_{U,2}$ & $\displaystyle c_{U,2}:=\max\left\{ 1+|V_\delta(0)|+c_\delta+\frac{g_\delta^2}{2}, \frac{L_\delta+1}{2}\right\}$ &
\eqref{eq:cU2:explicit}\\
$q_{U,\delta}$ & $\displaystyle q_{U,\delta}:=\frac{m_\delta^2}{c_{U,2}}$ &
\eqref{eq:moment:euler:qr:definition}\\
$r_{U,\delta}$ & $r_{U,\delta}:=m_\delta^2+2m_\delta b_\delta$ &
\eqref{eq:moment:euler:qr:definition}\\
$C_{U,\delta}^{\rm E}$ & $\displaystyle C_{U,\delta}^{\rm E}:=\max\left\{ \mathbb E_{\nu_0}U_\delta(X_0), \frac{r_{U,\delta}+2L_\delta d}{q_{U,\delta}}\right\}$ &
\eqref{eq:moment:euler:first:uniform:constant}\\
$C_x$ & $\displaystyle C_x:=\frac{C_{U,\delta}^{\rm E}}{c_{U,1}}$ &
\eqref{eq:moment:euler:second:constants}\\
$C_{\nabla,2}$ & $C_{\nabla,2}:=2L_\delta C_{U,\delta}^{\rm E}$ &
\eqref{eq:moment:euler:second:constants}\\
$R_{U,\delta}$ & $\displaystyle R_{U,\delta}:= \frac{(r_{U,\delta}+2L_\delta d+8L_\delta)^2}{2q_{U,\delta}} +16L_\delta^2d(d+2)$ &
\eqref{eq:moment:euler:lyapunov:drift}\\
$C_{U,2,\delta}^{\rm E}$ & $\displaystyle C_{U,2,\delta}^{\rm E}:=\max\left\{ \mathbb E_{\nu_0}U_\delta(X_0)^2, \frac{4R_{U,\delta}}{q_{U,\delta}}\right\}$ &
\eqref{eq:moment:euler:U:fourth:uniform}\\
$C_{4,\delta}^{\rm E}$ & $C_{4,\delta}^{\rm E}:=c_{U,1}^{-2}C_{U,2,\delta}^{\rm E}$ &
\eqref{eq:moment:euler:fourth:state}\\
$C_{\nabla,4}^{\rm E}$ & $C_{\nabla,4}^{\rm E}:=4L_\delta^2C_{U,2,\delta}^{\rm E}$ &
\eqref{eq:moment:euler:fourth:constants}\\
$C_x^{\mathrm{sg}}$ & $\displaystyle\begin{aligned}
C_x^{\mathrm{sg}}:=\frac1{c_{U,1}}\max\Bigg\{&\mathbb E_{\nu_0}U_\delta(x_0),\,\frac4{q_{U,\delta}}\Bigl[\frac{r_{U,\delta}}2+L_\delta d+L_\delta(1+M_J^2)\sigma^2\|\nabla f(0)\|^2\Bigr]\Bigg\}
\end{aligned}$ &
\eqref{eq:Cx:sg:constant}\\
$\kappa(P_J)$ & $\displaystyle\kappa(P_J):=\frac{\lambda_{\max}(P_J)}{\lambda_{\min}(P_J)}$ &
\eqref{eq:PJ:condition:number}\\
$\beta_J$ & $\displaystyle\beta_J:=\min_{1\leq i\leq d}\mathrm{Re}\lambda_i((I+J(x_*))H_*)$ &
\eqref{eq:beta:J}\\
$\alpha_J$ & $0<\alpha_J<\beta_J$ &
\eqref{eq:PJ:linear:contraction}\\
$\alpha_{P,J},L_{P,J}$ & $\begin{aligned} \langle x-y,B_J(x)-B_J(y)\rangle_{P_J}
&\geq\alpha_{P,J}\|x-y\|_{P_J}^2,\\
\|B_J(x)-B_J(y)\|_{P_J} &\leq L_{P,J}\|x-y\|_{P_J}
\end{aligned}$ &
\eqref{eq:strong:full:drift}\\
$L_{P,J}$ & $\displaystyle L_{P,J}:=\sqrt{\kappa(P_J)} (1+\|J\|_{\mathrm{op}})L_\delta,\quad J(x)\equiv J$ &
\eqref{eq:PJ:lip}\\
$G_2$ & $G_2:=\left(\mathbb E\|\nabla V_{\delta}(Y_0)\|^2\right)^{1/2}$ &
\eqref{eq:strong:local:error:constant}\\
$C_{\mathcal R}^{\rm const}$ & $\displaystyle C_{\mathcal R}^{\rm const}:=(1+M_J)L_{\delta} \left[\frac12(1+M_J)G_2+\frac23\sqrt{2d}\right]$ &
\eqref{eq:strong:local:error:constant}\\
$G_4$ & $G_4:=\left(\mathbb E\|\nabla V_{\delta}(Y_0)\|^4\right)^{1/4}$ &
\eqref{eq:strong:local:error:sd:constant}\\
$C_{\mathcal R}^{\rm sd}$ & $\displaystyle C_{\mathcal R}^{\rm sd}:=C_{\mathcal R}^{\rm const} +L_JG_4\left[\frac12(1+M_J)G_4 +\frac{2\sqrt2}{3}[d(d+2)]^{1/4}\right]$ &
\eqref{eq:strong:local:error:sd:constant}\\
$C_{\mathcal R}$ & $\displaystyle C_{\mathcal R}:=
\begin{cases}
C_{\mathcal R}^{\rm const},&J(x)\equiv J,\\
C_{\mathcal R}^{\rm sd},&J(x)\not\equiv J
\end{cases}$ &
\eqref{eq:local:error:choice}\\
$C_{\mathrm{sg},\delta}$ & $\displaystyle C_{\mathrm{sg},\delta}:= 2L_\delta L^2(1+M_J^2)\sigma^2\|x_*\|^2 +L_\delta(1+M_J^2)\sigma^2\|\nabla f(0)\|^2$ &
\eqref{eq:strong:sg:constant}\\
$\sigma_V^2$ & $\displaystyle\begin{aligned} \sigma_V^2:=\frac{2\sigma^2}{d}\Bigg[L^2\Bigg(
&\frac4\mu\mathbb E_{\nu_0}[V_\delta(x_0)-V_\delta(x_*)]\\
&+\frac{4(\eta C_{\mathrm{sg},\delta}+L_\delta d)}{\mu^2} +2(1+c)^2R^2\Bigg) +\|\nabla f(0)\|^2\Bigg]
\end{aligned}$ &
\eqref{eq:sigma:V}\\
$\sigma_{P,J}$ & $\sigma_{P,J}:=\sqrt{\lambda_{\max}(P_J)}(1+M_J)\sigma_V$ &
\eqref{eq:strong:coupling:sigma}\\
$\eta_{\mathrm{const}}(\ep)$ & $\displaystyle\eta_{\mathrm{const}}(\ep):= 1\wedge\frac1{L_\delta(1+M_J)^2} \wedge\frac{\rho_*\ep^2}{4d(1+M_J)^2L_\delta^2} \wedge\left[ \frac{\rho_*\ep^2}{2(1+M_J)^4L_\delta^2C_{\nabla,2}} \right]^{1/2}$ &
\eqref{eq:stepsize:const:explicit}\\
$\eta_{\mathrm{sd}}(\ep)$ & $\begin{aligned}
\eta_{\mathrm{sd}}(\ep)&:=1\wedge\frac1{L_\delta(1+M_J)^2}\wedge\frac{q_{U,\delta}}{64L_\delta^2}\wedge\left[\frac{q_{U,\delta}}{128L_\delta^4(1+M_J)^4}\right]^{1/3}\\
&\qquad\wedge\frac{\rho_*\ep^2}{12d(1+M_J)^2L_\delta^2+24dL_J^2C_{\nabla,2}}\wedge\left[\frac{\rho_*\ep^2}{96L_J^2L_\delta^2(1+M_J)^4C_{\nabla,4}^{\rm E}}\right]^{1/4}\\
&\qquad\wedge\left(\frac{\rho_*\ep^2}{6}\right)^{1/2}\Bigl[(1+M_J)^4L_\delta^2C_{\nabla,2}+64L_J^2L_\delta^2d(d+2)+2L_J^2(1+M_J)^2C_{\nabla,4}^{\rm E}\Bigr]^{-1/2}
\end{aligned}$ &
\eqref{eq:stepsize:sd:explicit}\\
$B_h$ & $B_h:=\sup\{\|y\|:x\in\mathcal C,\ y\in\partial h(x)\}$ &
\eqref{eq:ghz:subgradient:bound}\\
$\mu_\delta^\alpha$ & $\displaystyle\mu_\delta^\alpha:= \frac{2(\alpha+\beta)\varrho} {\delta[B_h+\alpha R+(\alpha+\beta)\varrho]}-L$ &
\eqref{eq:ghz:mu:orders}\\
\end{longtable}
\endgroup
\end{document}